\documentclass{article}

\usepackage[nonatbib,preprint]{neurips_2019}
\usepackage[utf8]{inputenc}

\usepackage[T1]{fontenc}
\usepackage{hyperref}
\usepackage{url}
\usepackage{booktabs}
\usepackage{amsfonts}
\usepackage{nicefrac}
\usepackage{microtype}
\usepackage{amsmath}
\usepackage{graphicx}
\usepackage{subcaption}
\usepackage{multirow}
\usepackage{tabularx}
\usepackage{wrapfig}
\usepackage{listings}
\usepackage[dvipsnames]{xcolor}
\usepackage{amssymb,amsmath,amsthm}

\newcommand{\R}{\mathbb{R}}

\newtheorem{thm}{Theorem}[section]
\newtheorem{lem}[thm]{Lemma}

\newtheorem{coro}{Corollary}[thm]

\theoremstyle{definition}

\usepackage{xcolor}
\definecolor{Green}{RGB}{50,200,50}
\definecolor{Red}{RGB}{200,50,50}

\graphicspath{ {./images/} }

\title{On Mutual Information in Contrastive \\ Learning for Visual Representations}

\author{
  Mike Wu$^\blacksquare$, Chengxu Zhuang$^\bigstar$, Milan Moss\'e$^{\blacksquare\blacklozenge}$, Daniel Yamins$^{\blacksquare\bigstar}$, Noah Goodman$^{\blacksquare\bigstar}$ \\
  Department of Computer Science $(\blacksquare)$, Philosophy $(\blacklozenge)$, and Psychology $(\bigstar)$ \\
  Stanford University \\
  \texttt{\{wumike, chengxuz, mmosse19, yamins, ngoodman\}@stanford.edu}
}
\begin{document}

\maketitle
\begin{abstract}
In recent years, several unsupervised,  ``contrastive'' learning algorithms in vision have been shown to learn  representations that perform remarkably well on transfer tasks.
We show that this family of algorithms maximizes a lower bound on the mutual information between two or more ``views'' of an image where typical views come from a composition of image augmentations.
Our bound generalizes the InfoNCE objective to support negative sampling from a restricted region of ``difficult'' contrasts.
We find that the choice of negative samples and views are critical to the success of these algorithms.
Reformulating previous learning objectives in terms of mutual information also simplifies and stabilizes them.
In practice, our new objectives yield representations that outperform those learned with previous approaches for transfer to classification, bounding box detection, instance segmentation, and keypoint detection.
The mutual information framework provides a unifying comparison of approaches to contrastive learning and uncovers the choices that impact representation learning.
\end{abstract}

\section{Introduction}

Supervised learning algorithms have given rise to human-level performance in several visual tasks \cite{russakovsky2015imagenet,redmon2016you,he2017mask}, yet they require exhaustive labelled data, posing a barrier to widespread adoption.
In recent years, we have seen the growth of \textit{un}-supervised learning from the vision community \cite{wu2018unsupervised,zhuang2019local,he2019momentum,chen2020simple,chen2020improved} where the aim is to uncover vector representations that are ``semantically'' meaningful as measured by performance on downstream visual tasks.
In the last two years, this class of algorithms has already achieved remarkable results, quickly closing the gap to supervised methods \cite{he2016deep,simonyan2014very}.

The central idea behind these unsupervised algorithms is to \textit{treat every example as its own label} and perform classification, the intuition being that a good representation should be able to discriminate between  examples. Later algorithms build on this basic concept either through (1) technical innovations for numerical stability \cite{wu2018unsupervised}, (2) storage innovations to hold examples in memory \cite{he2019momentum}, (3) choices of data augmentation \cite{zhuang2019local,tian2019contrastive}, or (4) improvements in compute or hyperparameter choices \cite{chen2020simple}.

However, it is surprisingly difficult to compare these algorithms beyond intuition. As we get into the details, it is hard to rigorously justify their design.
For instance, why does clustering of a small neighborhood around every example \cite{zhuang2019local} improve performance? What is special about CIELAB filters \cite{tian2019contrastive}?
From a practitioners point of view, the choices may appear arbitrary.
Ideally, we wish to have a theory that provides a systematic understanding of the full class of algorithms.

In this paper, we describe such a framework based on mutual information between ``views.'' In doing so, we find several insights regarding the individual algorithms. Specifically, our contributions are:
\begin{itemize}
  \item We present an information-theoretic description that can characterize IR \cite{wu2018unsupervised}, LA\cite{zhuang2019local}, CMC\cite{tian2019contrastive}, and more. To do so, we derive a new lower bound on mutual information that supports sampling negative examples from a restricted  distribution.
  \item We simplify and stabilize existing contrastive algorithms, by formulating them as mutual information estimators.
  \item We identify two fundamental choices in this class of algorithms: (1) how to choose data  augmentations (or ``views'') and (2) how to choose negative samples.
  Together these two are the crux of why instance-classification yields useful representations.
  \item By varying how we choose negative examples (a previously unexplored direction), we find consistent improvements in multiple transfer tasks, outperforming IR, LA, and CMC.

\end{itemize}

\section{Background}
\label{sec:background}
We provide a review of three contrastive learning algorithms as described by their respective authors.
We will revisit each of these with the lens of mutual information in Sec.~\ref{sec:exist:algo}.

\textbf{Instance Discrimination (or Instance Recognition, IR)}
Introduced by Wu et. al. \cite{wu2018unsupervised}, IR was the first to classify examples as their own labels.
Let $x^{(i)}$ for $i$ in $[N]$ enumerate the images in a dataset. The function $g_\theta$ is a neural network mapping images to vectors of reals. The IR objective is given by
\begin{equation}
  \mathcal{L}^{\text{IR}}(x^{(i)}, M) = \log p(i|x^{(i)}, M) \text{ where } p(i|x^{(i)}, M) = \frac{e^{g_\theta(x^{(i)})^T M[i] / \omega}}{\sum_{j=1}^N e^{g_\theta(x^{(i)})^T M[j] / \omega}}.
  \label{eqn:ir}
\end{equation}
Here, $\omega$ is used to prevent gradient saturation.
The denominator in $p(i|x^{(i)},M)$ requires $N$ forward passes with $g_\theta$,  which is prohibitively expensive.
Wu et. al. suggest two approaches to ameliorate the cost.
First, they use a \textit{memory bank} $M$ to store the representations for every image.
The $i$-th representation $M[i]$ is updated using a linear combination of the stored entry and a new representation every epoch: $M[i] = \alpha * M[i] + (1-\alpha) * g_\theta(x^{(i)})$ where $\alpha \in [0, 1)$.
Second, the authors approximate $\sum_{j=1}^N e^{g_\theta(x^{(i)})^T M[j] / \omega} \approx \kappa\sum_{j=1}^K e^{g_\theta(x^{(i)})^T M[i_j] / \omega}$  where each $i_j$ is sampled uniformly from $[N]$, $K \ll N$, and $\kappa$ is a hardcoded constant.

\textbf{Local Aggregation (LA)}
The goal of IR is to learn representations such that it is equally easy to discriminate any example from the others.
However, such uniformity may be undesirable: images of the same class should intuitively be closer in representation than other images.
With this as motivation, LA \cite{zhuang2019local} seeks to pull ``nearby'' images closer while  ``pushing'' other images away:
\begin{equation}
  \mathcal{L}^{\text{LA}}(x^{(i)}, M) = \log \frac{p(C^{(i)} \cap B^{(i)} | x^{(i)}, M)}{p(B^{(i)} | x^{(i)}, M)} \text{ where } p(I|x, M) = \sum_{i \in I}p(i|x, M),
  \label{eqn:la}
\end{equation}
and $I$ is any set of indices.
Given the $i$-th image $x^{(i)}$, its \textit{background neighbor set} $B^{(i)}$  contains the $K$ closest examples to $M[i]$ in embedding space.
Second, the \textit{close neighbor set} $C^{(i)}$, contains elements that belong to the same cluster as $M[i]$ where clusters are defined by K-means on embeddings.
In practice, $C^{(i)}$ is a subset of $B^{(i)}$.
Throughout training, the elements of $B^{(i)}$ and $C^{(i)}$ change. LA outperforms IR by 6\% on the transfer task of ImageNet classification.

\textbf{Contrastive Multiview Coding (CMC)} CMC \cite{tian2019contrastive} adapts IR to decompose an input image into the luminance (L) and AB-color channels.
Then, CMC is the sum of two IR objectives where the memory banks for each modality are \textit{swapped}, encouraging the representation of the luminance of an image to be ``close'' to the representation of the AB-color of that image, and vice versa:
\begin{equation}
    \mathcal{L}^{\textup{CMC}}(x^{(i)}, M) = \mathcal{L}^{\textup{IR}}(x^{(i)}_{\textup{L}}, M_{\textup{ab}}) + \mathcal{L}^{\textup{IR}}(x^{(i)}_{\textup{ab}}, M_{\textup{L}})
    \label{eq:cmc}
\end{equation}
In practice, CMC outperforms IR by almost 10\% in ImageNet classification.

\subsection{An Unexpected Reliance on Data Augmentation}
As is standard practice, IR, LA, and CMC transform images with a composition of cropping, color jitter, flipping, and grayscale conversion.
In the original papers, these algorithms do not emphasize the importance of data augmentation:
in Eqs.~\ref{eqn:ir}-~\ref{eq:cmc},
the encoder $g_\theta$ is assumed to act on the image $x^{(i)}$ directly, although in reality it acts on a transformed image.
Despite this, we conjecture that without data augmentation, contrastive learning would not enjoy the success it has found in practice.

Without augmentations, the IR objective is a function only of the embeddings of data points, $g^{(i)} = g_\theta(x^{(i)})$. For an L$^2$ normalized encoder  $g_\theta$, the objective pushes these embedded points towards a uniform distribution on the surface of a sphere. Thus IR does not distinguish the locations of the data points themselves so long as they are uniformly distributed; it is \textit{permutation-invariant}:

\begin{thm}
  Fix a random variable $X$ with $N$ realizations $\mathcal{D} = \{ x^{(i)},...,x^{(N)} \}$. Enumerate an embedding $g^{(1)},...,g^{(N)} \in \mathbb{R}^d$ of $\mathcal{D}$. Then for any permutation $\pi$ of $[N]$ (with notation as in Eq.~\ref{eqn:ir}):
  \begin{equation*}
    \scriptstyle
    \sum_{i} \log \left(
    \frac{e^{({g^{(i)}})^T M[i] / \omega}}{\sum_{j=1}^N e^{({g^{(i)}})^T M[j] / \omega}} \right)
    =
    \sum_{i} \log \left(
    \frac{e^{({g^{\pi(i)}})^T M[\pi(i)] / \omega}}{\sum_{j=1}^N e^{{(g^{\pi(i)}})^T M[\pi(j)] / \omega}} \right)
  \end{equation*}
 In particular, the optima of the IR objective, without augmentation, are invariant under permutation of the embedding vectors of data points.
  \label{thm:marketing}
\end{thm}

\begin{wrapfigure}{l}{0.4\textwidth}
  \centering
  \begin{subfigure}[b]{0.19\textwidth}
    \centering
    \includegraphics[width=\textwidth]{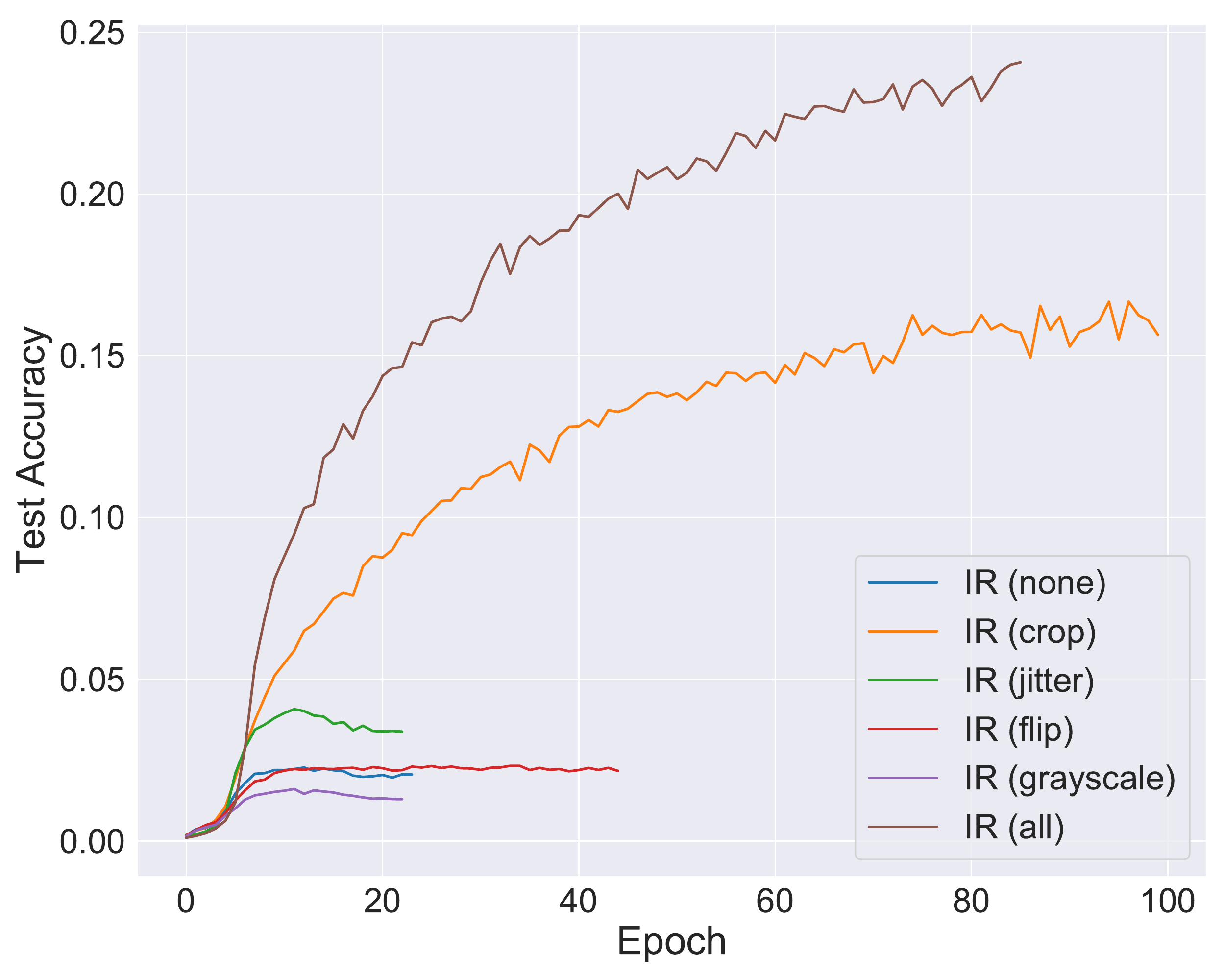}
    \caption{IR)}
    \vspace{0.15em}
  \end{subfigure}
    \centering
  \begin{subfigure}[b]{0.19\textwidth}
    \centering
    \includegraphics[width=\textwidth]{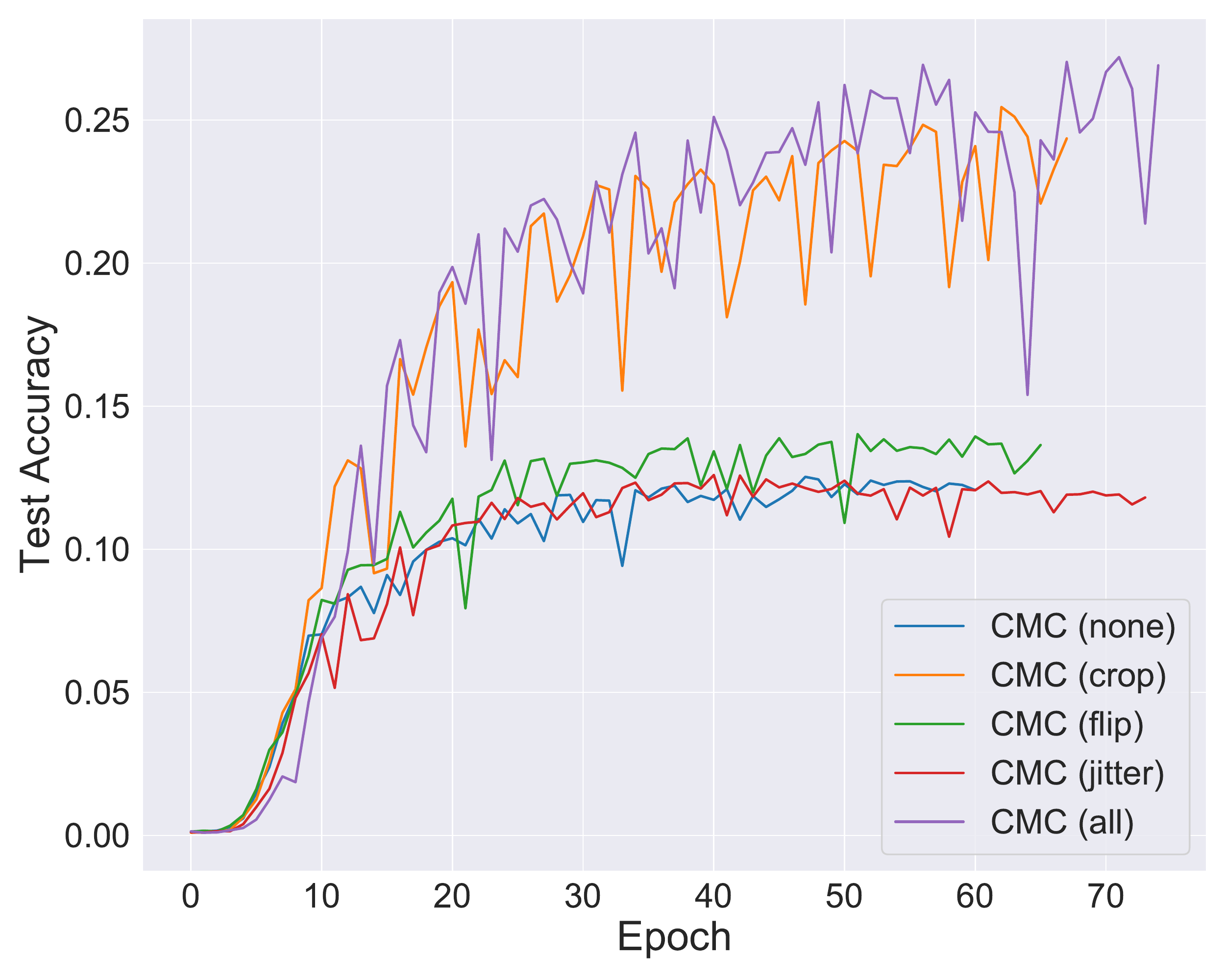}
    \caption{CMC)}
  \end{subfigure}
  \caption{Effect of choice of views on representation quality for IR and CMC on ImageNet. See Fig.~\ref{fig:viewset_lesion_dup} for CIFAR10.}
  \label{fig:viewset_lesion}
\end{wrapfigure}

Now, succeeding on a transfer task requires the embeddings of data points with similar task labels to be close.
While there exists a permutation such that examples of the same class are placed next to each other, there are many others where they are not.
It is only with data augmentation that the invariance of Thm.~\ref{thm:marketing} is broken: if different data points can yield the same view (a ``collision'') then their embeddings cannot be arbitrarily permuted.
This implies that augmentation is the crux of learning a good representation. Moreover, not all augmentations are equally good:
permutation invariance will be reduced most when the graph of shared views is connected but not dense.

The discussion so far idealizes the embedding functions by ignoring their intrinsic biases: architecture and implicit regularization from SGD may favor certain optima.
In the non-idealized setting, IR with no augmentations may already produce non-trivial representations.
We turn to experimental evidence to determine the impact of these intrinsic biases compared to those introduce by data augmentation.
We train IR and CMC on ImageNet and CIFAR10 with different subsets of augmentations.
To measure the quality of the representation, for every image in the test set we predict the label of its closest image in the training set \cite{zhuang2019local}.

Fig.~\ref{fig:viewset_lesion} clearly shows that without augmentations (the blue line), the representations learned by IR and CMC are significantly worse (though not trivial) than with all augmentations (the brown line). Further, we confirm that not all augmentations lead to good representations: collision-free transformations (e.g. flipping, grayscale) are on par with using no augmentations at all.
But, cropping, which certainly introduces collisions, accounts for most of the benefit amongst the basic image augmentations.

Having highlighted the importance of data augmentations, we restate IR to explicitly include them.
Define a \textit{view function}, $\nu: X \times \mathcal{A} \rightarrow X$ that maps the realization of a random variable and an index to another realization.
Each index $a$ in $\mathcal{A}$ is associated with a composition of augmentations (e.g. cropping plus jitter plus rotation). Assume $p(a)$ is uniform over $\mathcal{A}$. The IR objective becomes
\begin{equation}
    \scriptstyle
  \mathcal{L}^{\text{IR}}(x^{(i)}, M) = \mathbb{E}_{p(a)}\left[ \log \frac{e^{g_\theta(\nu(x^{(i)}, a))^T M[i] / \omega}}{\sum_{j=1}^N e^{g_\theta(\nu(x^{(i)}, a))^T M[j] / \omega}} \right],\label{eqn:ir_new}\\
\end{equation}
where the representation $M[i]$ of the $i$-th image $x^{(i)}$ is updated at the $m$-th epoch by the equation
$M[i] = \alpha * M[i] + (1 - \alpha) * g_\theta(\nu(x^{(i)}, a_m)) \text{ for }a_m \sim p(a)$.
We will show that Eq.~\ref{eqn:ir_new} is equivalent to a bound on mutual information \emph{between views}, explaining the importance of data augmentations. To do so, we first rehearse a commonly used lower bound on mutual information.

\section{Equivalence of Instance Discrimination and Mutual Information}
\label{sec:mi_rep}

Connections between mutual information and representation learning have been  suggested for a family of masked language models \cite{kong2019mutual}.
However, the connection has not been deeply explored and a closer look in the visual domain uncovers several insights surrounding contrastive learning.

Mutual information (MI) measures the statistical dependence between two random variables, $X$ and $Y_1$.
In lieu of infeasible integrals, a popular approach is to lower bound MI with InfoNCE \cite{poole2019variational,gutmann2010noise,oord2018representation}:
\begin{equation}
    \scriptstyle
  \text{MI}(X;Y_1) \geq
  \mathcal{I}^{\textup{NCE}}(X;Y_1) = \mathcal{I}^{\textup{NCE}}(X;Y_1, Y_{2:K}) =  \mathbb{E}_{p(x,y_1)}\left[f_{\theta,\phi}(x,y_1) - \mathbb{E}_{p(y_{2:K})}\left[\log \frac{1}{K} \sum_{j=1}^K e^{f_{\theta,\phi}(x, y_j)} \right]\right]
  \label{eqn:infonce}
\end{equation}
where $Y_{2:K}$ are $K-1$ independent copies of $Y_1$. The \textit{witness function} $f_{\theta,\phi}(x,y) = g_\theta(x)^T g_\phi(y)$ measures the ``compatibility'' of $x$ and $y$.
We use $g_\theta$ and $g_\phi$ to designate
encoders that map realizations of $X$ and $Y_i$ to vectors.
The rightmost term in Eq.~\ref{eqn:infonce} serves to normalize $f_{\theta,\phi}(x,y)$ with respect to other realizations of $Y_1$.
We use $y_{2:K} = \{y_2, \ldots, y_K \}$ to denote a set of \textit{negative samples}, or realizations of $Y_{2:K}$.
The standard choice for $p(y_{2:K})$ is $\prod_{i=2}^K p(y_i)$.
Note we use subscripts $_i$ to index values of random variables whereas we used superscripts $^{(i)}$ in Sec.~\ref{sec:background} to index a dataset.

In machine learning, we often treat examples from a dataset $\mathcal{D}$ as samples from an empirical distribution $p_{\mathcal{D}}$. In Eq.~\ref{eqn:infonce}, if we replace $p(x, y_1)$ and $p(y_{2:K})$ with their empirical analogs, $p_{\mathcal{D}}(x, y_1)$ and $p_{\mathcal{D}}(y_{2:K})$, we call this estimator $\hat{\mathcal{I}}^{\text{NCE}}(X; Y_1)$, or empirical InfoNCE.

\subsection{Equivalence of IR and CMC to InfoNCE}
\label{sec:exist:algo}
To move Eq.~\ref{eqn:infonce} toward Eq.~\ref{eqn:ir_new},
consider lower bounding the MI between two \textit{weighted view sets} with InfoNCE.
For views $\nu^*: X \times 2^\mathcal{A} \rightarrow 2^X$ of index \textit{sets} $A \subseteq \mathcal{A}$, and a set of $|A|$ weights $w$, define the extended encoder as $g^*_\theta(v^*(x, A), w) = \sum_{\substack{x_m \in v^*(x, A)}} w[m] \cdot g_\theta(x_m)$ where $w[m]$ fetches the $m$-th weight.
This encoder leads to the extended witness function $f^*_\theta(\nu^*(x, A), \nu^*(x', A'); w, w') = g^*_\theta(\nu^*(x, A), w)^T g^*_\theta(\nu^*(x', A'), w') / \omega$ for two realizations $x$, $x'$.
Treating weighted view sets as random variables derived from $X$, we express the InfoNCE bound as
\begin{align}
    \scriptstyle
    \mathcal{I}^{\text{NCE}}(X; \nu^*) &= \scriptstyle \mathbb{E}_{A,A'\sim p(A)}\left[\mathcal{I}^{\text{NCE}}(\nu^*(X, A), \nu^*(X, A'))\right] \\
    &= \scriptstyle \mathbb{E}_{p(x_{1:K})}\mathbb{E}_{A_{1:K},B,B' \sim p(A)}\left[\log \frac{e^{f^*_{\theta}(\nu^*(x_1,B), \nu^*(x_1, B'); w_B, w_{B'})}}{\frac{1}{K}\sum_{j=1}^K e^{f^*_{\theta}(\nu^*(x_1,B), \nu^*(x_j, A_j); w_B, w_{A_j}) }} \right].
    \label{eqn:infonce:view}
\end{align}

Our insight is to characterize the memory bank $M$ as a weighted view set. Given the $i$-th image $x^{(i)}$ and $M$ with update rate $\alpha$, we estimate the $i$-th entry as $M[i] = \sum_{n=1}^{|\mathcal{A}|} \alpha^{n-1} g_\theta(\nu(x^{(i)}, a_n)) \approx \sum_{m=1}^{n_\alpha} \alpha^{m-1} g_\theta(\nu(x^{(i)}, a_m))$.
The left expression enumerates over all elements of the index set $\mathcal{A}$ while the right sums over $a_m \sim p(a)$, the $m$-th sampled index from $\mathcal{A}$, over $n_\alpha$ epochs of training.
Because the contribution of any view to the memory bank entry exponentially decays, the first sum is tightly approximated by the second where $n_\alpha$ is a function of $\alpha$.
If $\alpha = 0$, then $n_\alpha = 1$.

For each example $x^{(i)}$, construct two index sets: $\{a\}$, the index sampled from $p(a)$ for the current epoch, and $A^{(i)} \subseteq \mathcal{A}$, the indices sampled over the last $n_\alpha$ epochs of training. Next, fix weights $w_1 = \{ 1 \}$ and $w_2 = \{ 1, \alpha, \ldots, \alpha^{{n_\alpha}-1} \}$. Consider the special case of Eq.~\ref{eqn:infonce:view} where the witness function is $f^*_\theta(\nu^*(x^{(i)}, \{a\}), \nu^*(x^{(i)}, A^{(i)}); w_1, w_2)$, the weighted dot product between the current view of $x^{(i)}$ and the stored representation from the bank.
The next lemma shows that this special case of InfoNCE is equivalent to IR and CMC. The proof can be found in Sec.~\ref{sec:prooflemir_mi}.

\begin{lem}
    Let $p_\mathcal{D}(x)$ be an empirical distribution over images $x$, $\mathcal{D}$ a dataset whose members are sampled i.i.d. from  $p_\mathcal{D}$, and $M$ a memory bank. Let $\mathcal{A}_{\textup{IR}}$ and $\mathcal{A}_{\textup{CMC}}$ be two sets of indices where each index is associated with an image augmentation in IR and an image augmentation in CMC, respectively. Now, define two extended view functions, $\nu^*_{\textup{IR}}: X \times 2^{\mathcal{A}_{\textup{IR}}} \rightarrow 2^X$ and $\nu^*_{\textup{CMC}}: X \times 2^{\mathcal{A}_{\textup{CMC}}} \rightarrow 2^X$. If we use $K$ negative samples, then
    $\hat{\mathcal{I}}^{\textup{NCE}}(X; \nu^*_{\textup{IR}}) \equiv \mathbb{E}_{p_\mathcal{D}}\left[\mathcal{L}^{\textup{IR}}(x, M)\right] + \log \frac{K}{\kappa}$, and $\hat{\mathcal{I}}^{\textup{NCE}}(X; \nu^*_{\textup{CMC}}) \equiv \mathbb{E}_{p_\mathcal{D}}\left[\mathcal{L}^{\textup{CMC}}(x, M)\right]/2 + \log \frac{K}{\kappa}$ where the symbol $\equiv$ implies equivalence.
  \label{lem:ir_mi}
\end{lem}
 Lemma~\ref{lem:ir_mi} provides insight into IR and CMC: the objectives both lower bound MI but using different view sets.
As L and ab capture nearly disjoint information, CMC imposes a strong information bottleneck between any two views. Thus, we see why in Fig.~\ref{fig:viewset_lesion}, CMC maintains a higher performance without any augmentations compared to IR: the view set for CMC is always nontrivial.

\subsection{Simplifying Instance Discrimination}
\label{sec:mi_insights}

The equivalence to mutual information can help us pick hyperparameters and simplify IR.

\begin{figure}[h!]
  \centering
  \begin{subfigure}[b]{0.15\textwidth}
    \centering
    \includegraphics[width=\textwidth]{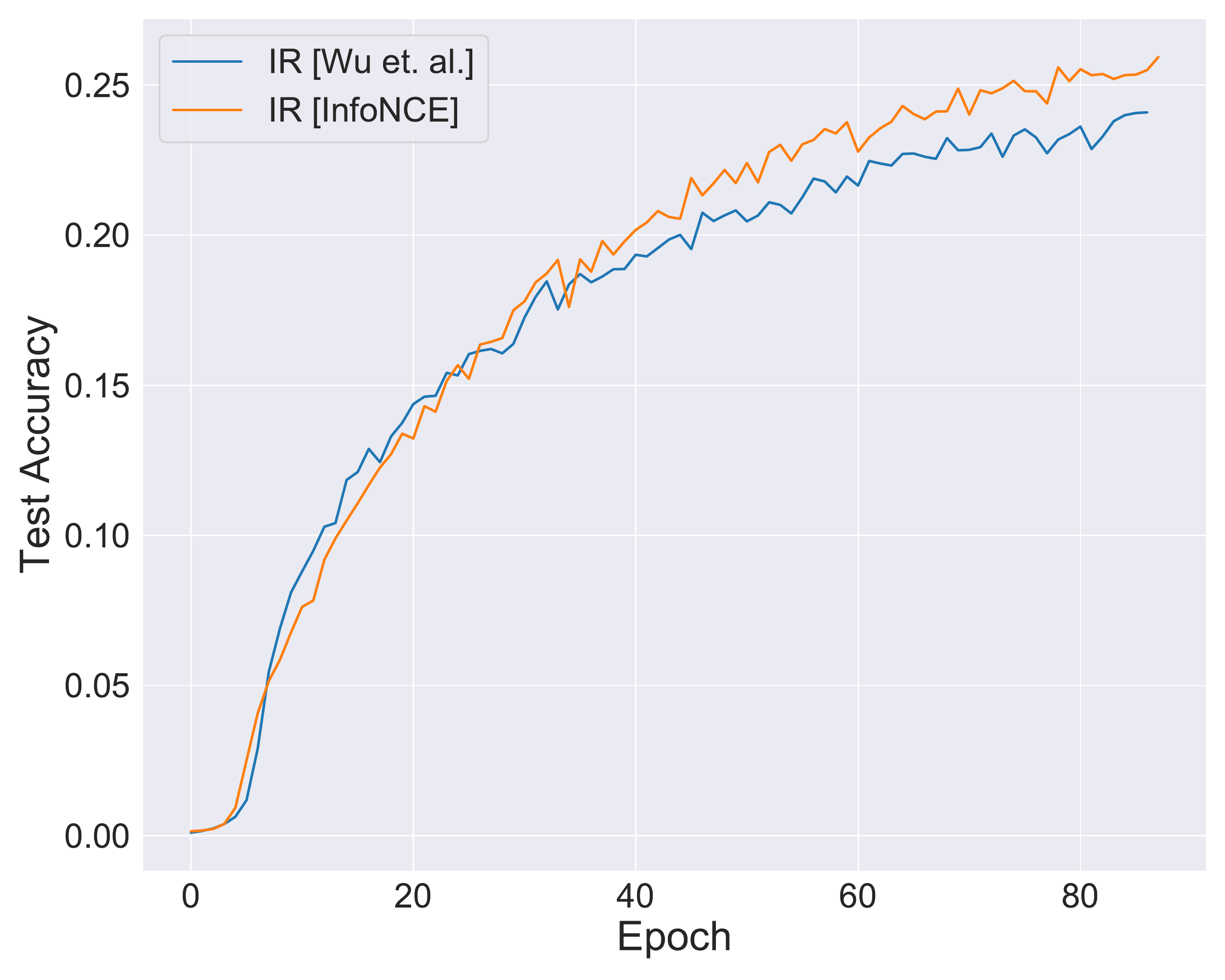}
    \caption{Stability ($\dagger$)}
  \end{subfigure}
  \begin{subfigure}[b]{0.15\textwidth}
    \centering
    \includegraphics[width=\textwidth]{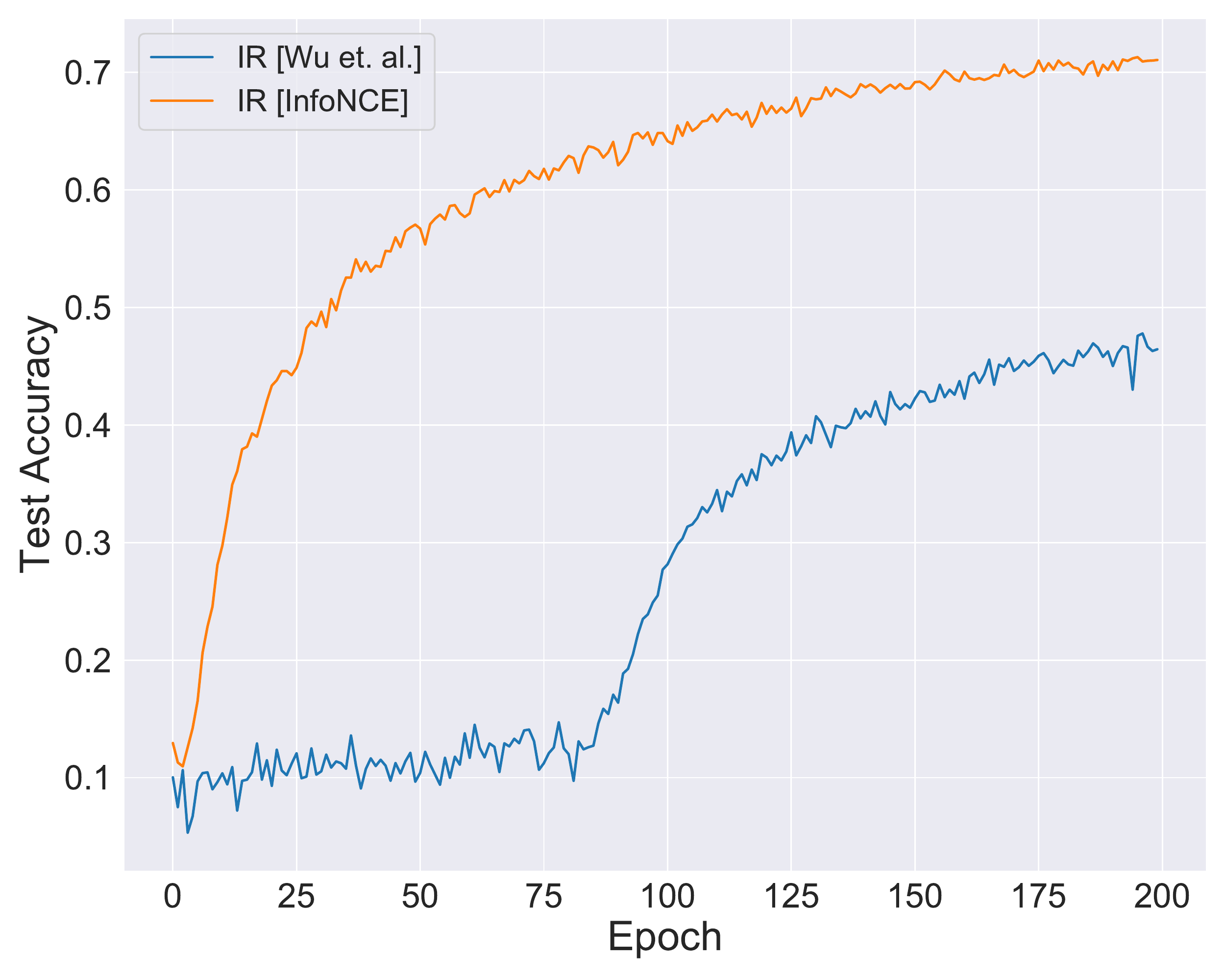}
    \caption{Stability ($\ddagger$)}
  \end{subfigure}
  \begin{subfigure}[b]{0.15\textwidth}
    \centering
    \includegraphics[width=\textwidth]{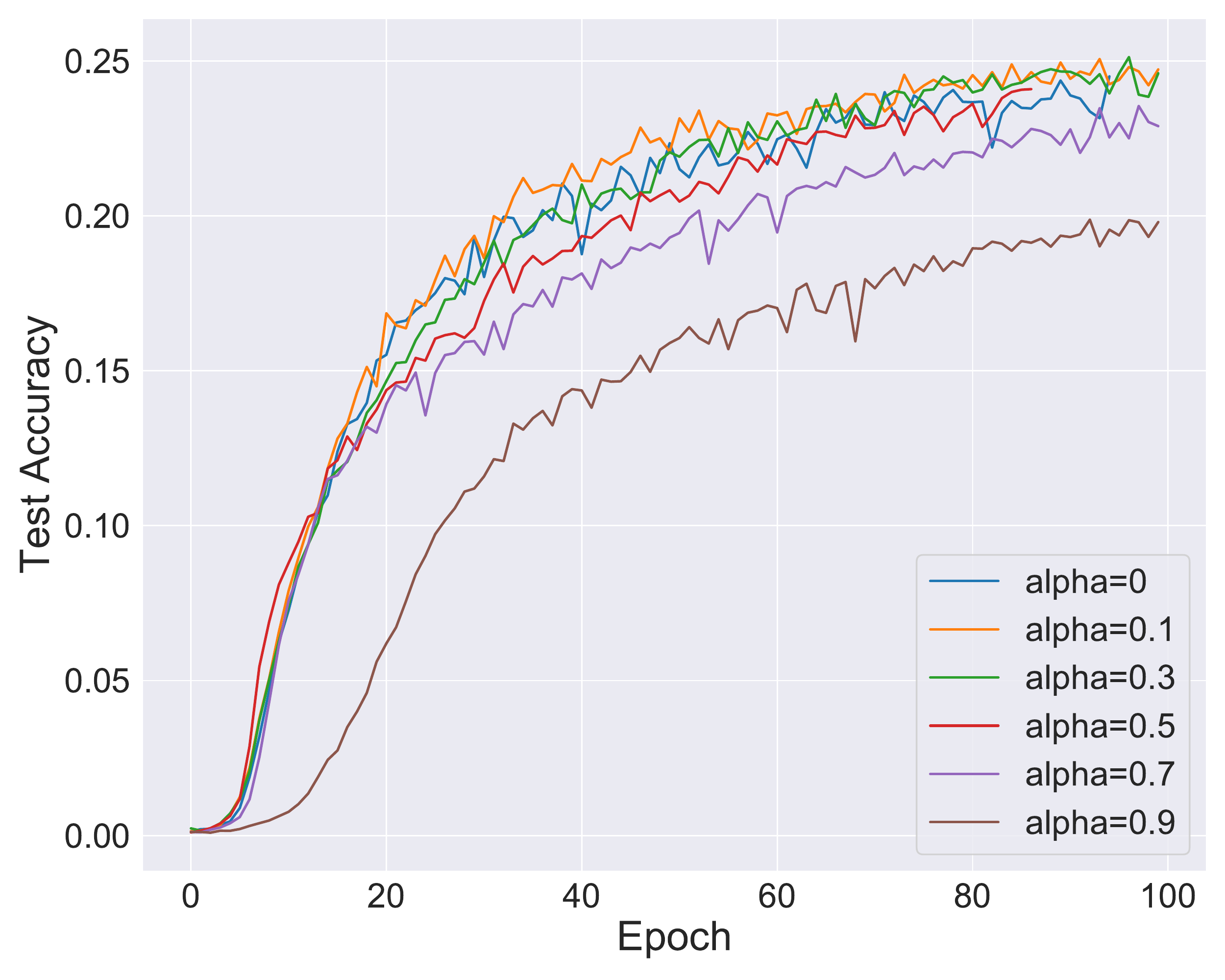}
    \caption{$\alpha$+IR ($\dagger$)}
  \end{subfigure}
  \begin{subfigure}[b]{0.15\textwidth}
    \centering
    \includegraphics[width=\textwidth]{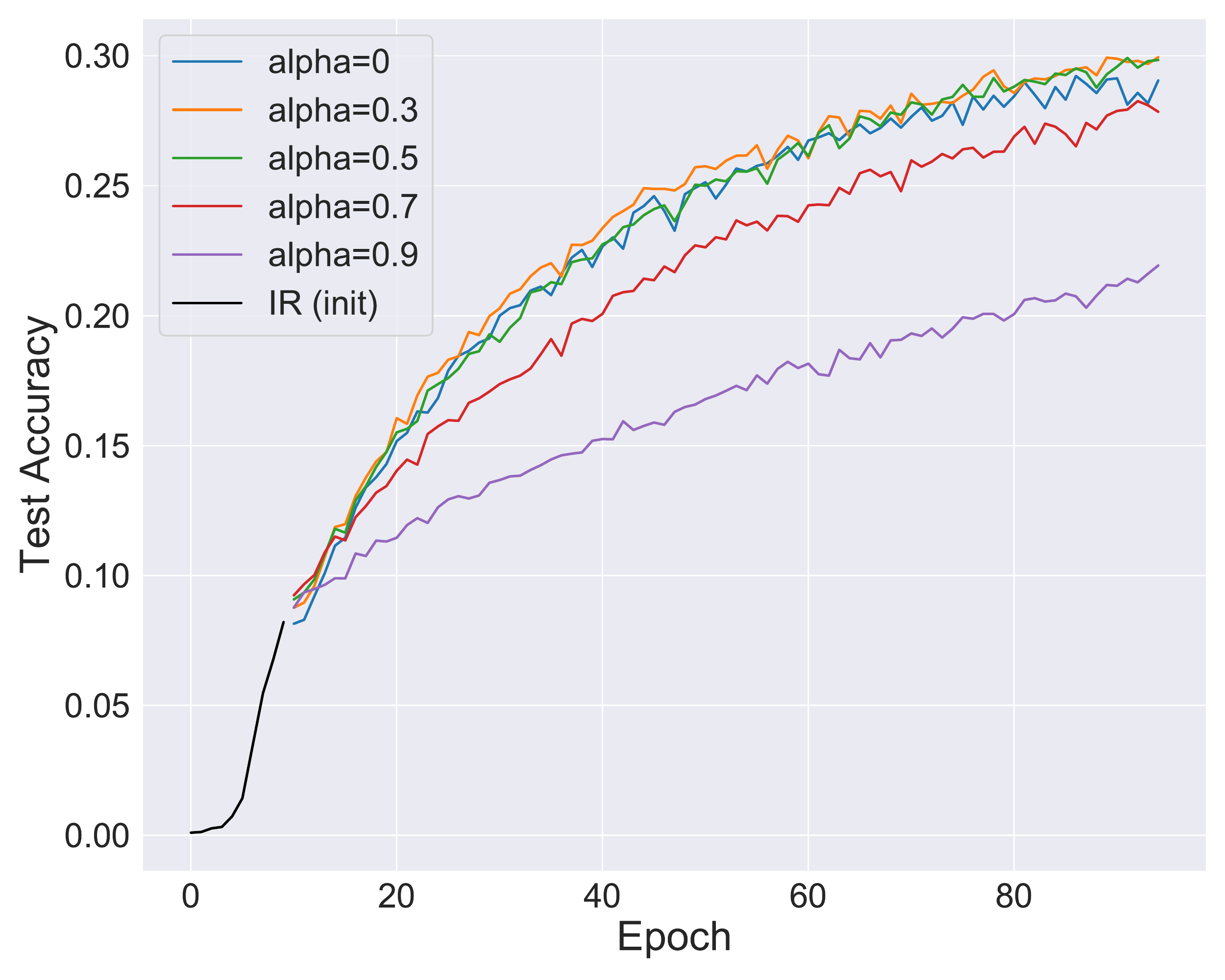}
    \caption{$\alpha$+LA ($\dagger$)}
  \end{subfigure}
    \begin{subfigure}[b]{0.15\textwidth}
    \centering
    \includegraphics[width=\textwidth]{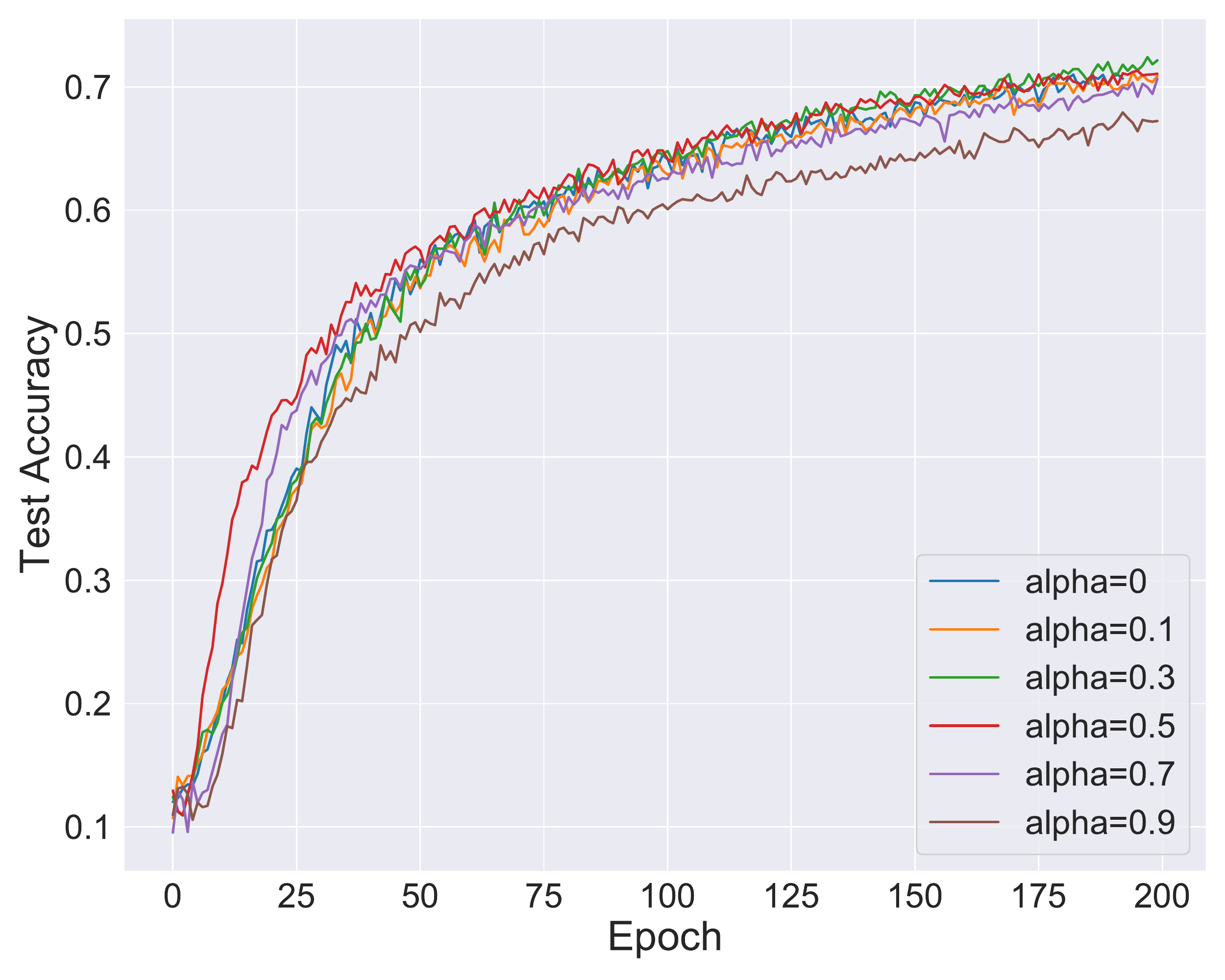}
    \caption{$\alpha$+IR ($\ddagger$)}
  \end{subfigure}
  \begin{subfigure}[b]{0.15\textwidth}
    \centering
    \includegraphics[width=\textwidth]{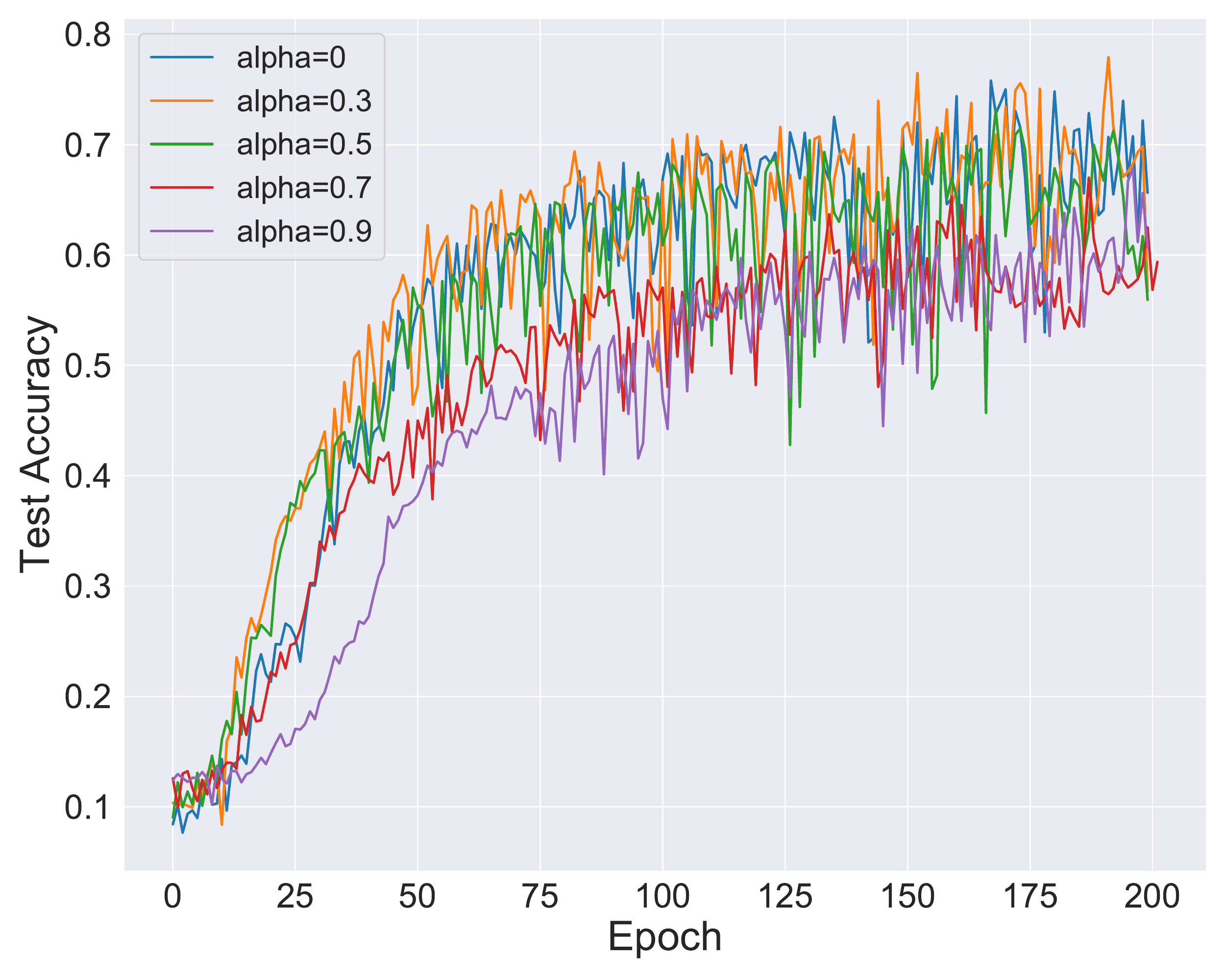}
    \caption{$\alpha$+CMC ($\ddagger$)}
  \end{subfigure}
  \caption{Nearest neighbor classification accuracy comparing IR with InfoNCE in stability (a,b), and different $\alpha$ for the memory bank of IR, LA (e-h) on Imagenet ($\dagger$) and CIFAR10 ($\ddagger$).}
  \label{fig:discussion}
\end{figure}

\textbf{Softmax ``hacks'' are unnecessary.}$\quad$
The original IR implementation used several innovations to make a million-category classification problem tractable.
For instance, IR approximates the denominator in Eq.~\ref{eqn:ir} with $K$ samples scaled by $\kappa = \frac{2876934.2}{1281167}$.
Such practices were propagated to LA and CMC in their official implementations.
While this works well for ImageNet, it is not clear how to set the constant for other datasets. For small datasets like CIFAR10, such large constants introduce numerical instability in themselves.
However, once we draw the relation between IR and MI, we immediately see that \textit{there is no need to compute the softmax} (and no need for $\kappa$); InfoNCE only requires a \texttt{logsumexp} over the $K$ samples, a much more stable operation.
Fig.~\ref{fig:discussion}c and d show the effect of switching from the original IR code to InfoNCE. While we expected to find the large impact in CIFAR10, we also find that in ImageNet (for which the constants were tuned) our simplification improves performance.
Later, we refer to this ``simplified IR'', as IR$^{\text{nce}}$.

\textbf{The memory bank is not critical.}$\quad$
The memory bank makes Eq.~\ref{eqn:infonce:view} quite involved.
Yet the MI between weighted views should be very similar to the average MI between the underlying views.
We thus consider the special case when $\alpha = 0$.
Here, the $i$-th entry $M[i]$ stores only the encoding of a single view chosen in the last epoch.
As such, we can simplify and \textit{remove the memory bank altogether}.
We compare these formulations experimentally by nearest neighbor classification. Fig.~\ref{fig:discussion}e-g show results varying $\alpha$ from 0 (no memory bank) to near 1 (a very slowly updated bank).
We find performance when $\alpha=0$ and when $\alpha=0.5$ (the standard approach) is  equal across algorithms and datasets.
This suggests that we can replace $M$ with two random views every iteration.
SimCLR \cite{chen2020simple} has suggested a similar finding -- and the MI framework makes such a simplification natural.
With this formulation, we show SimCLR is equivalent to InfoNCE as well (see lemma~\ref{lem:simclr}).
In the experiments below we keep the memory bank to make comparison to previous work easier.

\section{Equivalence of Local Aggregation and Mutual Information}
\label{sec:lavince}
Having shown an equivalence between IR, CMC, SimCLR, and InfoNCE, we might wish to do the same for LA.
However, the close and background neighbor sets of LA are not obviously related to MI.
To uncover the relation, we generalize InfoNCE to support sampling from a variational distribution.

\subsection{A Variational Lower Bound on Mutual Information}
\label{sec:vince}
Recall that InfoNCE draws negative samples independently from $p(y)$. However,
we may wish to choose negative samples from a different distribution over sets $q(y_{2:K})$ or even conditionally sample negatives, $q(y_{2:K}|y_1)$ where $y_1 \sim p(y)$.
While previous literature presents InfoNCE with an arbitrary distribution \cite{oord2018representation,kong2019mutual} that would justify either of these choices, we have not found a proof supporting this.
One of our contributions is to formally define a class of variational distributions $q(y_{2:K}|y_1)$ such that Eq.~\ref{eqn:infonce} remains a valid lower bound if we replace $p(y_{2:K})$ with $q(y_{2:K}|y_1)$:

\begin{thm}
Fix a distribution $\mathbb{P}$ over $(\mathbb{R}^d, \mathcal{B}_{\mathbb{R}^d})$. Fix any $x^* \in \mathbb{R}^d$ and any $f: \mathbb{R}^d \times \mathbb{R}^d \rightarrow \mathbb{R}$. Define $g(x) = e^{f(x^*, x)}$  and suppose that $g$ is $\mathbb{P}$-integrable with mean $c$. Pick a set $T\subset \mathbb{R}$ lower-bounded by some $\tau > \log c$ in $\overline{\mathbb{R}}$ and define $S_{T} = \{x | f(x^*,x) \in T\}$ to be the pre-image under $f(x^*,\cdot)$ of $T$; suppose that $\mathbb{P}(S_{T})> 0$. Define $\mathbb{Q}_{T} (A) = \mathbb{P}(A | S_{T})$ for any Borel $A$. Then $\mathbb{E}_\mathbb{P}[g(x)] < \mathbb{E}_{\mathbb{Q}_{T}}[g(x)]$.
\label{thm:vince}
\end{thm}
Next, we define the variational InfoNCE, or VINCE, estimator and  use Thm.~\ref{thm:vince} to show that it lower-bounds
InfoNCE, and thus mutual information.
\begin{coro}
 Fix $X$ and $Y_1,...,Y_K$, the latter i.i.d. according to a distribution $p$. Suppose we sample $x, y_1 \sim p(x, y_1)$. For $f: (Y_i, Y_j)\rightarrow \R$ define $S_T = \{ y | \tau \leq f(y_1,y) \in T \}$ with $\tau > \log \mathbb{E}_{p(y)}[e^{f(y_1,y)}]$ and $T$ lower bounded by $\tau$. For any Borel $A = A_2 \times....\times A_K$, define a  distribution over $Y_{2:K}= (Y_2,...,Y_K)$ by $q_T(Y_{2:K} \in A) = \prod_{j=2}^K p(A_j | S_T)$.
 If we let $\mathcal{I}^{\text{VINCE}}(X; Y_1) =  \mathbb{E}_{p(x,y_1)}\mathbb{E}_{y_{2:K}\sim q_T}\left[\log \frac{e^{f(x, y_1)}}{\frac{1}{K}\sum_{j=1}^K e^{f(x, y_j)}}\right]$, then $ \mathcal{I}^{\text{VINCE}} \leq \mathcal{I}^{\text{NCE}}$.
  \label{coro:vince}
\end{coro}
With Coro.~\ref{coro:vince}, we are equipped to formalize the connection between LA and information.

\subsection{Equivalence of LA and VINCE}
\label{sec:lavince}

Focusing first on the background neighborhood, consider sampling from $q_{[\tau,\infty)}(x|\nu(x_1,a))$ as defined in Thm~\ref{thm:vince} with $x_1$ being the current image and $a$ the current index in $\mathcal{A}$.
A larger threshold $\tau$ chooses negatives that more closely resemble $g_\theta(\nu(x_1,a))$, the representation of the current view of $x_1$.
This poses a more challenging problem: encoders must now distinguish between more similar objects, forcing the representation to be more semantically meaningful.
Replacing $p(x_1)$ with $q_{[\tau,\infty)}(x|\nu(x_1,a))$ immediately suggests a new algorithm that ``interpolates'' between IR and LA.

\begin{lem}
Let $X_1$ be a random variable and $X_{2:K}$ be i.i.d. copies. Let  $\mathcal{D}$ be a dataset of realizations of $X_1$ sampled from an empirical distribution $p_{\mathcal{D}}$. Fix a realization $x_1 \in \mathcal{D}$, an index $a \in \mathcal{A}$, and a view function $\nu$. Assuming $\tau > \log \mathbb{E}_{x \sim p(x)}\left[ e^{f_\theta(\nu(x_1,a), x)} \right]$, we define  $q_T(x_{2:K}|\nu(x_1,a)) = \prod_{j=2}^K q_T(x_j|\nu(x_1,a))$ for any set $T$ lower bounded by $\tau$, from which we draw  realizations of $X_{2:K}$. Assume $x_j \in \mathcal{D}$ for $j\in [K]$ and let $\rho(x_j)$ return the index of $x_j$ in $\mathcal{D}$.
For any $T$, we define $\scriptstyle\mathcal{L}^T(x_1, M) = \mathbb{E}_{p(a)}\mathbb{E}_{q_T(x_{2:K}|\nu(x_1, a))}\left[\log\frac{e^{g_\theta(\nu(x_1, a))^T M[\rho(x_1)]/\omega}}{\sum_{j=1}^K  e^{g_\theta(\nu(x_1, a))^T M[\rho(x_j)]/\omega}}\right]$. Then  $\hat{\mathcal{I}}^{\textup{VINCE}}(X_1; \nu^*) \equiv \mathbb{E}_{p_\mathcal{D}(x)}[\mathcal{L}^T(x, M)] + \log K$. In particular, we call $\mathcal{L}^{\textup{BALL}}(x_1, M) = \mathcal{L}^{[\tau, \infty)}(x_1, M)$, Ball Discrimination (BALL).
 \label{lem:coro_ballann}
\end{lem}

The primary distinction of BALL from IR is that negative samples are drawn from a restricted domain. Thus, we cannot equate the BALL estimator to InfoNCE; we must rely on VINCE, which provides the machinery to use a conditional distribution with smaller support.
We will use BALL to show that LA lower bounds MI between weighted view sets.
To do so, consider a simplified version of LA where we assume that the close neighbor set contains only the current image. That is, $C^{(\rho(x_1))} = \{\rho(x_1)\}$ in the notation of Sec.~\ref{sec:background}. We call this LA$_0$. We show that LA$_0$ is equivalent to BALL.
\begin{lem}
Fix $x \sim p_{\mathcal{D}}(x)$ and $C^{(\rho(x))} = \{\rho(x)\}$. Then $\mathcal{L}^{\textup{LA}_0}(x, M) = \mathcal{L}^{\textup{BALL}}(x, M)$.
  \label{lem:la_ml}
\end{lem}

Now, the numerator of the LA$_0$ objective contains a term  $e^{g_\theta(\nu(x^{(i)}, a))^T M[i]}$ whereas the numerator of the more general LA is $\sum_{k \in C^{(i)}} e^{g_\theta(\nu(x^{(i)}, a))^T M[k]}$. Note that the dot product between embeddings of any two views of $x^{(i)}$ must be at least as large as the dot product between embeddings of a view of $x^{(i)}$ and a view of $x^{(j)}$, with equality if and only if views collide. Since $|C^{(i)}|e^{g_\theta(\nu(x^{(i)}, a))^T M[i]} \geq \sum_{k \in C^{(i)}} e^{g_\theta(\nu(x^{(i)}, a))^T M[k]}$, we can bound LA by LA$_0$ subject to an additive constant.

\begin{lem} Fix $x \sim p_{\mathcal{D}}(x)$ and $\rho(x) \in C^{(\rho(x))}$. Then $\mathcal{L}^{\textup{LA}_0}(x, M) \geq \mathcal{L}^{\textup{LA}}(x, M) - \log |C^{(\rho(x))}|$.
\label{lem:la_la0}
\end{lem}

As LA is bounded by BALL, it is bounded by VINCE and thus MI as well.

We end with an observation about LA.
The overall objective for BALL includes an expectation over view pairs.
An elegant way to understand ``close neighbors'' in LA is as an extension to this view distribution: treat the views of all elements from $C^{(i)}$ as views of $x^{(i)}$.
The actual LA objective makes one additional change, applying Jensen's inequality to move the sum over close neighbors inside the log, yielding a tighter bound on BALL:
\begin{align*}
    \scriptstyle
    \mathbb{E}_{p(a)}\mathbb{E}_{q_{[\tau, \infty)}}\left[\log \frac{ \frac{1}{|C|}\sum_{l \in C } e^{g_\theta(\nu(x_1, a))^T M[l] }  }{ \sum_{j=1}^K e^{g_\theta(\nu(x_1, a))^T M[\rho(x_j)]} } \right] \scriptstyle \geq \mathbb{E}_{p(a)}\mathbb{E}_{q_{[\tau, \infty)}}\mathbb{E}_{q_{C}}\left[\sum_{l \in \{\rho({x'_l})\} } \log \frac{ e^{g_\theta(\nu(x_1, a))^T M[l] }  }{ \sum_{j=1}^K e^{g_\theta(\nu(x_1, a))^T M[\rho(x_j)]} } \right]
\end{align*}
where $x_{2:K} \sim q_{[\tau,\infty)}$, $q_C$ is uniform over close neighbors (to the current view), and $x'_{2:L} \sim q_{C}$. The left expression is $\mathcal{L}^{\textup{LA}}(x_1, M)$ and the right expression is the objective that extends the view set.

\subsection{Simplifying and Generalizing Local Aggregation}
As in Sec.~\ref{sec:mi_insights}, we can simplify LA after formulating it as mutual information. In particular, these insights lead to new variants of Ball Discrimination and Local Aggregation.


\textbf{Ring and Cave Discrimination.}
We consider two extensions of BALL, where the negative samples are picked from a ball centered around $g_\theta(\nu(x_1,a))$ for some $a$. Like LA, we might fear that the elements very close to $x_1$ in the ball would serve as poor negative examples. However, LA treats these elements as both close neighbors and negatives, which seems contradictory. Instead, we propose two ways to cut out a region in the ball of negative samples. First, take a smaller ball with the same center. Define Ring Discrimination (RING) as $\mathcal{L}^{\textup{RING}}(x_1, M) = \mathbb{E}_{p(a)}\mathbb{E}_{q_{[\tau,\gamma]}(x_{2:K}|\nu(x_1,a))}\left[\log \frac{e^{g_\theta(\nu(x_1, a))^T M[\rho(x_1)]/\omega}}{\sum_{j=1}^K e^{g_\theta(\nu(x_1, a))^T M[\rho(x_j)]/\omega}}\right]$. That is, RING samples negatives from a distribution $q_{[\tau,\gamma]}$ over a ring formed by the difference between the two balls.

Next, like LA, consider running K-means on samples, $x_1 \sim p_\mathcal{D}(x)$ and $x_{2:K} \sim q_{[\tau,\infty)}$ (the outer ball). Let $S_1$ contain elements with the same cluster assignment as $x_1$. Define CAVE Discrimination (CAVE) as $\mathcal{L}^{\textup{CAVE}}(x_1, M) = \mathbb{E}_{p(a)}\mathbb{E}_{q_{[\tau,\infty)\setminus S_1}(x_{2:K}|\nu(x_1,a))}\left[\log \frac{e^{g_\theta(\nu(x_1, a))^T M[\rho(x_1)]/\omega}}{\sum_{j=1}^K e^{g_\theta(\nu(x_1, a))^T M[\rho(x_j)]/\omega}}\right]$. That is, CAVE samples negatives from a distribution $q_{[\tau,\infty)\setminus S_1}$ over the outer ball subtracting a set.
Lemma~\ref{lem:ring_cave} show RING and CAVE lower bound MI between weighted view sets.
\vspace{-1em}
\begin{figure}[h!]
  \centering
  \begin{subfigure}[b]{0.13\textwidth}
    \centering
    \includegraphics[width=0.9\textwidth]{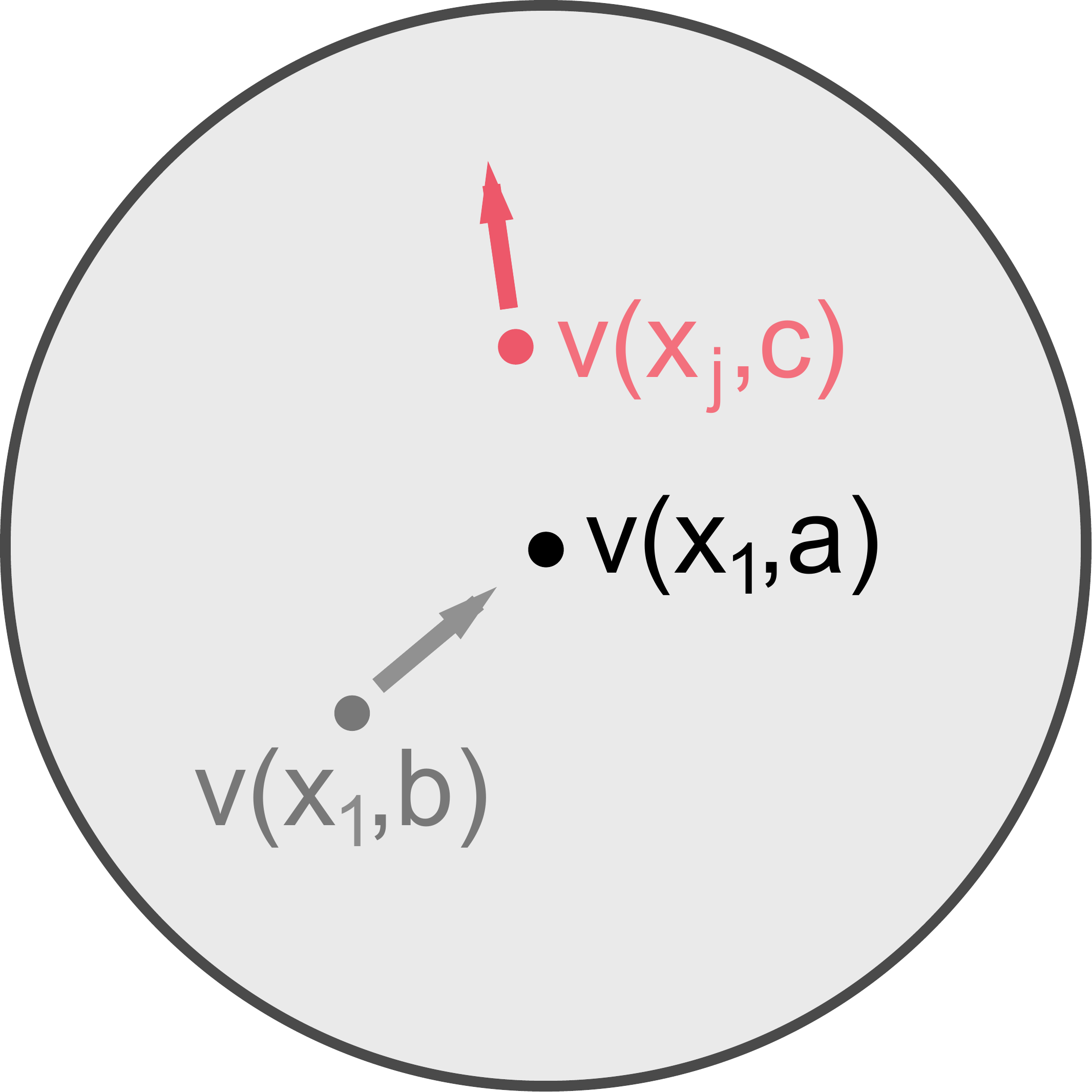}
    \caption{IR}
  \end{subfigure}
  \begin{subfigure}[b]{0.13\textwidth}
    \centering
    \includegraphics[width=0.9\textwidth]{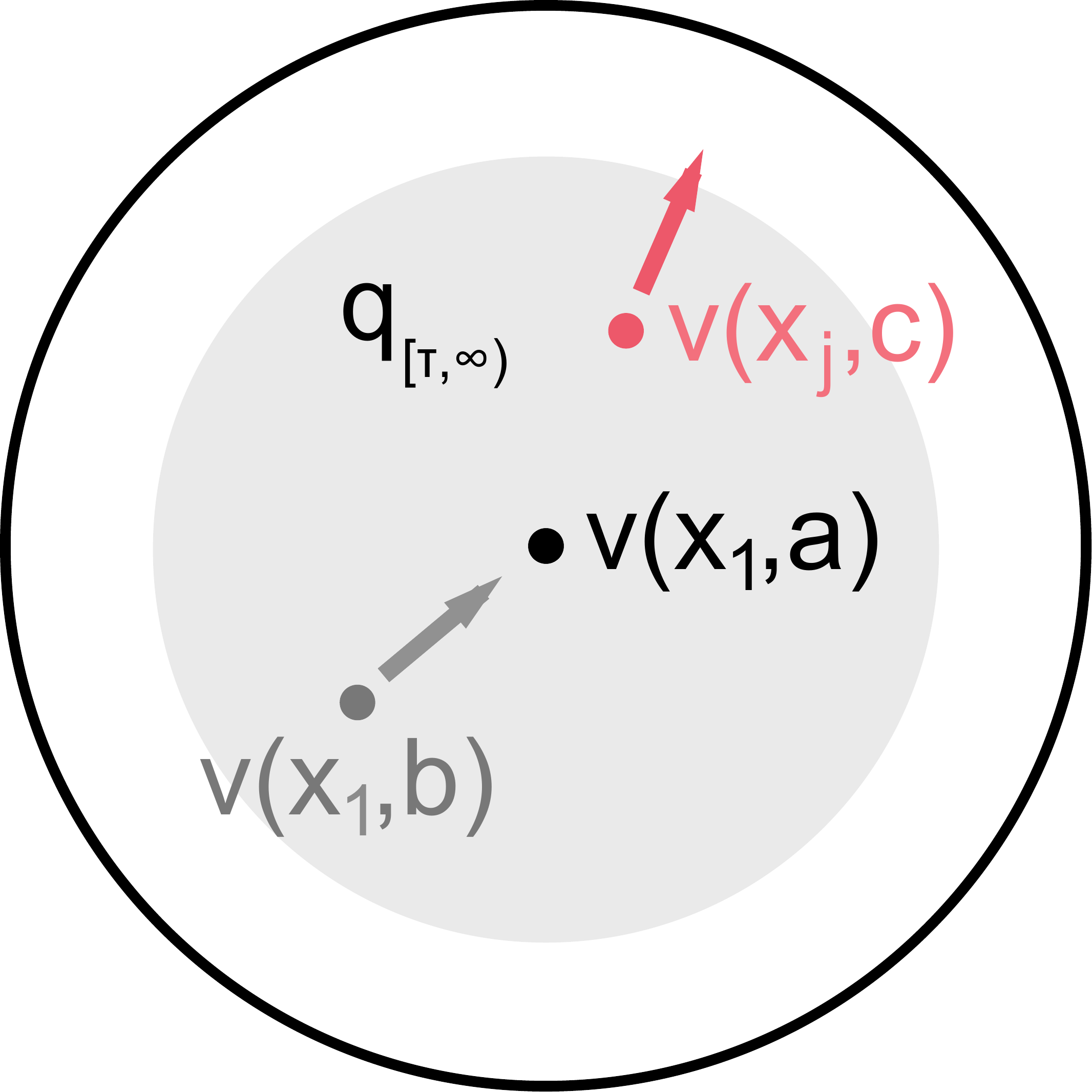}
    \caption{BALL}
  \end{subfigure}
  \begin{subfigure}[b]{0.13\textwidth}
    \centering
    \includegraphics[width=0.9\textwidth]{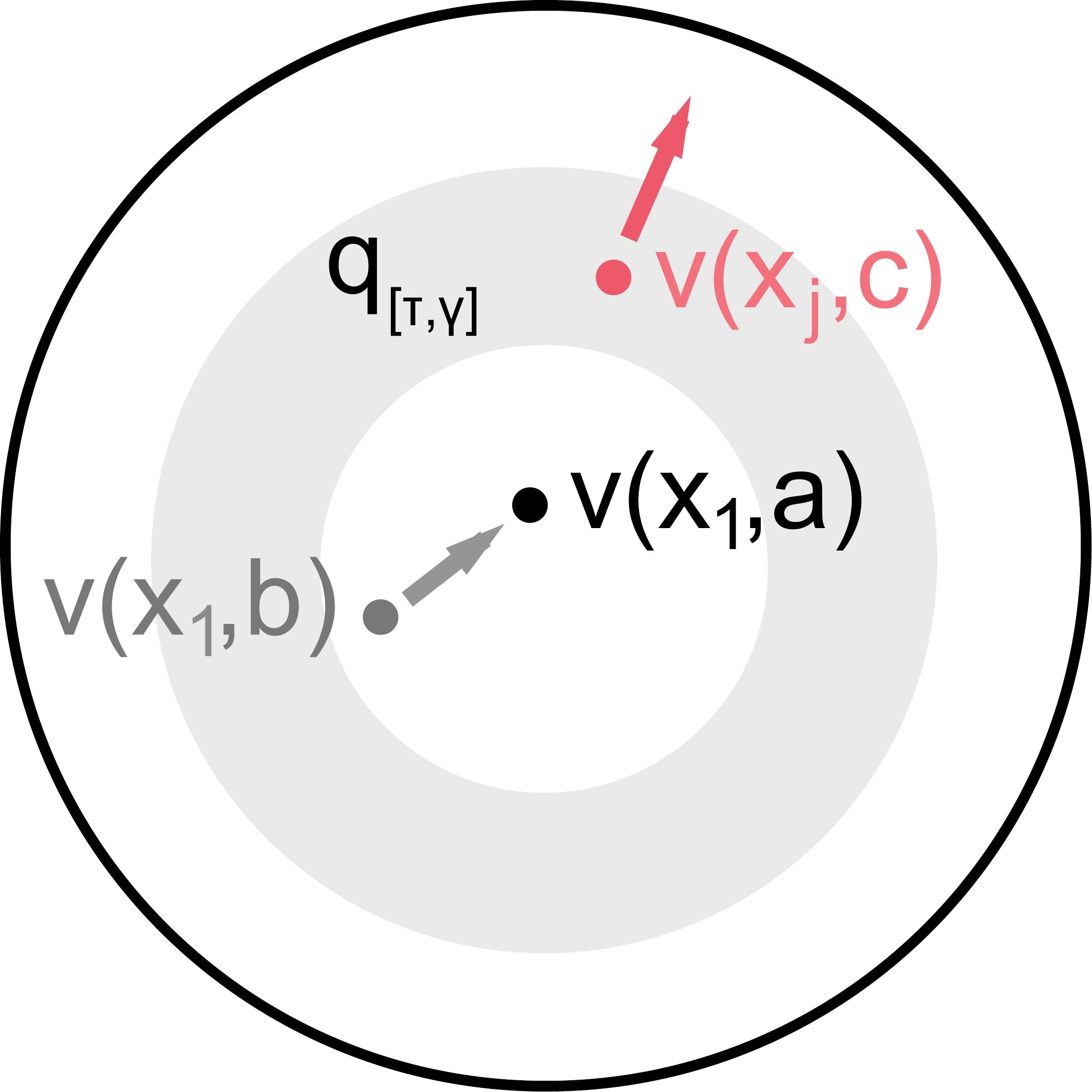}
    \caption{RING}
  \end{subfigure}
  \begin{subfigure}[b]{0.13\textwidth}
    \centering
    \includegraphics[width=0.9\textwidth]{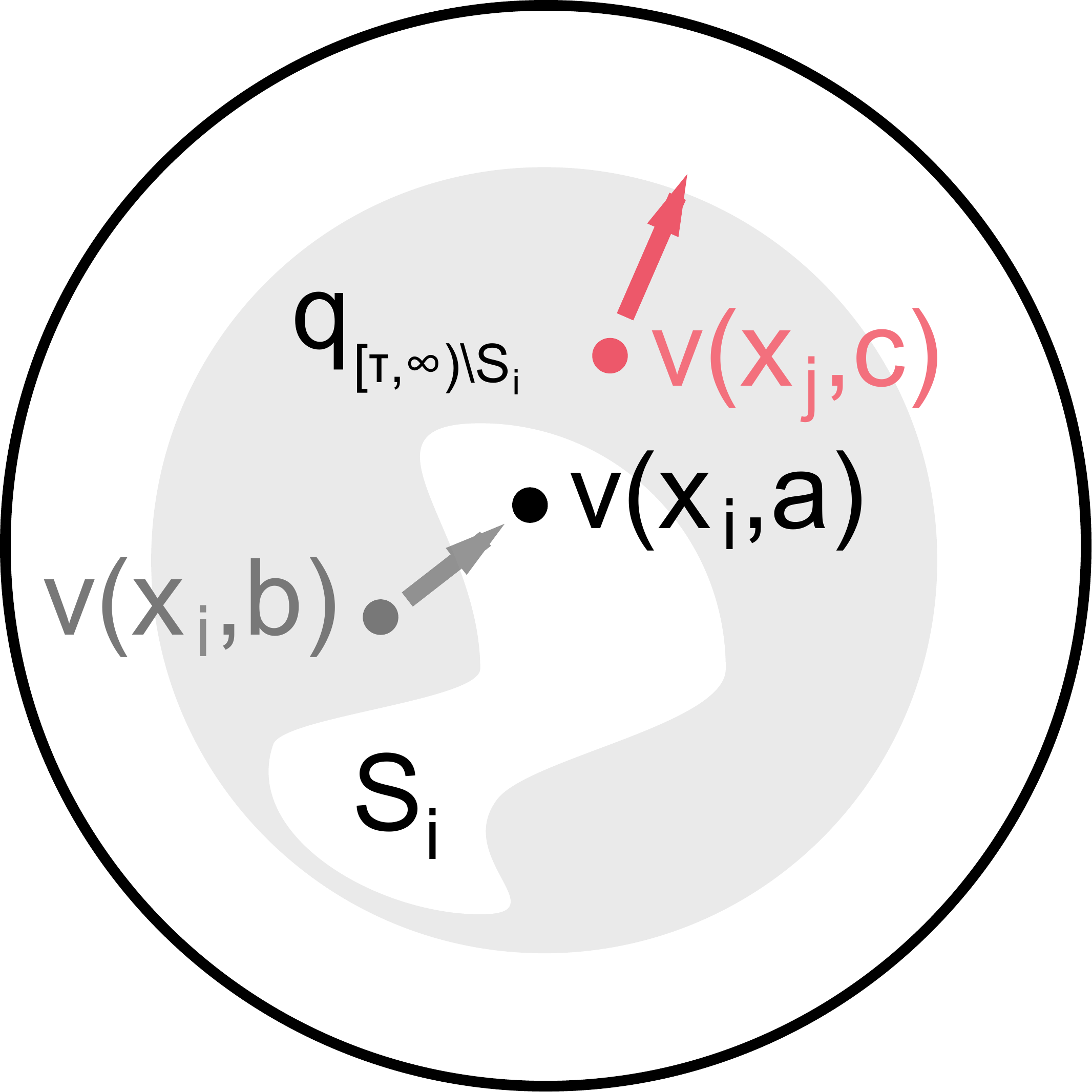}
    \caption{CAVE}
  \end{subfigure}
  \begin{subfigure}[b]{0.13\textwidth}
    \centering
    \includegraphics[width=0.9\textwidth]{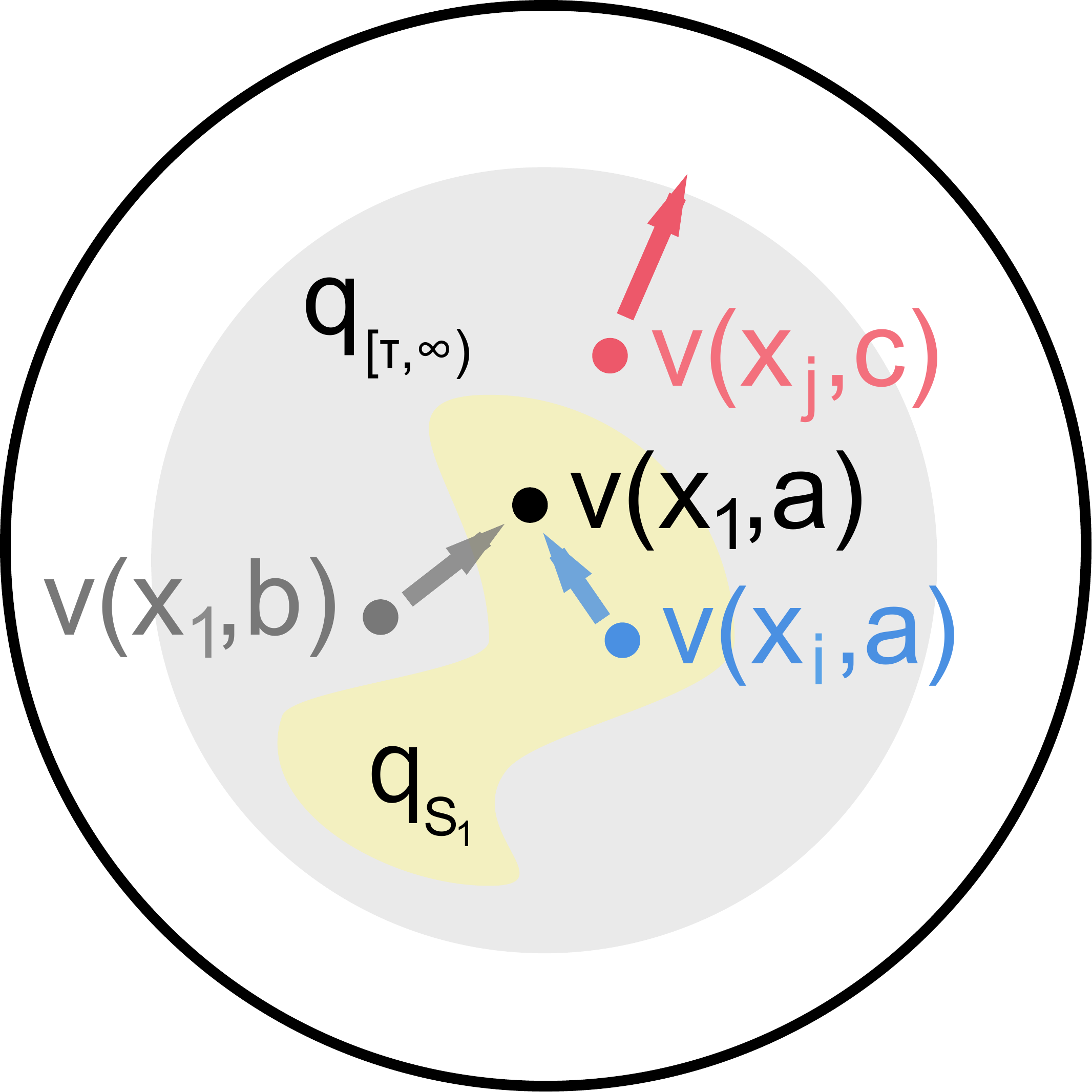}
    \caption{LA}
  \end{subfigure}
  \begin{subfigure}[b]{0.13\textwidth}
    \centering
    \includegraphics[width=0.9\textwidth]{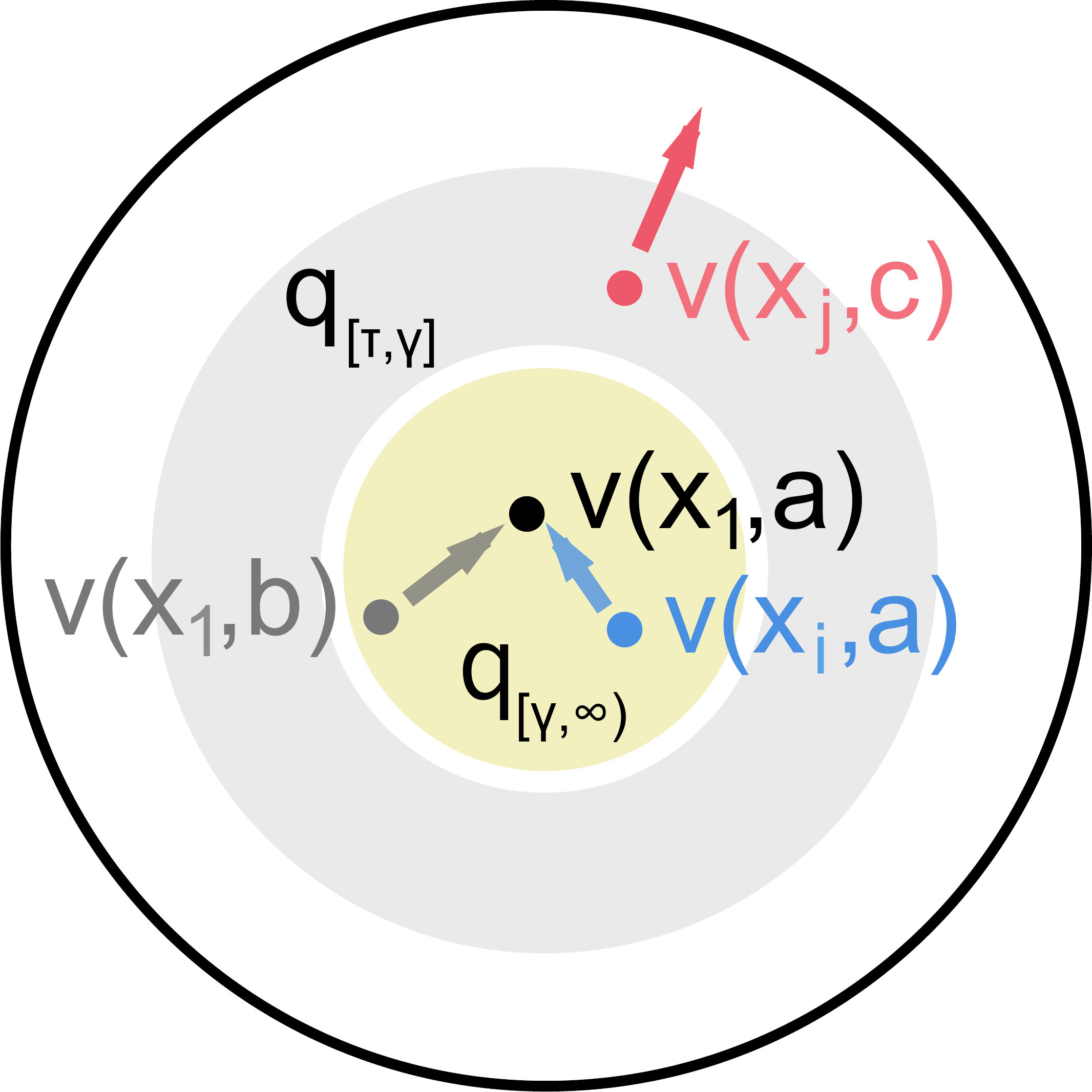}
    \caption{RING$^{\text{s-neigh}}$}
  \end{subfigure}
  \begin{subfigure}[b]{0.13\textwidth}
    \centering
    \includegraphics[width=0.9\textwidth]{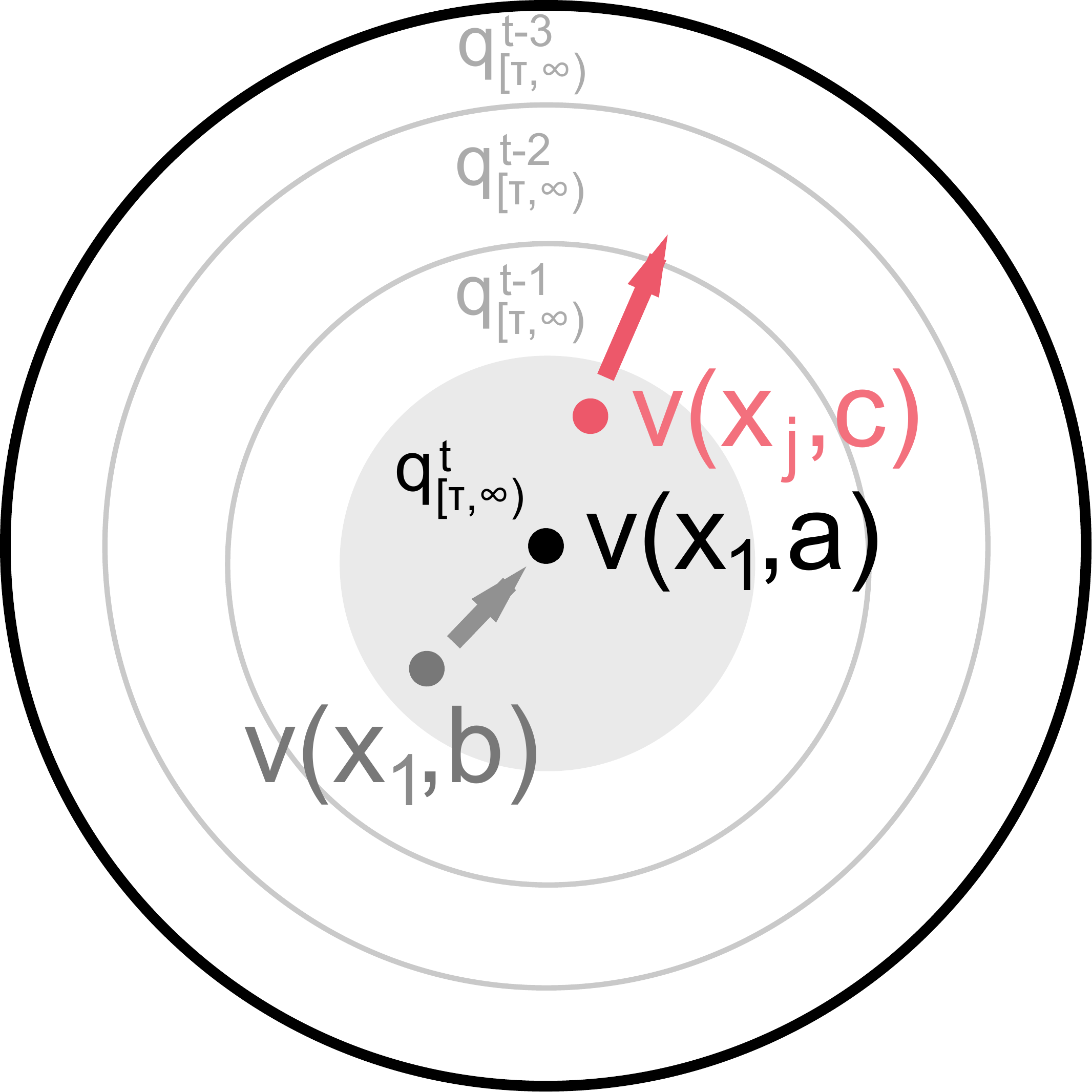}
    \caption{BALL$^{\text{anneal}}$}
  \end{subfigure}
  \caption{Various contrastive algorithms. Black: current view of image $x_1$; gray: other views of $x_1$; red: negative samples; gray: negative distribution; yellow: neighbor distribution.}
  \label{fig:models}
  \vspace{-2em}
\end{figure}

\textbf{K-Neighbor and S-Neighbor views.} We hypothesize that one of the strengths of LA is optimizing representations to put close neighbors near each other. Naturally, we  similarly extend BALL, RING, and CAVE to push close neighbors together. In addition to the VINCE distribution $q_T$ used for negative sampling (e.g. $T = [\tau,\gamma]$ for RING), define a second VINCE distribution $q_C$ with set $C$ lower bounded by $\tau$ to sample ``close neighbors''.
We explore two possible choices for $q_C$. Fix $x_1$ as some example from our dataset. First, by \textit{S-neighborhood}, we set $C = [\gamma, \infty)$, the inner ball defined by RING. Second, by \textit{K-neighborhood}, we set $C = S_1$, the set of elements with the same cluster assignment as $x_1$ given by K-means. In our experiments, we denote each with a superscript $^{\text{s-neigh}}$ or $^{\text{k-neigh}}$ with no superscript meaning no close neighbors. As shown in  Sec.~\ref{sec:lavince}, when using a neighborhood, we sum over close neighbors inside the log, via Jensen's, to stay faithful to the design of LA. As an example, we may write $\scriptstyle \mathcal{L}^{\text{RING}^{\text{s-neigh}}}(x_1, M) = \mathbb{E}_{p(a)}\mathbb{E}_{x_{2:K} \sim q_{[\tau,\gamma]}}\mathbb{E}_{x'_{2:L} \sim q_{[\gamma,\infty)}}\left[\log \frac{\sum_{k \in \{\rho(x'_l)\}_{l=1}^L} e^{g_\theta(\nu(x_1, a))^T M[k]/\omega}}{\sum_{j=1}^K e^{g_\theta(\nu(x_1, a))^T M[\rho(x_j)]/\omega}}\right]$ where $x'_1 = x_1$ and $\rho$ maps $x$ to its index in the dataset.
Note that the negative distribution and the ``neighborhood'' distribution are two distinct objectives although the same sets can be used for both.

\begin{wraptable}{l}{7cm}
\tiny
\centering
\begin{tabular}{l|c|l|c}
\toprule
Model & Top1 & Model & Top1 \\
\midrule
IR \cite{wu2018unsupervised} & 64.3 & CMC \cite{tian2019contrastive} & 72.1 \\
IR$^{\text{nce}}$ & 81.2 & CMC$^{\text{nce}}$ & 85.6 \\
LA \cite{zhuang2019local} & 81.8 & - & -   \\
LA$^{\text{nce}}$ & 82.3  & - & -  \\
BALL & 81.4 & BALL$^{\textup{Lab}}$ & 85.7 \\
BALL$^{\textup{anneal}}$ & 82.1 & BALL$^{\textup{Lab+anneal}}$ & 86.8 \\
BALL$^{\textup{s-neigh}}$ & 75.1 & BALL$^{\textup{Lab+s-neigh}}$ & 63.3 \\
BALL$^{\textup{s-neigh+anneal}}$ & 84.8 & BALL$^{\textup{Lab+s-neigh+anneal}}$ & 87.0 \\
CAVE & 81.4 & CAVE$^{\textup{Lab}}$ & 86.1 \\
CAVE$^{\textup{anneal}}$ & 84.8 & CAVE$^{\textup{Lab+anneal}}$ & \textbf{87.8} \\
RING & 84.7 & RING$^{\textup{Lab}}$ & 87.3 \\
RING$^{\textup{s-neigh}}$ & 76.6 & RING$^{\textup{Lab+s-neigh}}$ & 68.0 \\
RING$^{\textup{anneal}}$ & 85.2 & RING$^{\textup{Lab+anneal}}$ & 87.6 \\
RING$^{\textup{s-neigh+anneal}}$ & \textbf{85.5} & RING$^{\textup{Lab+s-neigh+anneal}}$ & \textbf{87.8} \\
\bottomrule
\end{tabular}
\caption{CIFAR10 Transfer Accuracy}
\label{table:cifar}
\vspace{-2em}
\end{wraptable}

\textbf{Estimating the partition is unnecessary.}
 Both the numerator and denominator of the original formulation of LA (Eq.~\ref{eqn:la}) require estimating the dot product of the current image against all possible negative examples with a finite approximation. Even with a large number of samples, doing so twice can introduce unnecessary variance.
However, LA is equivalent to BALL with K-Neighbor views, which simplifies the computation to the difference of two $\texttt{logsumexp}$ terms (see lemma~\ref{lem:twologumexp}).
We also note the equivalence to MI demands $i \in C^{(i)}$, which is not always true in Eq.~\ref{eqn:la}. In practice, ensuring so improves performance as well.
In Table~\ref{table:cifar}, we refer to this simplified version as LA$^{\text{nce}}$.

\textbf{Annealing the variational distribution.}
Early in learning, representations may be too noisy to properly distinguish similar images when choosing negative examples (see Fig.~\ref{fig:limitation}).
We thus consider \textit{annealing} $\tau$, increasing it on each gradient step, beginning at $-\infty$.
In the results, we use the superscript $^{\text{anneal}}$ when annealing is used. For CAVE models, annealing $\tau$ will shrink the size of the set $S_1$ defined by K-means automatically. For RING, we can choose to further anneal $\gamma$ from $\infty$ to $\tau$. That is, while we shrink the outer ball, we also  enlarge the inner ball.

\begin{table}[t!]
\tiny
\centering
\begin{subtable}[h]{0.49\textwidth}
\centering
\begin{tabular}{l|c|c|l|c|c}
\toprule
Model & Top1 & Top5 & Model & Top1 & Top5 \\
\midrule
IR & 43.2 & 67.0 & CMC & 48.2 & 72.4 \\
BALL & 45.3 & 68.6 & BALL$^{\text{Lab}}$ & 48.9 & 73.1 \\
CAVE & 46.6 & 70.4  & CAVE$^{\text{Lab}}$ & 49.2 & 73.3 \\
LA & 48.0 & 72.4 & -- & -- & -- \\
BALL$_{\text{anneal}}$ & 47.3 & 71.1 & BALL$^{\text{Lab+anneal}}$ & 49.7 & 73.6 \\
CAVE$_{\text{anneal}}$ & \textbf{48.4} & \textbf{72.5} & CAVE$^{\text{Lab+anneal}}$ & \textbf{50.5} & \textbf{74.0} \\
\bottomrule
\end{tabular}
\caption{ImageNet: Classification Accuracy}
\end{subtable}
\begin{subtable}[h]{0.49\textwidth}
    \centering
    \begin{tabular}{l|c c c | c c c }
    \toprule
    Model & AP$^{\text{bb}}$ & AP$_{50}^{\text{bb}}$ & AP$_{75}^{\text{bb}}$ & AP$^{\text{mk}}$ & AP$_{50}^{\text{mk}}$ & AP$_{75}^{\text{mk}}$ \\
    \midrule
    IR & 8.6 & 19.0 & 6.6 & 8.5 & 17.4 & 7.4 \\
    BALL & 9.4 & 20.3 & 7.3 & 9.3 & 18.8 & 8.4 \\
    CAVE & 9.9 & 21.2 & 7.8 & 9.7 & 19.7 & 8.6 \\
    LA & 10.2 & 22.0 & 8.1 & 10.0 & 20.3 & 9.0 \\
    BALL$^{\text{anneal}}$ & 9.8 & 21.0 & 7.7 & 9.8 & 19.8 & 8.8 \\
    CAVE$^{\text{anneal}}$ & \textbf{10.7} & \textbf{22.8} & \textbf{8.5} & \textbf{10.7} & \textbf{20.9} & \textbf{9.4} \\
    \bottomrule
    \end{tabular}
    \caption{COCO: Object Detection and Inst. Segmentation}
\end{subtable}
\begin{subtable}[h]{0.33\textwidth}
    \centering
    \begin{tabular}{l|c c c }
    \toprule
    Model & AP$^{\text{kp}}$ & AP$_{50}^{\text{kp}}$ & AP$_{75}^{\text{kp}}$ \\
    \midrule
    IR & $34.6$ & $63.0$ & $32.9$ \\
    BALL & $34.9$ & $63.6$ & $33.8$ \\
    CAVE & $35.9$ & $64.9$ & $34.2$  \\
    LA & $36.3$ & $65.3$ & $35.1$ \\
    BALL$^{\text{anneal}}$ & $36.3$ & $65.2$ & $35.0$ \\
    CAVE$^{\text{anneal}}$ & $\mathbf{37.0}$ & $\mathbf{66.1}$ & $\mathbf{35.6}$\\
    \bottomrule
    \end{tabular}
    \caption{COCO: Keypoint Detection}
\end{subtable}
\begin{subtable}[h]{0.32\textwidth}
    \centering
    \begin{tabular}{l|c c c}
    \toprule
    Model & AP$^{\text{bb}}$ & AP$_{50}^{\text{bb}}$ & AP$_{75}^{\text{bb}}$  \\
    \midrule
    IR & $5.5$ & $14.5$ & $3.3$ \\
    BALL & $6.0$ & $15.7$ & $3.5$ \\
    CAVE & $6.6$ & $17.6$ & $4.1$ \\
    LA & $\mathbf{7.6}$ & $20.0$ & $\mathbf{4.3}$ \\
    BALL$^{\text{anneal}}$ & $7.1$ & $18.3$  & $4.1$ \\
    CAVE$^{\text{anneal}}$ & $7.5$ & $\mathbf{20.1}$ & $4.3$ \\
    \bottomrule
    \end{tabular}
    \caption{VOC: Object Detection}
\end{subtable}
\begin{subtable}[h]{0.32\textwidth}
    \centering
    \begin{tabular}{l|c c c}
    \toprule
    Model & AP$^{\text{mk}}$ & AP$_{50}^{\text{mk}}$ & AP$_{75}^{\text{mk}}$  \\
    \midrule
    IR & $3.0$ & $5.9$ & $2.8$ \\
    BALL & $3.4$ & $6.5$ & $3.3$ \\
    CAVE & $3.8$ & $7.3$ & $3.5$ \\
    LA & $3.8$ & $7.2$ & $3.5$ \\
    BALL$^{\text{anneal}}$ & $4.0$ & $7.7$ & $3.8$ \\
    CAVE$^{\text{anneal}}$ & $\mathbf{4.3}$ & $\mathbf{8.2}$ & $\mathbf{4.0}$ \\
    \bottomrule
    \end{tabular}
    \caption{LVIS: Instance Segmentation}
\end{subtable}
\caption{Evaluation of the representations using six visual transfer tasks}
\label{tab:othertransfertask}
\vspace{-3em}
\end{table}

\section{Experiments}
\label{sec:mainresults}
\vspace{-1em}
We conduct two sets of experiments: one on CIFAR10 to compare all the objective variants, and one on ImageNet to compare select algorithms on several transfer tasks. In aggregate, we give the reader a thorough description of the algorithms uncovered by the theory and their relative performance.

\textbf{Comparing negative samples and view functions.}
Table~\ref{table:cifar} compares 26 algorithms, of which 23 are introduced in this paper. For each algorithm, we measure its performance on classification by training a logistic regression model on the frozen representations only (see Sec.~\ref{sec:mainappendix_sec5}).

We make a few observations: First, our simplifications to IR, CMC, and LA outperform their original implementations by up to 15\%. We note that the difference between LA$^{\textup{nce}}$ and LA is much less than the others since $\kappa$ cancels from the top and bottom of the objective. Still, we see improvements due to less randomness in LA$^{\textup{nce}}$. Second, harder negative samples build stronger representations as BALL, RING, CAVE outperform IR$^{\textup{nce}}$ by up to 3\%. In particular, RING outperforms CAVE suggesting it as a simpler but effective alternative to K-means. Further, if we anneal in the negative distributions, we find even better performance, with gains of 3-4\% over IR$^{\textup{nce}}$. Third, adding close neighbor views is not always beneficial: without annealing, we observe convergence to local minima as the representations cling to spurious dependencies in a noisy close neighbor set. This explains why RING even outperforms LA, which heavily relies on a well-defined neighborhood. However, with annealing, we can wait until the representation space is reasonable before adding neighbors. Together, this forms our best model, RING$^{\text{(Lab)+s-neigh+anneal}}$, outperforming IR$^{\text{nce}}$ by 4\%, LA$^{\text{nce}}$ by 3\%, and CMC$^{\text{nce}}$ by 2\% on CIFAR10 with gains of up to 20\% without the changes to improve stability.

\textbf{Comparing Transfer Tasks.} Picking six algorithms, we compare their representations on five transfer tasks for a more general measure of usefulness: classification on ImageNet, object detection, instance and keypoint segmentation on COCO \cite{lin2014microsoft},  object detection on Pascal VOC '07 \cite{everingham2010pascal}, and instance segmentation on LVIS \cite{gupta2019lvis}. Like prior work \cite{he2019momentum}, we use Detectron2 \cite{wu2019detectron2} but with a frozen ResNet18 backbone to focus on representation quality (see Sec.~\ref{sec:detectron}).
From Table~\ref{tab:othertransfertask}, we find more evidence of the importance of negative distributions as BALL and CAVE consistently outperform IR. Also, our best models $\text{CAVE}^{\text{(Lab)+anneal}}$ outperform IR, LA, and CMC across tasks.




\section{Conclusion}
\vspace{-1em}
We presented an interpretation of representation learning based on mutual information between image views. This formulation led to more systematic understanding of a family of existing approaches.
In particular, we uncovered that the choices of views and negative sample distribution strongly influence the performance of contrastive learning.
By choosing more difficult negative samples, we surpassed high-performing algorithms like LA and CMC across several visual tasks.

Several theoretical and practical questions remain.
We identified the importance of views
but left open questions about what makes for a ``good'' view.
Our choice of negative examples was justified by the VINCE bound on MI. Yet VINCE is  looser than InfoNCE. Future work can elucidate the relationship between learning a good representation and tightening bounds.

\section*{Broader Impact}

The recent successes of machine learning may have profound societal consequences. Yet the extremely data- and compute-intensive nature of recent methods limits their availability to a small segment of society. It is urgent that we find ways to make the use of machine learning possible for a wide variety of people with a wide variety of tasks.
One promising route is to solve particular tasks, with small data and little compute, by building from widely applicable and available \textit{pre-trained representations}.
Contrastive algorithms are becoming a popular method for  unsupervised learning of representation useful for transfer tasks.
This makes it important to provide a good theoretical understanding of these methods, and to make the implementations and trained representations publicly available.
Our work is a step toward making these techniques clearer and better performing.
We are committed to open sourcing our implementations and providing trained representations.
We believe this has the potential to allow many people to solve their own tasks by building on our representations.
Yet the representation learning algorithms themselves remain both data- and compute-intensive. This has the danger to provide an additional barrier to access.
Future research should ameliorate this risk by committing to open access to code and trained representations, and by research ways to reduce the data or compute requirements of these methods.


\bibliographystyle{plain}
{\small
\linespread{1}
\bibliography{report}
}
\newpage
\appendix

\section{Proofs}

\subsection{Proof of Theorem~\ref{thm:marketing}}
\begin{proof}
As sums are invariant to permutation, the result is immediate.
\end{proof}

\subsection{Proof of Lemma~\ref{lem:ir_mi}}
\label{sec:prooflemir_mi}
\begin{proof}
We focus on showing an equivalence between InfoNCE and IR, upon which its relation to CMC becomes clear. We rename $X$ to $X_1$ and make $K-1$ i.i.d. copies of $X_1$, named $X_{2:K}$.

First, let the view function $\nu_{\textup{IR}}^*$ of index sets be defined with the standard image augmentations (cropping, flipping, color jitter, and grayscale) as used in training IR. Define an empirical estimate of $\mathcal{I}(X_1; \nu_{\textup{IR}}^*)$ using the dataset $\mathcal{D}$. That is,
\begin{equation*}
    {\scriptstyle
    \hat{\mathcal{I}}^{\text{NCE}}(X_1; \nu^*) = \mathbb{E}_{p_{\mathcal{D}}(x_{1:K})} \mathbb{E}_{A_{1:K},B,B' \sim p(A)}\left[\log \frac{e^{f^*_{\theta}(\nu^*(x_1,B), \nu^*(x_1, B'); w_B, w_{B'})}}{\frac{1}{K}\sum_{j=1}^K e^{f^*_{\theta}(\nu^*(x_1,B\}), \nu^*(x_j, A_j); w_B, w_{A_j}) }} \right]}
\end{equation*}
where $p_{\mathcal{D}}(x_{1:K}) = \prod_{j=1}^K p_{\mathcal{D}}(x_j)$.
In practice, we do not know $p_{\mathcal{D}}$, but since for all $x^{(i)}$ in $\mathcal{D}$, $x^{(i)} \sim p_{\mathcal{D}}$, we can sample $p_{\mathcal{D}}$ by uniformly picking from $\mathcal{D}$. So, we assume in our computation that for any $x \sim p_{\mathcal{D}}(x)$, $x$ is in $\mathcal{D}$.
Now, define an indexing function $\rho$ that maps a sample $x$ from $p_{\mathcal{D}}$ to an index in the dataset $\mathcal{D}$ such that $x \sim p_{\mathcal{D}} = x^{(\rho(x))} \in \mathcal{D}$. This is well-formed by our assumption.

Let $x_1$ represent the current image, a realization of $X_1$ sampled from $p_{\mathcal{D}}(x_1)$. As described in the main text, we consider a special case of $\hat{\mathcal{I}}^{\text{NCE}}(X; \nu^*)$ where the two index sets for the current image $x_1$ are $\{a\}$, with  $a$ being the sampled view index in $\mathcal{A}$ for the current epoch of training, and $A^{(\rho(x_1))}$, the set of sampled view indices from the last $n_\alpha$ epochs of training. Further, define weights $w_1 = \{1\}$ and $w_2 = \{1, \alpha, \ldots, \alpha^{n_\alpha - 1}\}$. Then, we can rewrite the above equation as:
\begin{align*}
    \scriptstyle
    \hat{\mathcal{I}}^{\text{NCE}}(X_1; \nu^*) &= \scriptstyle \mathbb{E}_{p_{\mathcal{D}}(x_{1:K})} \mathbb{E}_{a \sim p(a)}\left[\log \frac{e^{f^*_{\theta}(\nu^*(x_1,\{a\}), \nu^*(x_1, A^{(\rho(x_1))}); w_1, w_2)}}{\frac{1}{K}\sum_{j=1}^K e^{f^*_{\theta}(\nu^*(x_1,\{a\}\}), \nu^*(x_j, A^{(\rho(x_j))}); w_1, w_2)) }} \right] \\
    &= \scriptstyle \mathbb{E}_{p_{\mathcal{D}}(x_{1:K})} \mathbb{E}_{a \sim p(a)}\left[\log \frac{e^{g_{\theta}(\nu(x_1,a))^T (\sum_{k=1}^{|A^{\rho(x_1)}|} w_2[k] g_\theta(\nu(x_1, A^{(\rho(x_1))}[k])))}}{\frac{1}{K}\sum_{j=1}^K e^{g_{\theta}(\nu(x_1,a))^T (\sum_{k=1}^{|A^{\rho(x_1)}|} w_2[k] g_\theta(\nu(x_1, A^{(\rho(x_j))}[k])))}} \right] \\
    &= \scriptstyle \mathbb{E}_{p_{\mathcal{D}}(x_{1:K})} \mathbb{E}_{a \sim p(a)}\left[\log \frac{e^{g_{\theta}(\nu(x_1,a))^T M[\rho(x_1)]}}{\frac{1}{K}\sum_{j=1}^K e^{g_{\theta}(\nu(x_1,a))^T M[\rho(x_j)]}} \right]  \\
    &= \scriptstyle \mathbb{E}_{p_{\mathcal{D}}(x_{1})} \left[ \mathcal{L}^{\textup{IR}}(x_1, M) \right] + \log K - \log \kappa
\end{align*}
Note that $w_1, w_2$ are shared across the $K$ witness evaluations.
We apply the definition of a memory bank entry in the third inequality. To see why the last equality holds, consider:
\begin{align*}
    \scriptstyle \mathcal{L}^{\textup{IR}}(x_1, M) &\scriptstyle= \mathbb{E}_{p(a)}\left[\log \frac{e^{g_\theta(\nu(x_1, a)^T M[\rho(x_1)]}}{\sum_{j=1}^{N} e^{g_\theta(\nu(x_1, a)^T M[\rho(x_j)]}}\right] \\
    &\scriptstyle \approx \scriptstyle\mathbb{E}_{p(a)}\left[\log \frac{e^{g_\theta(\nu(x_1, a)^T M[\rho(x_1)]}}{\kappa \sum_{j=1}^{K} e^{g_\theta(\nu(x_1, a)^T M[\rho(x_j)]}}\right]\\
    & \scriptstyle= \mathbb{E}_{p(a)}\mathbb{E}_{p_{\mathcal{D}}(x_{2:K})}\left[\log \frac{e^{g_\theta(\nu(x_1, a)^T M[\rho(x_1)]}}{\kappa \sum_{j=1}^{K} e^{g_\theta(\nu(x_1, a)^T M[\rho(x_j)]}}\right] \\
   &\scriptstyle= \mathbb{E}_{p(a)}\mathbb{E}_{p_{\mathcal{D}}(x_{2:K})}\left[\log \frac{e^{g_\theta(\nu(x_1, a)^T M[\rho(x_1)]}}{\frac{1}{K}\sum_{j=1}^{K} e^{g_\theta(\nu(x_1, a)^T M[\rho(x_j)]}}\right] - \log K + \log \kappa
\end{align*}
where $N$ is the size of the dataset $\mathcal{D}$; $K <N$ is the number of samples and $\kappa$ is a constant. We note that second line is an approximation already made in the original implementation of IR. Further, the third line holds since the $K-1$ negative examples $x_{2:K}$ are sampled uniformly from the dataset $\mathcal{D}$, in other words, i.i.d. from the empirical distribution, $p_{\mathcal{D}}$.

Having shown the equivalence between InfoNCE and IR, we move to CMC.
The view set for CMC can be bisected into two: one of luminescence filters $\mathcal{A}_{\textup{CMC-L}}$, and one of AB-color filters $\mathcal{A}_{\textup{CMC-ab}}$.
Together, $\mathcal{A}_{\textup{CMC}} = \mathcal{A}_{\textup{CMC-L}} \cup \mathcal{A}_{\textup{CMC-ab}}$.
Recall that the CMC objective is a sum of two IR objectives:
\[\mathcal{L}^{\textup{CMC}}(x, M) = \mathcal{L}^{\textup{IR}}(x_{\textup{L}}, M_{\textup{ab}}) + \mathcal{L}^{\textup{IR}}(x_{\textup{ab}}, M_{\textup{L}}) \]
So if we define two extended view functions, $\nu^*_{\text{CMC-L}}: X \times 2^{\mathcal{A}_{\textup{CMC-L}}} \rightarrow 2^X$ and $\nu^*_{\text{CMC-ab}}: X \times 2^{\mathcal{A}_{\textup{CMC-ab}}} \rightarrow 2^X$. Then by the above, we have $\hat{\mathcal{I}}^{\text{NCE}}(X_1; \nu_{\text{CMC-L}}^*) \equiv \mathbb{E}_{p_{\mathcal{D}}(x_{1})} \left[ \mathcal{L}^{\textup{IR}}(x_{1, \textup{L}}, M_{\textup{ab}}) \right] + \log K - \log \kappa$ and $\hat{\mathcal{I}}^{\text{NCE}}(X_1; \nu_{\text{CMC-ab}}^*) \equiv \mathbb{E}_{p_{\mathcal{D}}(x_{1})} \left[ \mathcal{L}^{\textup{IR}}(x_{1, \textup{ab}}, M_{\textup{L}}) \right] + \log K - \log \kappa$.
So $\hat{\mathcal{I}}^{\text{NCE}}(X_1; \nu_{\text{CMC}}^*) \equiv \mathbb{E}_{p_{\mathcal{D}}(x_{1})} \left[ \mathcal{L}^{\textup{CMC}}(x_{1}, M) \right] / 2 + \log K - \log \kappa$ where $M = (M_{\textup{L}}, M_{\textup{ab}})$.
\end{proof}

\subsection{Proof of Theorem~\ref{thm:vince}}
\label{sec:proof:thm:vince}
\begin{proof}
It suffices to show the inequalities $c < e^\tau \leq \mathbb{E}_{\mathbb{Q}_T}[g(x)]$. The first holds since $ \log c < \tau$ by assumption.
The second holds since $\mathbf{1}_{S_T}e^{\tau} \leq\mathbf{1}_{S_T} g(x)$, and taking the expectation $\mathbb{E}_{\mathbb{Q}_T}$ of both sides of this inequality gives the desired result upon observing that $\mathbb{E}_{\mathbb{Q}_T}[\mathbf{1}_{S_T}e^{\tau}] = \mathbb{Q}_T(S_T ) e^\tau= e^\tau$ and that $\mathbb{E}_{\mathbb{Q}_T}[\mathbf{1}_{S_T} g(x)] = \mathbb{E}_{\mathbb{Q}_T}[g(x)]$.
\end{proof}

\subsection{Proof of Corollary~\ref{coro:vince}}
\begin{proof}
    Note that each $e^{f(x,y_j)}$ for $j$ in $[K]$ satisfies the conditions on $g$ in Thm~\ref{thm:vince}. So for any $j$, since $\tau > \log  \mathbb{E}_{p(y)}[e^{f(y_1,y)}]$, by Thm~\ref{thm:vince}, we have that $\mathbb{E}_{p(y_j)}\left[e^{f(x,y_j)}\right] \leq \mathbb{E}_{q(y_j)}\left[e^{f(x,y_j)}\right]$. Then by linearity of expectation, $\mathbb{E}_{p(x,y_1)}\mathbb{E}_{p(y_{2:K})}\left[\frac{1}{K}\sum_{j=1}^K e^{f(x,y_j)}\right] \leq \mathbb{E}_{p(x,y_1)}\mathbb{E}_{q(y_{2:K})}\left[\frac{1}{K}\sum_{j=1}^K e^{f(x,y_j)}\right]$.
    Finally, let $h: \mathbb{R} \rightarrow \mathbb{R}$ be any monotonic increasing function.
    Then, $\mathbb{E}_{p(x,y_1)}\mathbb{E}_{p(y_{2:K})}\left[h(\frac{1}{K} \sum_{j=1}^K e^{f(x,y_j)}) \right] \leq  \mathbb{E}_{p(x,y_1)}\mathbb{E}_{q_T(y_{2:K})}\left[h(\frac{1}{K} \sum_{j=1}^K e^{f(x,y_j)} )\right]$.
    But if  $h = \log$, then $\mathcal{I}^{\text{VINCE}}(X;Y_1) \leq \mathcal{I}^{\text{NCE}}(X;Y_1)$.
  As $-\tau \rightarrow \infty = \gamma$, the bound is tight.
\end{proof}

\subsection{Proof of Lemma~\ref{lem:coro_ballann}}
\label{sec:proof:lemmaball}
\begin{proof}
    We prove the more general statement for any set $T$ lower bounded by $\tau$, then focus on BALL as a special case.
    Recall that the  objective we are interested in is
    \begin{equation*}
        \scriptstyle
        \mathcal{L}^T(x_1, M) = \mathbb{E}_{p(a)}\mathbb{E}_{q_T(x_{2:K}|\nu(x_1, a))}\left[\log\frac{e^{g_\theta(\nu(x_1, a))^T M[\rho(x_1)]/\omega}}{\sum_{j=1}^K  e^{g_\theta(\nu(x_1, a))^T M[\rho(x_j)]/\omega}}\right]
    \end{equation*}
    where $x_1$ is a realization of random variable $X_1$, sampled an empirical distribution $p_{\mathcal{D}}(x_1)$, and index $a \sim p(a)$.
    We call this class of objectives $T$-Discrimination.
    By construction, $T$-Discrimination uses a negative sampling distribution $q_T$ that lies in the VINCE family of distributions.
    We use the notation $q_{\mathcal{D},T}$ to refer to the empirical distribution of $q_{T}$, much like $p_{\mathcal{D}}$ and $p$. In other words, samples from $q_{\mathcal{D},T}$ will be elements of $\mathcal{D}$. Although, because $q_{\mathcal{D},T}$ has restricted supported, not all members of $\mathcal{D}$, which are sampled i.i.d. from $p_{\mathcal{D}}$,  will be valid samples from $q_{\mathcal{D},T}$.
    By assumption, $\tau > \log \mathbb{E}_{x \sim p(x)}\left[ e^{f_\theta(\nu(x_1,a), x)} \right]$, so if we can show that $T$-Discrimination is equivalent to VINCE, we will get that it lower bounds mutual information between weighted view sets by Thm.~\ref{thm:vince}.

    To show the equivalence between $T$-Discrimination and VINCE, we need only  expand the memory bank as a weighted sum of views. Notice that  $g_\theta(\nu(x_1, a)) = g^*_\theta(\nu^*(x_1, \{a\}), w_1)$ and $M[\rho(x_1)] = g^*_\theta(\nu^*(x_1, A^{(\rho(x_1))}), w_2)$ where $a$ is view index at the current epoch and $A^{(\rho(x_1))}$ is a set of view indices from the last $n_\alpha$ epochs as presented in Sec.~\ref{sec:exist:algo}. Doing so immediately proves the statement:
    \begin{align*}
        \scriptstyle
        \mathbb{E}_{p_\mathcal{D}(x_1)}\left[\mathcal{L}^T(x_1, M)\right] + \log K &=\scriptstyle \mathbb{E}_{p_\mathcal{D}(x_1)}\mathbb{E}_{p(a)}\mathbb{E}_{q_{\mathcal{D}, T}(x_{2:K}|\nu(x_1, a))}\left[\log\frac{e^{g_\theta(\nu(x_1, a))^T M[\rho(x_1)]/\omega}}{\sum_{j=1}^K  e^{g_\theta(\nu(x_1, a))^T M[\rho(x_j)]/\omega}}\right] + \log K \\
        &= \scriptstyle \mathbb{E}_{p_\mathcal{D}(x_1)}\mathbb{E}_{p(a)}\mathbb{E}_{q_{\mathcal{D}, T}(x_{2:K}|\nu(x_1, a))}\left[\log\frac{e^{g^*_\theta(\nu^*(x_1, \{a\}), w_1)^T g^*_\theta(\nu^*(x_1, A^{(\rho(x_1))}), w_2)/\omega}}{\frac{1}{K}\sum_{j=1}^K  e^{g^*_\theta(\nu^*(x_1, \{a\}), w_1)^T g^*_\theta(\nu^*(x_1, A^{(\rho(x_j))}), w_2)/\omega}}\right]  \\
        &= \scriptstyle \mathbb{E}_{p_\mathcal{D}(x_1)}\mathbb{E}_{p(a)}\mathbb{E}_{q_{\mathcal{D}, T}(x_{2:K}|\nu(x_1, a))}\left[\log\frac{e^{f^*_\theta(\nu^*(x_1, \{a\}), \nu^*(x_1, A^{(\rho(x_1))}); w_1, w_2)}}{\frac{1}{K}\sum_{j=1}^K  e^{f^*_\theta(\nu^*(x_1, \{a\}), \nu^*(x_1, A^{(\rho(x_j))}); w_1, w_2)}}\right] \\
        &= \scriptstyle \hat{\mathcal{I}}^{\text{VINCE}}(X_1; \nu^*)
    \end{align*}
    where $\hat{\mathcal{I}}^{\text{VINCE}}$ is VINCE with respect to empirical distributions $p_{\mathcal{D}}$, $q_{\mathcal{D},T}$.  In particular, since the equivalence between VINCE and $T$-Discrimination holds for all $T$ lower bounded by $\tau$, it holds for $T = [\tau, \infty)$, which we call BALL Discrimination.
\end{proof}

\subsection{Proof of Lemma~\ref{lem:la_ml}}
\begin{proof}
  First, we note that $\mathcal{L}^{\text{LA$_0$}}(x^{(i)}, M) =  \mathbb{E}_{p(a)} \left[\log \frac{p(i|x^{(i)},M)}{p(B^{(i)}|x^{(i)},M)}\right]$.
  Say the background neighbor set $B^{(i)}$ has $K$ elements. By construction, we can view each element in $B^{(i)}$ as an i.i.d. sample from some VINCE distribution, $q_{[\tau,\infty)}(x|\nu(x^{(i)}, a))$ for some $a \in \mathcal{A}$. This is true because we define $B^{(i)}$ programmatically by sampling uniformly from the $b$ closest examples (in the dataset) to the current view of the image, $\nu(x^{(i)}, a)$, where distance is measured in representation space. But because the elements from $\mathcal{D}$ are sampled i.i.d. from the empirical distribution $p_\mathcal{D}$, sampling uniformly from the closest $b$ examples is the same as sampling i.i.d. from the true marginal distribution $p$ but with support restricted to elements at least as close to $\nu(x^{(i)}, a)$ as the furthest element in $B^{(i)}$. In other words, sampling from a VINCE distribution. In short, the construction of the background neighborhood is precisely equivalent to sampling from $q_{[\tau,\infty)}$ $K$ times i.i.d.

  Note $\log \frac{p(i|x^{(i)},M)}{p(B^{(i)}|x^{(i)},M)} = \log \left( \frac{\frac{e^{g_\theta(\nu(x^{(i)}), a)^T M[i]/\omega}}{\sum_j e^{g_\theta(\nu(x^{(i)}), a)^T M[j]/\omega}}}{\frac{\sum_{k=1}^K e^{g_\theta(\nu(x^{(i)}), a)^T M[k]/\omega}}{\sum_j e^{g_\theta(\nu(x^{(i)}), a)^T M[j]/\omega}}} \right) = \log \left( \frac{e^{g_\theta(\nu(x^{(i)}), a)^T M[i]/\omega}}{\sum_{j=1}^K e^{g_\theta(\nu(x^{(i)}), a)^T M[j]/\omega}} \right)$ .
  So $\mathcal{L}^{\text{LA$_0$}}(x_1, M) = \mathbb{E}_{p(a)}\mathbb{E}_{q_{[\tau,\infty)}(x_{2:K}|\nu(x_1,a))} \left[\log \left( \frac{e^{g_\theta(\nu(x_1), a)^T M[\rho(x_i)]/\omega}}{\sum_{j=1}^K e^{g_\theta(\nu(x_1), a)^T M[\rho(x_j)]/\omega}} \right)\right] = \mathcal{L}^{\text{BALL}}(x_1, M)$ where $\rho$ is a function realizations of $X$ to an index in $[1, |\mathcal{D}|]$ as defined in Sec.~\ref{sec:prooflemir_mi}.
\end{proof}

\subsection{Proof of Lemma~\ref{lem:la_la0}}
\begin{proof}
Take the $i$-th from the  dataset $\mathcal{D}$, $x^{(i)}$. Then  let $C^{(i)}$ be the close neighbor set for $x^{(i)}$ and $B^{(i)}$ the background neighbor set for $x^{(i)}$ as defined in LA. Let $N$ be the number of elements in $C^{(i)} \cap B^{(i)} = C^{(i)}$.
Now, we can show the following
\begin{align*}
    \scriptstyle
    \mathcal{L}^{\text{LA}}(x^{(i)},M) - \log N &= \scriptstyle \log \left(\frac{p(C^{(i)} \cap B^{(i)} | x^{(i)}, M)}{p(B^{(i)} | x^{(i)}, M)}\right) - \log N \\
    &\scriptstyle = \log \left(\frac{\mathbb{E}_{p(a)}\left[\frac{1}{N}\sum_{k \in C^{(i)} \cap B^{(i)}} e^{(g_\theta(\nu(x^{(i)},a))^T M[k] / \omega)} / Z\right]}{p(B^{(i)} | x^{(i)}, M)}\right) \\
    & \scriptstyle \leq \log \left(\frac{\mathbb{E}_{p(a)}\left[ e^{(g_\theta(\nu(x^{(i)},a))^T M[i]/\omega)} / Z \right]}{p(B^{(i)} | x^{(i)}, M)}\right)\\
    &\scriptstyle =  \log \left(\frac{p(i | x_i, M)}{p(B_i | x_i, M)}\right) = \mathcal{L}^{\text{LA}_0}(x^{(i)},M)
\end{align*}
where we use $Z$ to abbreviate the denominator of $p(i|x^{(i)},M)$. The inequality in the third line holds since adding noise to a vector decreases its correlation with any other vector. In other words, of all entries in the memory bank, $M[i]$ is at least as close to $g_\theta(\nu(x^{(i)}, a))$ as any other entry $M[j]$ for any $a \in \mathcal{A}$. Hence, $g_\theta(\nu(x^{(i)}, a))^T M[i] \geq g_\theta(\nu(x^{(i)}, a))^T M[j]$ for any $j \neq i$.
\end{proof}

\subsection{Proof of InfoNCE}
\label{sec:appendix:infonce}

Lemmas~\ref{lemma:ba},~\ref{lem:nwj} and Thm.~\ref{thm:infonce} are largely taken from \cite{poole2019variational}.
We will start with a simpler estimator of mutual information, known as the Barber-Agakov estimator, and work our way to InfoNCE.

\begin{lem}
  Define the (normalized) Barber-Agakov bound as $\mathcal{I}^{\textup{BA}}(X; Y) = \mathbb{E}_{p(x,y)}\left[\log q(x|y) - \log p(x) \right]$ where $q(x|y)$ is a variational distribution approximating $p(x|y)$. Then $\mathcal{I}(X; Y) \geq \mathcal{I}^{\textup{BA}}(X; Y)$.
  If we let $q(x|y) = \frac{p(x)e^{f(x,y)}}{Z(y)}$ where $f(x,y)$ is the witness function and $Z(y) = \mathbb{E}_{p(x)}[e^{f(x,y)}]$ is the partition,  define the unnormalized Barber-Agakov bound as $\mathcal{I}^{\textup{UBA}}(X; Y) = \mathbb{E}_{p(x,y)}\left[ f(x,y) - \log Z(y) \right]$ Then, $\mathcal{I}(X; Y) \geq \mathcal{I}^{\textup{UBA}}(X; Y)$.
  \label{lemma:ba}
\end{lem}
\begin{proof}
  We first observe that $\mathcal{I}(X; Y) = \mathbb{E}_{p(x,y)}\left[\log \frac{p(x|y)}{p(x)}\right] = \mathbb{E}_{p(x,y)}\left[\log \frac{q(x|y)p(x|y)}{p(x)q(x|y)}\right] = \mathbb{E}_{p(x,y)}\left[\log \frac{q(x|y)}{p(x)}\right] + \mathcal{D}_{\text{KL}}(p(x|y)||q(x|y)) \geq \mathbb{E}_{p(x,y)}\left[\log \frac{q(x|y)}{p(x)} \right] = \mathcal{I}^{\textup{BA}}(X; Y)$.
  For the unnormalized bound, use the definition for $q(x|y)$. Lastly, $\mathcal{I}(X; Y) \geq \mathbb{E}_{p(x,y)}\left[\log \frac{q(x|y)}{p(x)} \right] = \mathbb{E}_{p(x,y)}\left[\log \frac{p(x)e^{f(x,y)}}{p(x)Z(y)}\right] = \mathbb{E}_{p(x,y)}\left[f(x,y) - \log Z(y)\right] = \mathcal{I}^{\textup{UBA}}(X; Y)$.
\end{proof}
The log-partition $\log Z(y)$ is difficult to evaluate. The next estimator bounds it instead.
\begin{lem}
Define the Nguyen-Wainwright-Jordan bound as $\mathcal{I}^{\textup{NWJ}}(X; Y) = \mathbb{E}_{p(x,y)}\left[f(x,y)\right] - \mathbb{E}_{p(y)}\left[\frac{1}{e} Z(y)\right]$ where $Z(y)$ is defined as in Lemma~\ref{lemma:ba}. Then $\mathcal{I}(X; Y) \geq \mathcal{I}^{\textup{NWJ}}(X; Y)$.
\label{lem:nwj}
\end{lem}
\begin{proof}
  For $a(y)$ positive we use the bound $\log Z(y) \leq \frac{Z(y)}{a(y)} + \log a(y) - 1$, tight when $a(y) = Z(y)$.
  Use this on $\mathcal{I}^{\textup{UBA}}(X; Y)$. $\mathcal{I}(X; Y) \geq \mathcal{I}^{\textup{UBA}}(X; Y) = \mathbb{E}_{p(x,y)}\left[ f(x,y) \right] - \mathbb{E}_{p(y)}\left[\log Z(y) \right] \geq \mathbb{E}_{p(x,y)}\left[ f(x,y) \right] - \mathbb{E}_{p(y)}\left[\frac{Z(y)}{a(y)} + \log a(y) - 1 \right] = \mathbb{E}_{p(x,y)}\left[ f(x,y) \right] - \mathbb{E}_{p(y)}\left[\frac{1}{a(y)}\mathbb{E}_{p(x)}[e^{f(x,y)}] + \log a(y) - 1 \right]$.
  Now, choose $a(y) = e$. Then,
  $\mathcal{I}(X; Y) \geq \mathbb{E}_{p(x,y)}\left[ f(x,y) \right] - \mathbb{E}_{p(y)}\left[\frac{1}{e}\mathbb{E}_{p(x)}[e^{f(x,y)}] + \log e - 1 \right] = \mathbb{E}_{p(x,y)}\left[ f(x,y) \right] - \mathbb{E}_{p(y)}\left[\frac{Z(y)}{e} \right] = \mathcal{I}^{\textup{NWJ}}(X; Y)$
\end{proof}
Finally, we are ready to show that the InfoNCE estimator lower bounds mutual information.
One of primary differences between the NWJ and the InfoNCE estimators is that the latter reuses the sample $y$ from the joint $p(x,y)$ in estimating the partition.
\begin{thm}
  Let $X$, $Y_1,...,Y_K$ be two random variables with the $Y_j$ i.i.d. for $j = 1,...,K$.
  Let $K - 1$ be the number of negative samples.
  Define $p(x, y_1)$ as the joint distribution for $X$ and $Y_1$.
  Define $p(y_{2:K}) = \prod_{j=2}^K p(y_j)$ as the distribution over negative samples. Fix $p(y_j) = p(y_1)$, the marginal distribution for $Y_1$.
  Recall the InfoNCE estimator
  \begin{equation*}
    \scriptstyle
    \mathcal{I}^{\textup{NCE}}(X;Y_1) = \mathbb{E}_{p(x,y_1)}\left[ f_{\theta,\phi}(x,y_1) - \mathbb{E}_{p(y_{2:K})}\left[ \log \frac{1}{K} \sum_{j=1}^K e^{f_{\theta,\phi}(x,y_j)} \right] \right]
  \end{equation*}
  . Then, $\mathcal{I}(X;Y_1) \geq \mathcal{I}^{\textup{NCE}}(X,Y_1)$.
  That is, the InfoNCE objective is a lower bound for the mutual information between $X$ and $Y_1$.
  \label{thm:infonce}
\end{thm}

\begin{proof}
  We will show that the InfoNCE bound arises from augmenting $y_1$ with a set of auxiliary samples $ y_2,...,y_K$ and employ the Nguyen-Wainwright-Jordan bound.

 Let $Y_{2:K}$ be a set of random variables distributed independently according to $p(y_1)$. As $Y_1$ and $Y_{2:K}$ are independent, $\mathcal{I}(X; Y_1) = \mathcal{I}(X; Y_1, Y_{2:K})$.
 Now assume a witness function: $f(x, y_{1:K}) = 1 + \log \frac{e^{f(x,y_1)}}{a(x, y_{1:K})}$, where $a(x, y_{1:K}) = \frac{1}{K}\sum_{i=1}^K e^{f(x, y_i)}$, a Monte Carlo approximation of the partition $Z(y)$. Apply the Nguyen-Wainwright-Jordan bound.
 $\mathcal{I}(X; Y_1) = \mathcal{I}(X; Y_1, Y_{2:K}) \geq \mathbb{E}_{p(x,y_{1:K})}\left[ 1 + \log \frac{e^{f(x,y_1)}}{a(x, y_{1:K})}\right] - \mathbb{E}_{p(x)p(y_{1:K})}\left[\frac{1}{e} \cdot e^{1 + \log \frac{e^{f(x,y_1)}}{a(x, y_{1:K})}}\right] = 1 + \mathbb{E}_{p(x, y_{1:K})}\left[ \log \frac{e^{f(x,y_1)}}{a(x, y_{1:K})}\right] - \mathbb{E}_{p(x)p(y_{1:K})}\left[\frac{e^{f(x,y_1)}}{a(x, y_{1:K})}\right]$
Because $p(y_{1:K})=\prod_i p(y_i)$ is invariant under permutation of indices ${1:K}$, we have $\mathbb{E}_{p(x)p(y_{1:K})}\left[\frac{e^{f(x,y_1)}}{a(x, y_{1:K})}\right] = \mathbb{E}_{p(x)p(y_{1:K})}\left[\frac{e^{f(x,y_i)}}{a(x, y_{1:K})}\right]$ for any $i$. Use this and the definition of $a$: $\mathbb{E}_{p(x)p(y_{1:K})}\left[\frac{e^{f(x,y_1)}}{a(x, y_{1:K})}\right] = \frac{1}{K}\sum_{i=1}^K \mathbb{E}_{p(x)p(y_{1:K})}\left[ \frac{e^{f(x,y_i)}}{a(x, y_{1:K})} \right] = \mathbb{E}_{p(x)p(y_{1:K})}\left[\frac{\frac{1}{K}\sum_{i=1}^K e^{f(x,y_i)}}{a(x, y_{1:K})} \right] = 1$
We now substitute this result into full equation above: $\mathcal{I}(X; Y_1) \geq 1 + \mathbb{E}_{p(x,y_{1:K})}\left[ \log \frac{e^{f(x,y_1)}}{a(x, y_{1:K})}\right] - 1 = \mathbb{E}_{p(x,y_{1:K})}\left[f(x,y_1) - \log \frac{1}{K}\sum_{i=1}^K e^{f(x, y_i)} \right] = \mathbb{E}_{p(x,y_1)} \mathbb{E}_{p(y_{2:K})}  \left[f(x,y_1) - \log \frac{1}{K}\sum_{i=1}^K e^{f(x, y_i)} \right] = \mathcal{I}^{\textup{NCE}}(X; Y_1)$
\end{proof}

\section{Additional lemmas}
In the main text, we owe the reader a few additional lemmas.

\begin{lem}
Define the SimCLR objective (with explicit data augmentation) as $\mathcal{L}^{\textup{SimCLR}}(x_1) = \mathbb{E}_{a_{1:K}, b, b' \sim p(a)}\left[\frac{e^{h_{\psi}(u_{\phi}(\nu(x_1, b)))^T h_{\psi}(u_{\phi}(\nu(x_1, b'))) / \omega}}{\sum_{j=1}^K e^{h_{\psi}(u_{\phi}(\nu(x_1, b)))^T h_{\psi}(u_{\phi}(\nu(x_j, a_j))) / \omega}}\right]$ where $h_{\psi}$ and $u_{\phi}$ are two encoders, usually neural networks, and $\nu$ is a view function. Assume we have a dataset $\mathcal{D}$ sampled i.i.d. from an empirical distribution $p_{\mathcal{D}}$. Let $K$ be the size of minibatch. Then, $\mathbb{E}_{p_{\mathcal{D}}(x_1)}\left[\mathcal{L}^{\text{SimCLR}}(x_1)\right] + \log K = \hat{\mathcal{I}}^{\text{NCE}}(X; \nu)$, an empirical estimate of InfoNCE between two views of $X$.
\label{lem:simclr}
\end{lem}
\begin{proof}
     Let the empirical estimate of InfoNCE between two views be $\hat{\mathcal{I}}^{\text{NCE}}(X; \nu) = \mathbb{E}_{p_{\mathcal{D}}(x_{1:K})}\mathbb{E}_{a_{1:N},b,b^\prime\sim p(a)}\left[\log \frac{e^{f_{\theta}(\nu(x_1,b), \nu(x_1, b^\prime))}}{\frac{1}{K}\sum_{j=1}^K e^{f_{\theta}(\nu(x_1,b), \nu(x_j, a_{j})) }} \right]$. Now, define $g_\theta = h_{\psi} \circ u_{\phi}$ where $\theta = \psi \cup \phi$, the composition of the two functions. Now, we see
     \begin{align*}
         \mathcal{L}^{\textup{SimCLR}}(x_1) + \log K&= \mathbb{E}_{a_{1:K}, b, b' \sim p(a)}\left[\frac{e^{h_{\psi}(u_{\phi}(\nu(x_1, b)))^T h_{\psi}(u_{\phi}(\nu(x_1, b'))) / \omega}}{\sum_{j=1}^K e^{h_{\psi}(u_{\phi}(\nu(x_1, b)))^T h_{\psi}(u_{\phi}(\nu(x_j, a_j))) / \omega}}\right]  + \log K \\
         &= \mathbb{E}_{a_{1:K}, b, b' \sim p(a)}\left[ \frac{e^{g_\theta(\nu(x_1, b))^T g_\theta(\nu(x_1, b'))}}{\frac{1}{K}\sum_{j=1}^K e^{g_\theta(\nu(x_1, b))^T g_\theta(\nu(x_1, a_j))}} \right] \\
         &=  \mathbb{E}_{a_{1:K}, b, b' \sim p(a)}\mathbb{E}_{p_{\mathcal{D}}(x_{2:K})}\left[ \frac{e^{g_\theta(\nu(x_1, b))^T g_\theta(\nu(x_1, b'))}}{\frac{1}{K}\sum_{j=1}^K e^{g_\theta(\nu(x_1, b))^T g_\theta(\nu(x_j, a_j))}} \right]
     \end{align*}
     where the last line holds since in SimCLR, $x_2, \ldots x_K$ are the other samples in the same minibatch as $x_1$. Since elements in a minibatch are chosen uniformly from the dataset $\mathcal{D}$, they are sampled independently from the empirical distribution. Then result follows if we add $\mathbb{E}_{p_\mathcal{D}(x_1)}$ to both sides.
\end{proof}

\begin{lem}
Let $\mathcal{L}^T(x_1,M)$ be as in lemma~\ref{lem:coro_ballann} and $K$ be the number of negative samples. We explicitly include to $q_T$ as argument to $\mathcal{I}^{\textup{VINCE}}(X; \nu^*, q_T)$ to explicitly show the variational distribution the VINCE estimator is using. Then
$\mathcal{I}^{\textup{VINCE}}(X; \nu^*, q_{[\tau, \gamma]}) \equiv \mathcal{L}^{[\tau, \gamma]}(x_1,M) + \log K$ and $\mathcal{I}^{\textup{VINCE}}(X; \nu^*, q_{[\tau, \infty) \setminus S}) \equiv \mathcal{L}^{[\tau, \infty) \setminus S}(x_1,M) + \log K$ where $S$ is any subset of $[\tau, \infty)$. As such, RING and CAVE both lower bound mutual information between weighted view sets.
\label{lem:ring_cave}
\end{lem}
\begin{proof}
    The proof in Sec.~\ref{sec:proof:lemmaball} showed an equivalence between VINCE and $T$-Discrimination for any set $T$ lower bounded by $\tau$. We note that RING is equivalent to $T$-Discrimination with $T = [\tau, \gamma]$ and CAVE is equivalent to $T = [\tau, \infty) \setminus S$ where $S$ is the set of elements in a dataset $\mathcal{D}$ with same cluster label as $x_1$, the current image as assigned by K-Means. So by lemma~\ref{lem:coro_ballann}, RING and CAVE are equivalent to VINCE with $q_{[\tau, \gamma]}$ and $q_{[\tau, \infty)\setminus S}$ as the variational distribution, respectively. By Thm.~\ref{thm:vince}, we conclude that both bound the mutual information between weighted view sets.
\end{proof}

\begin{lem}
Local Aggregation is the difference of two \textsc{logsumexp} terms.
\label{lem:twologumexp}
\end{lem}
\begin{proof}
\begin{equation*}
    \scriptstyle
    \mathcal{L}^{\textup{LA}}(x^{(i)}, M)
    = \mathbb{E}_{p(a)} \left[ \log \left( \frac{\frac{\sum_{k \in C^{(i)}} e^{g_\theta(\nu(x^{(i)}, a))^T M[k] / \omega} }{\sum_{j=1}^N e^{g_\theta(\nu(x^{(i)}, a))^T M[j] / \omega}}}{\frac{\sum_{k \in B^{(i)}} e^{g_\theta(\nu(x^{(i)}, a))^T M[k] / \omega} }{\sum_{j=1}^N e^{g_\theta(\nu(x^{(i)}, a))^T M[j] / \omega}}} \right)\right] = \mathbb{E}_{p(a)}\left[ \log  \frac{\sum_{k \in C^{(i)}} e^{g_\theta(\nu(x^{(i)}, a))^T M[k] / \omega}}{\sum_{k \in B^{(i)}} e^{g_\theta(\nu(x^{(i)}, a))^T M[k] / \omega}} \right]
\end{equation*}
\end{proof}

\section{Training Details}

\subsection{Experiment shown in Fig.~\ref{fig:viewset_lesion} and Fig.~\ref{fig:mi_views}}
For each point, we train using the IR objective ($t = 0.5$, $\tau=0.07$) for 100 epochs (no finetuning) with SGD (momentum $0.9$, batch size $256$, weight decay 1e-4, and learning rate $0.03$). We use 4096 negative samples train from the marginal distribution (training dataset). To estimate transfer accuracy, for each image in the test set, we embed it using our learned encoder and then find the label of the nearest neighbor in the training dataset with the memory bank. To vary the crop and color jitter hyperparameters we use the functions in  \texttt{torchvision.transforms}.

\subsection{Section ~\ref{sec:mainresults} Experiments}
\label{sec:mainappendix_sec5}
For these experiments, we follow the training practices in \cite{zhuang2019local,tian2019contrastive}. In short, we used a temperature $\tau = 0.07$ and memory bank update rate $t = 0.5$. For optimization, we used SGD with momentum of $0.9$, batch size $256$, weight decay of 1e-4, learning rate $0.03$.
For pretraining, the learning rate was dropped twice by a factor of $10$ once learning ``converged,'' defined as no improvment in nearest neighbor classification accuracy (computed on validation set after every training epoch) for $10$ epochs. To define the background neighbors, we used $4096$ of the nearest points.
To define the close neighbors, we used 10 kNN with different initializations where $k = 30000$ for ImageNet and $k=300$ for CIFAR10.
To train kNNs, we use the FAISS libary \cite{johnson2019billion}.
For the image model architecture, we used ResNet18 \cite{he2016deep} to prioritize efficiency; future work can investigate larger architectures.
The representation dimension was fixed to 128.
For image augmentations during training, we use a random resized crop to 224 by 224, random grayscale with probability 0.2, random color jitter, random horizontal flip, and normalize pixels using ImageNet statistics.
When training LA on ImageNet, we initialize ResNet and the memory bank using 10 epochs of IR --- we found this initialization to be crucial to finding a good optima with LA.

Once training the representation is complete, we freeze the parameters and begin the transfer phase by fitting a logistic regression on top of the learned representations to predict the labels (this is the only time the labels are required).  We optimize with SGD with a learning rate of 0.01, momentum 0.9, weight decay 1e-5 and batch size of 256 for 30 epochs. Like in pretraining, we drop the learning rate twice with a patience of 10. In transfer learning, the same augmentations are used as in pretraining. The performance from fitting logistic regression is always higher than the nearest neighbor estimate, which is nonetheless still useful as a cheaper surrogate metric.

For CIFAR10, the same hyperparameters are used as above. Images in CIFAR are reshaped to be 256 by 256 pixels to use the exact same encoder architectures as with ImageNet. When fitting LA or ANN on CIFAR10, we set $k=300$ in clustering with KNN. In both pretraining and transfer for CIFAR10, we drop the learning rate once as we find a doing so a second time does not improve performance.
For CMC experiments, the hyperparameters are again consistent above with a few exceptions. First, we exclude grayscale conversion from the data augmentation list since it does not make sense for CMC. (However, we do keep color jitter which is not done in the original paper; we chose to do so to preserve similarities to IR training setup). Second, the batch size was halved to 128 to fit two ResNet18 models in memory (one for L and one for AB). Thus, one should not compare the learning curves of IR against CMC (which is never done in the main text) but one can compare learning curves of CMC models to each other and IR models to each other. When training to completion (as in Table.~\ref{tab:othertransfertask}), it is fair to compare CMC and IR as both models have converged. Implementation wise, we adapted the official PyTorch implementation on Github.

For BALL experiments on CIFAR10, we define the negative distribution $q_[\tau,\infty)$ as i.i.d samples from the closest 10\% of images in the dataset to the representation of $x_1$. For BALL experiments on ImageNet, we sample from the closest 0.3\%. For RING experiments on CIFAR10, we define the neighbor distribution $q_{[\gamma, \infty)}$ as i.i.d. samples from the the closest 1\% of images in the dataset to the representation of $x_1$. When annealing the negative distribution for BALL, we start by sampling from 100\% of the dataset and linearly anneal that to 10\% for CIFAR10 and 0.3\% for ImageNet over 100 epochs, leaving 100 additional epochs at the final percentage. When annealing the neighbor distribution for RING, we start with an empty close neighborhood and being annealing from 0 to 1\% at epoch 10 to 100, again leaving 100 additional epochs at the final percentage. When doing so, we do not need to hot-start from a trained IR model. We point out that these hyperparmeters were chosen arbitrary and we did not do a systematic search, which we think can uncover better strategies and choices.


\subsection{Section~\ref{sec:mi_insights}: Experiments}
All hyperparameters are as described in Sec.~\ref{sec:mainappendix_sec5} which the exception of the particular hyperparameter we are varying for the experiment. To compare InfoNCE and the original IR formulation, we adapted the public PyTorch implementation found at \url{https://github.com/neuroailab/LocalAggregation-Pytorch}.

\subsection{Detectron2 Experiments}
\label{sec:detectron}
We make heavy usage of the Detectron2 code found at \url{https://github.com/facebookresearch/detectron2}. In particular, the script \url{https://github.com/facebookresearch/detectron2/blob/master/tools/convert-torchvision-to-d2.py} allows us to convert a trained ResNet18 model from \texttt{torchvision} to the format needed for Detectron2. The repository has default configuration files for all the experiments in this section. We change the following fields to support using a frozen ResNet18:
\begin{lstlisting}
INPUT:
    FORMAT: RGB
MODEL:
    BACKBONE:
        FREEZE_AT: 5
    PIXEL_MEAN:
        - 123.675
        - 103.53
        - 116.28
    PIXEL_STD:
    - 58.395
    - 57.12
    - 57.375
    RESNETS:
        DEPTH: 18
        RES2_OUT_CHANNELS: 64
        STRIDE_IN_1X1: false
    WEIGHTS: <PATH_TO_CONVERTED_TORCHVISION_WEIGHTS>
\end{lstlisting}
We acknowledge that ResNet50 and larger are the commonly used backbones, so our results will not be  state-of-the-art. However, the ordering in performance between algorithms is still meaningful and our primary interest.
For COCO object detection and instance segmentation, we use Mask R-CNN, R$_{18}$-FPN, 1x schedule. For COCO keypoint detection, we use R-CNN, R$_{18}$-FPN. For Pascal VOC '07, we use Faster R-CNN, R$_{18}$-C4. For LVIS, we use Mask R-CNN, R$_{18}$-FPN.

\section{More Background}
\paragraph{A Simple Framework for Contrastive Learning (SimCLR)} SimCLR \cite{chen2020simple} performs an expansive experimental study of how data augmentation, architectures, and computational resources effect the IR objective. In particular, SimCLR finds better performance without a memory bank and by adding a nonlinear transformation on the representation before the dot product. That is,
\begin{equation}
    \mathcal{L}^{\text{SimCLR}}(x^{(i)}) = \log p(i|x^{(i)}),\qquad p(i|x^{(i)}) = \frac{e^{h_\psi(g_\theta(x^{(i)}))^T h_\psi(g_\theta(x^{(i)})) / \omega}}{\sum_{j=1}^N e^{h_\psi(g_\theta(x^{(i)}))^T h_\psi(g_\theta(x^{(j)})) / \omega}}
    \label{eq:simclr}
\end{equation} where $h_\psi: \mathbb{R}^d \rightarrow \mathbb{R}^d$ is usually an MLP. Note no memory banks are used in Eq.~\ref{eq:simclr}. Instead, the other elements in the same minibatch are used as negative samples. Combining these insights with significantly more compute, SimCLR achieves the state-of-the-art on transfer to ImageNet classification. In this work, we find similar takaways as \cite{chen2020simple} and offer some theoretical support.

\section{Toy Example of VINCE}
\label{sec:toyvince}
Thm.~\ref{thm:vince} and Coro.~\ref{coro:vince} imply an ordering to the MI bounds considered. We next explore the tightness of these bounds in a toy example with known MI.
Consider two random variables $Z$ and $\epsilon$ distributed such that we can pick independent samples $z_i \sim \mathcal{N}(0, \Sigma_Z)$ and $\epsilon_i \sim \mathcal{N}(0, \Sigma_\epsilon)$ where $\Sigma_Z = \begin{pmatrix}
1 & -0.5 \\
-0.5  & 1
\end{pmatrix}$ and $\Sigma_\epsilon = \begin{pmatrix}
1 & 0.9 \\
0.9 & 1
\end{pmatrix}$
Then, let $(X, Y) = Z + \epsilon$. That is, introduce a new random variable $X$ as the first dimension of the sum and $Y$ as the second. The mutual information between $X$ and $Y$ can be analytically computed as
$\mathcal{I}(X; Y) = -\frac{1}{2}\log (1- \frac{\Sigma[1,2]\Sigma[2,1]}{\Sigma[1,1]\Sigma[2,2]})$
 as $(X, Y)$ is jointly Gaussian with covariance $\Sigma = \Sigma_Z + \Sigma_\epsilon$.

We fit InfoNCE and VINCE with varying values for $\tau$ and measure the estimated mutual information.
Table~\ref{table:toy1} shows the results. For rows with VINCE, the percentage shown in the left column represents the percentage of training examples (sorted by L$_2$ distance in representation space to $y_1$) that are ``valid'' negative samples. For instance, 50\% would indicate that negative samples can only be drawn from the 50\% of examples whose representation is closest to $g_\theta(y_1)$ in distance.
Therefore, a smaller percentage is a higher $\tau$. From Table~\ref{table:toy1}, we find that a higher $\tau$ results in a looser estimate (as expected from Coro.~\ref{coro:vince}). It might strike the reader as peculiar to use VINCE for representation learning --- it is a looser bound than InfoNCE, and in general we seek tightness with bounds.
However, we argue that learning a good representation and tightening the bound are two subtly related but fundamentally distinct problems.

\begin{table}[h!]
  \small
  \caption{Looseness of VINCE}
  \centering
    \begin{tabular}{ll}
        \toprule
        Method & Estimate of MI\\
        \midrule
        True & $0.02041$ \\
        InfoNCE & $0.01345 \pm 0.001$ \\
        VINCE (90\%) & $0.01241 \pm 3$e$-4$ \\
        VINCE (75\%) & $0.00220 \pm 1$e$-4$ \\
        VINCE (50\%) & $7.29$e$-5 \pm 9$e$-6$ \\
        VINCE (25\%) & $1.67$e$-5 \pm 2$e$-6$\\
        VINCE (10\%) & $5.87$e$-6 \pm 1$e$-6$\\
        VINCE (5\%) & $1.97$e$-6 \pm 4$e$-6$\\
        \bottomrule
    \end{tabular}
 \label{table:toy1}
\end{table}

The encoders are 5-layer MLPs with 10 hidden dimensions and ReLU nonlinearities. To build the dataset, we sample 2000 points and optimize the InfoNCE objective with Adam with a learning rate of 0.03, batch size 128, and no weight decay for 100 epochs. Given a percentage for VINCE, we compute distances between all elements in the memory bank and the representation the current image --- we only sample 100 negatives from the top $p$ percent. We conduct the experiment with 5 different random seeds to estimate the variance.

\section{Revisiting Mutual Information and Data Augmentation}
\label{sec:mi_dataaug}
In Thm.~\ref{thm:marketing}, we discussed the notion that not all the optima of IR are good for transfer learning. Having drawn connections between IR and MI, we can revisit this statement from an information-theoretic angle. Lemma~\ref{lem:ir_mi} shows that contrastive learning is the same as maximizing MI for the special case where $X$ and $Y$ are the same random variable. As $\mathcal{I}(X; X)$ is trivial, optimizing it should not be fruitful.
But we can understand ``views'' as a device for information bottleneck that force the encoders to estimate MI under missing information: the stronger the bottleneck, the more robust the representation as it learns to be invariant.
Critically, if the data augmentations are lossy, then the mutual information between two views of an image  is not obvious, making the objective nontrivial.
At the same time, if we were to train IR with lossless data augmentations, we should not expect any learning beyond inductive biases introduced by the neural network $g_\theta$.

\begin{figure}[h!]
    \centering
    \begin{subfigure}[b]{0.19\textwidth}
        \centering
        \includegraphics[width=\textwidth]{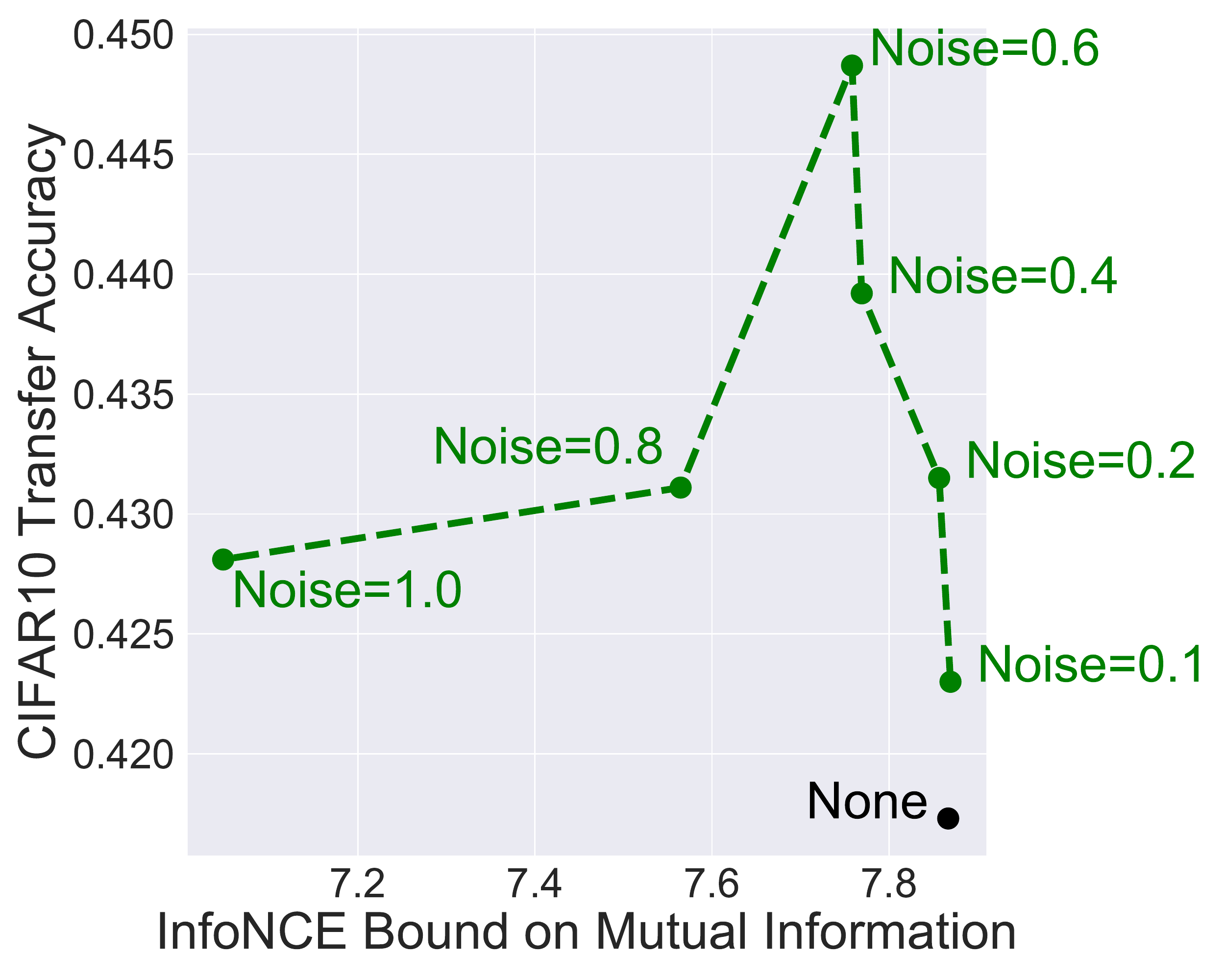}
        \caption{Noise Level}
    \end{subfigure}
    \begin{subfigure}[b]{0.19\textwidth}
        \centering
        \includegraphics[width=\textwidth]{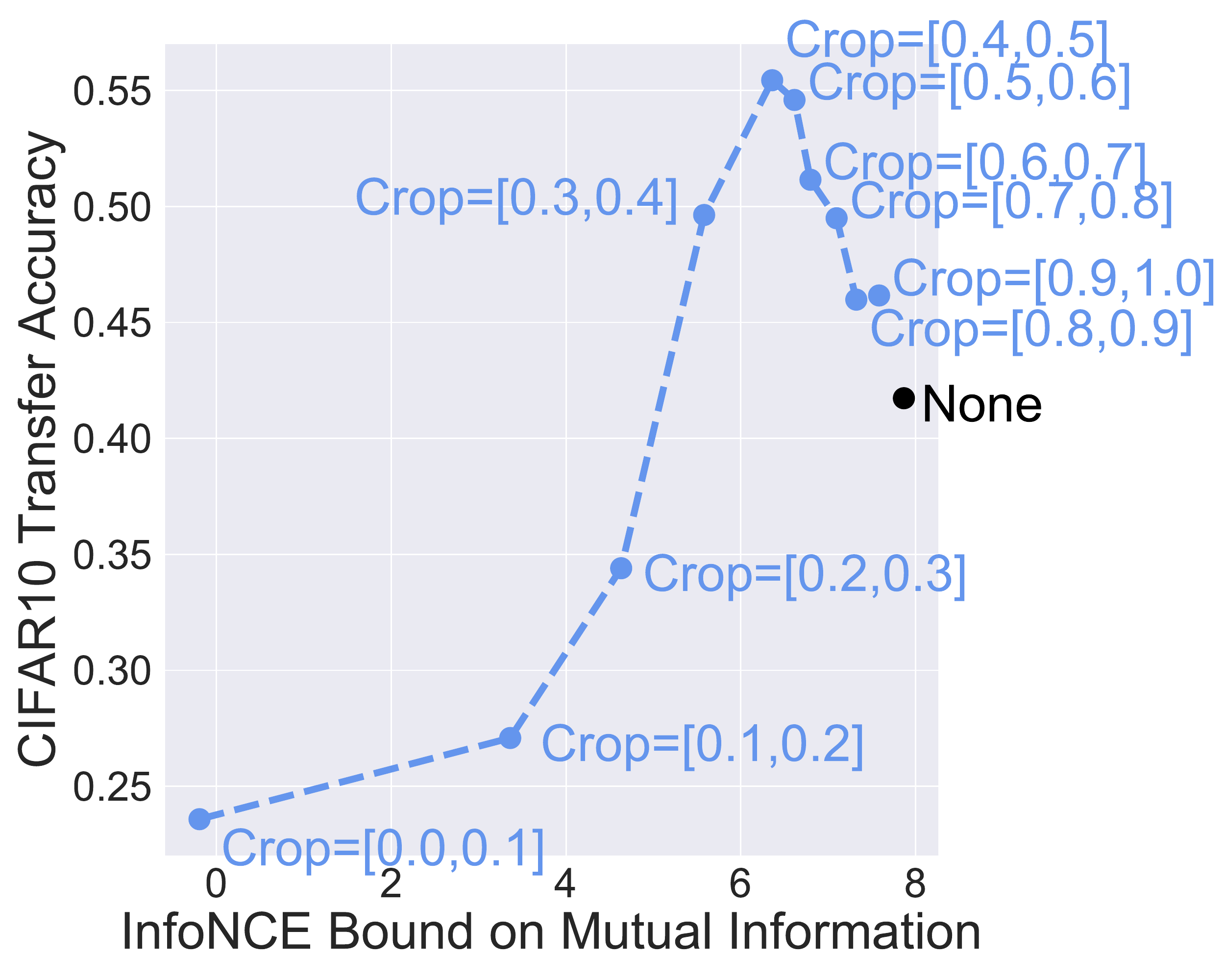}
        \caption{Crop Intervals}
    \end{subfigure}
    \begin{subfigure}[b]{0.19\textwidth}
        \centering
        \includegraphics[width=\textwidth]{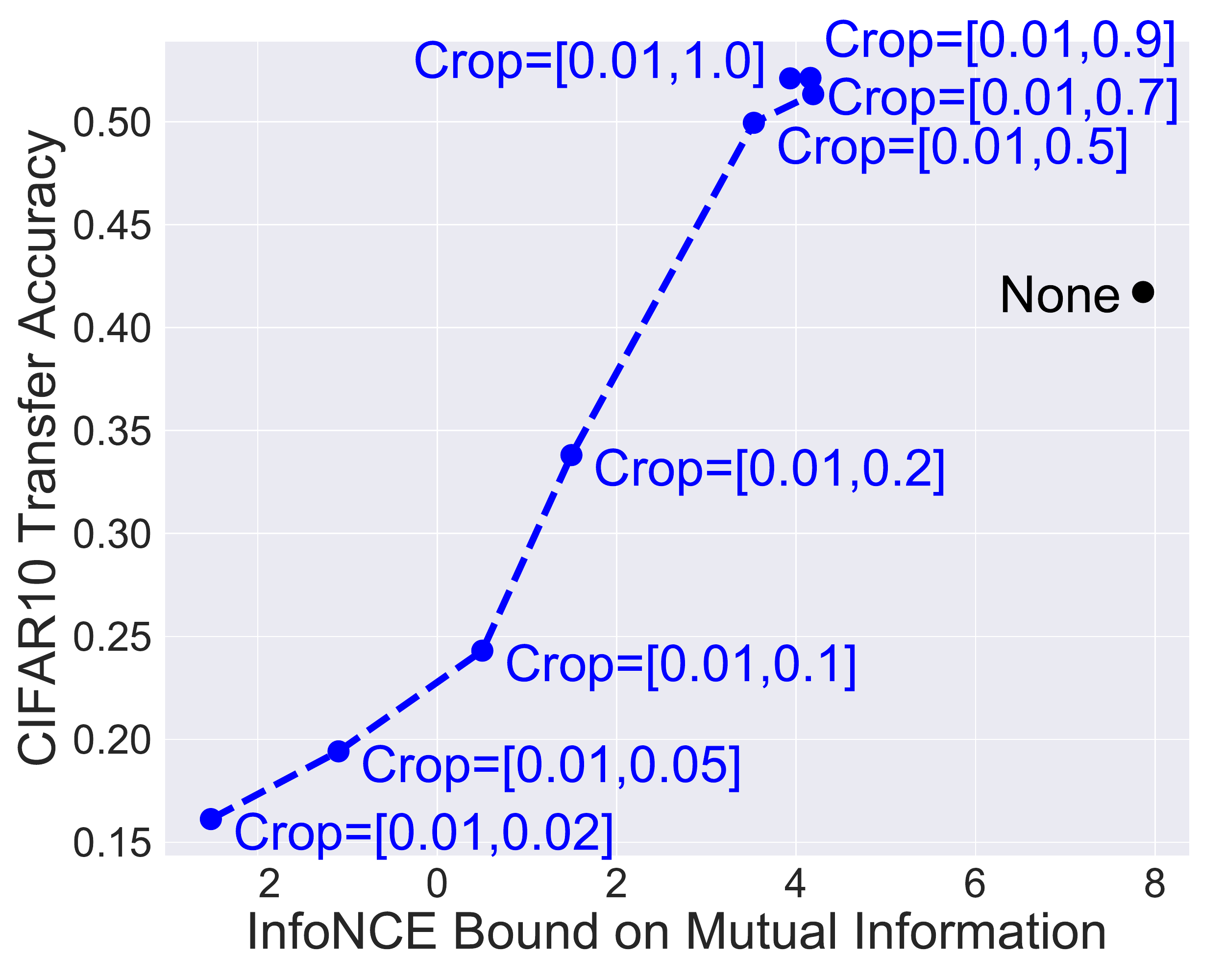}
        \caption{Crop Min.}
    \end{subfigure}
    \begin{subfigure}[b]{0.19\textwidth}
        \centering
        \includegraphics[width=\textwidth]{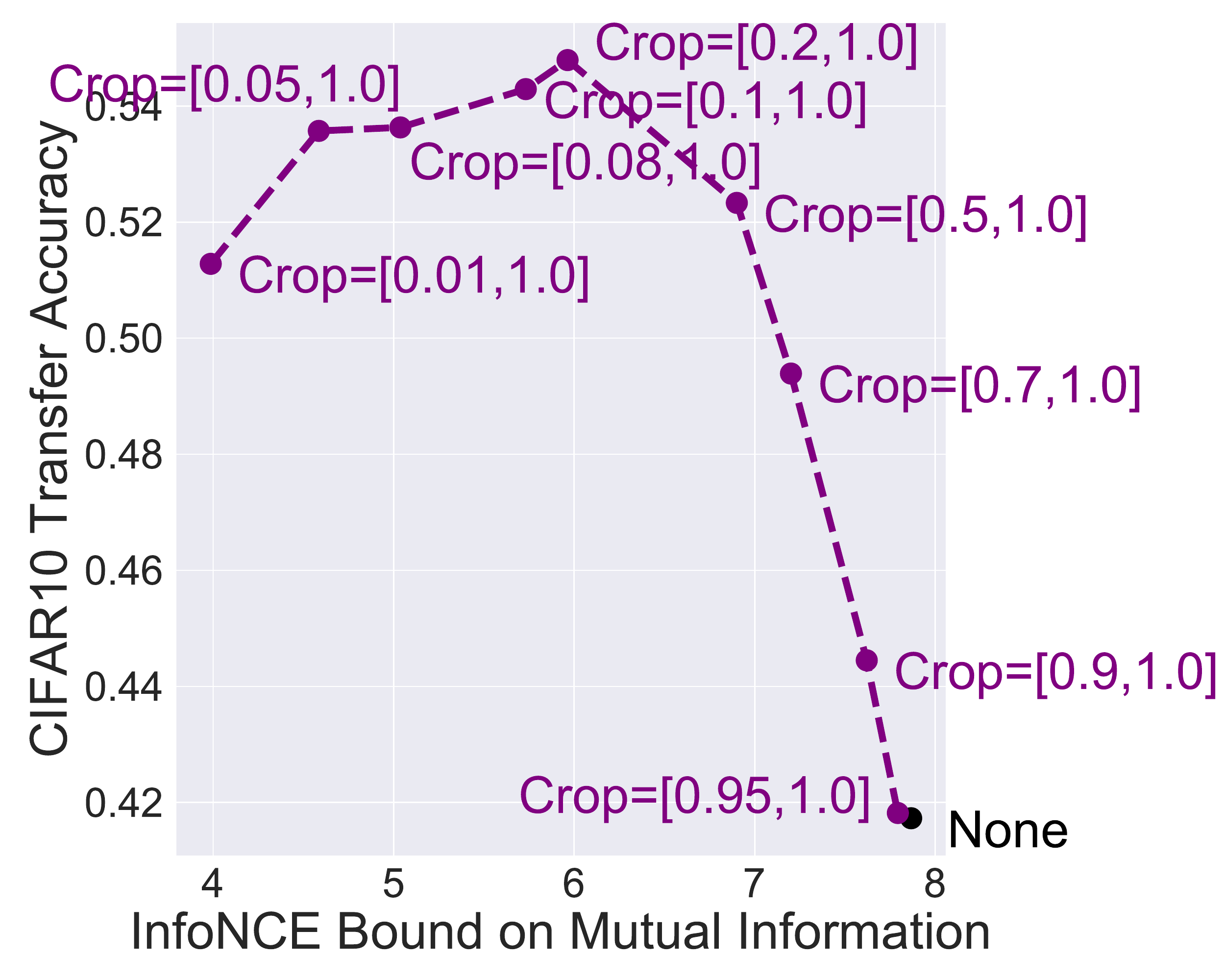}
        \caption{Crop Max.}
    \end{subfigure}
    \begin{subfigure}[b]{0.19\textwidth}
        \centering
        \includegraphics[width=\textwidth]{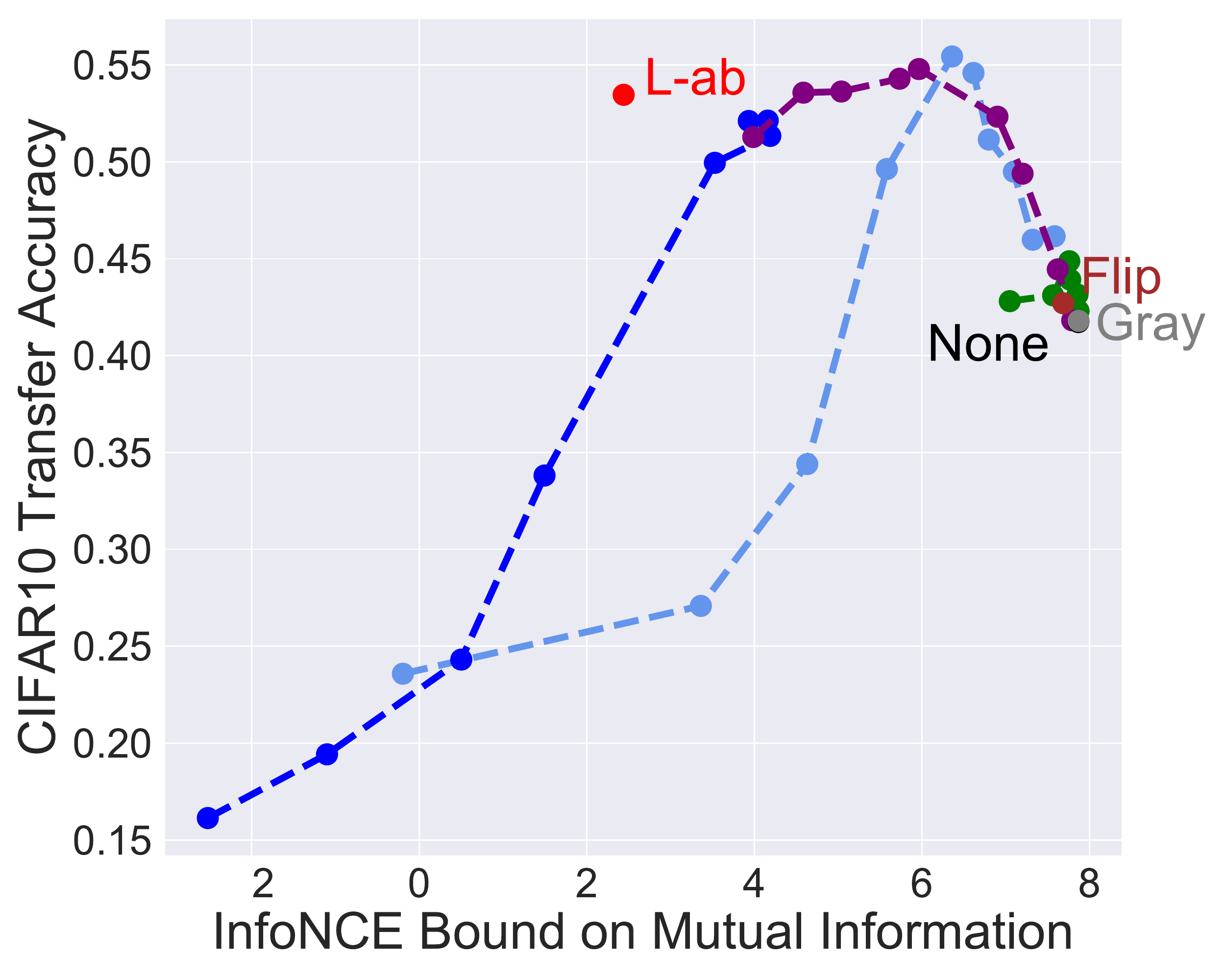}
        \caption{Combined}
    \end{subfigure}
    \caption{Nearest neighbor classification accuracy versus MI on CIFAR10 under different view sets. View sets that hide too little (e.g. grayscale, flip) or too much information (e.g. crops that preserve too little) result in poor transfer performance. Subfigure (e) combines (a) through (d) for scale.}
    \label{fig:mi_views}
\end{figure}

Fig.~\ref{fig:mi_views} depicts a more careful experiment studying the relationship been views and MI. Each point in Fig.~\ref{fig:mi_views} represents an IR model trained on CIFAR10 with different views: no augmentations (black point), grayscale images (gray point),  flipped images (brown point), L-ab filters (red point), color jitter (green line) where Noise=$x$ means adding higher levels of noise for larger $x$, and cropping (blue and purple lines) where the bounds $[a, b]$ represent the minimum and maximum crop sizes with 0 being no image at all and 1 retaining the full image. By Lemma~\ref{lem:ir_mi}, we estimate MI (x-axis of Fig.~\ref{fig:mi_views}) using $\mathcal{L}^{\textup{IR}}$ plus a $\log N$ constant. We see a parabolic relationship between MI and transfer performance: views that preserve more information lead to both a higher MI (trivially) and poorer representations. Similarly, views that hide too much information lead to very low MI and again, poor representations.
It is a sweet spot in between where we find good classification performance.
Similar findings were reported in \cite{tian2020makes}.

\textbf{Learning curves of different augmentations. }In the main text, we tracked the nearest neighbor classification accuracy throughout training for various augmentations using IR and CMC on ImageNet. Here, we include results on CIFAR10, which show similar patterns as discussed.

\begin{figure}[h!]
  \centering
  \begin{subfigure}[b]{0.24\textwidth}
    \centering
    \includegraphics[width=\textwidth]{ir_view_lesion_imagenet.pdf}
    \caption{IR (ImageNet)}
    \vspace{0.15em}
  \end{subfigure}
    \centering
  \begin{subfigure}[b]{0.24\textwidth}
    \includegraphics[width=\textwidth]{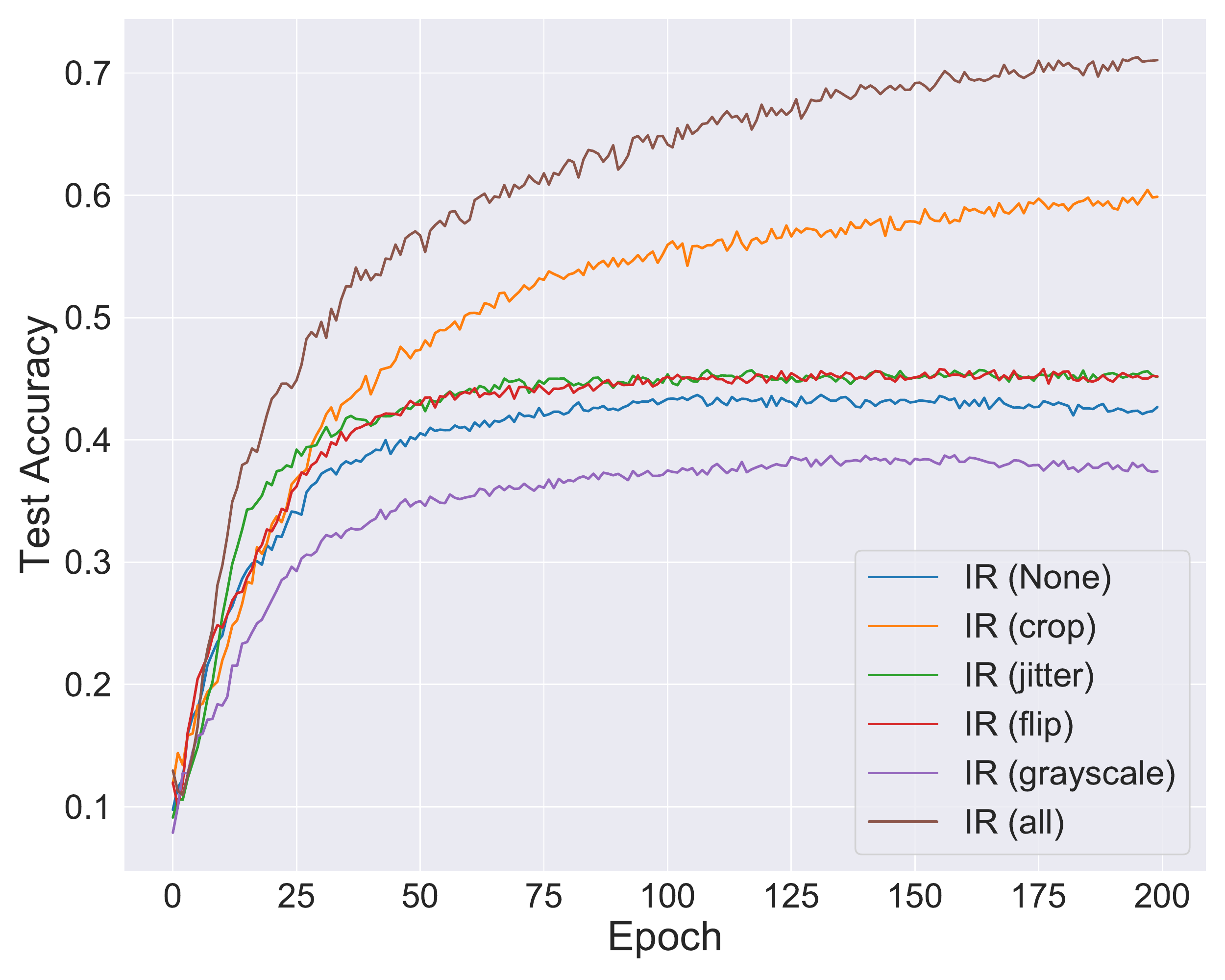}
    \caption{IR (CIFAR10)}
    \vspace{0.15em}
  \end{subfigure}
  \begin{subfigure}[b]{0.24\textwidth}
    \centering
    \includegraphics[width=\textwidth]{cmc_view_lesion_imagenet.pdf}
    \caption{CMC (ImageNet)}
  \end{subfigure}
  \begin{subfigure}[b]{0.24\textwidth}
    \centering
    \includegraphics[width=\textwidth]{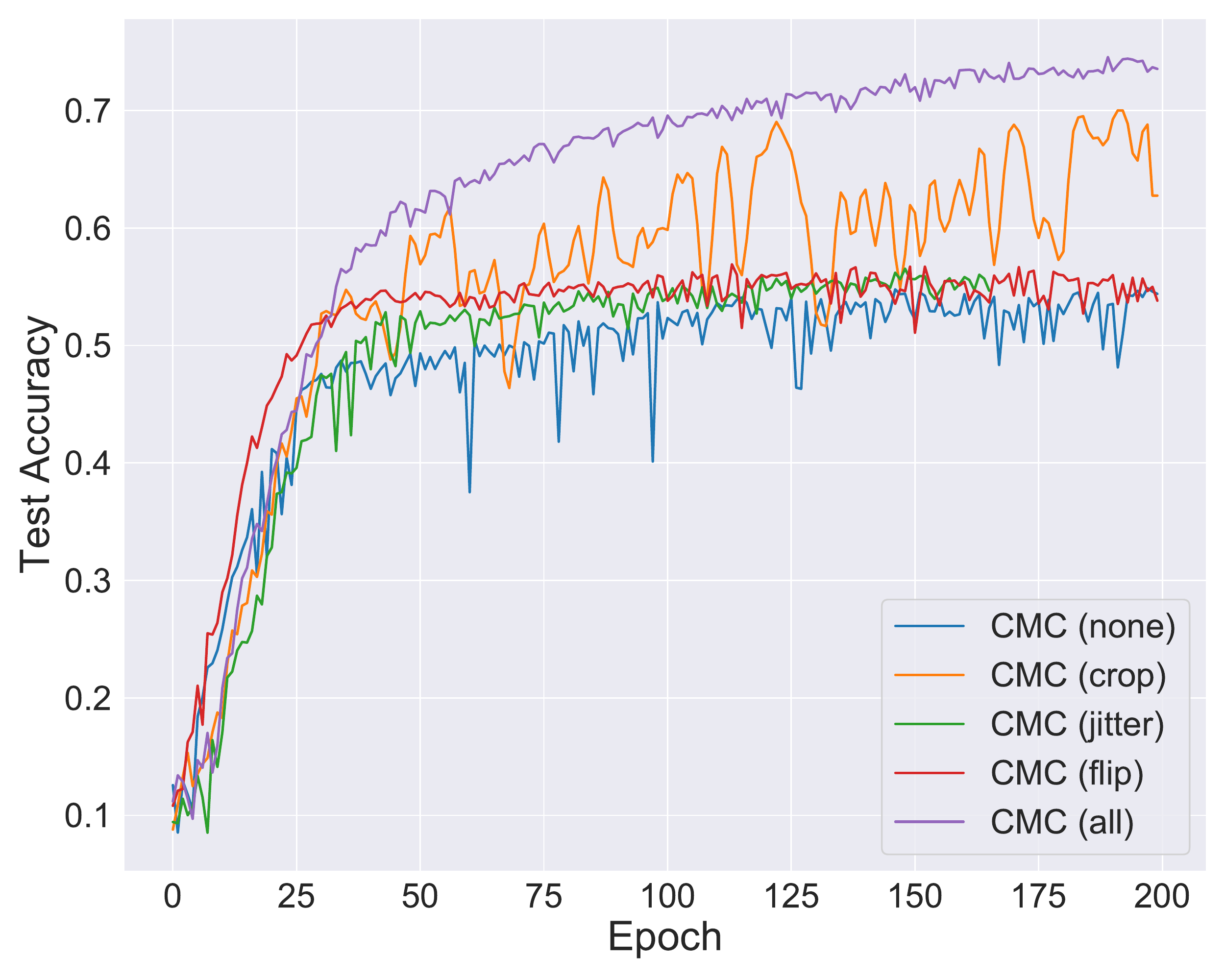}
    \caption{CMC (CIFAR10)}
  \end{subfigure}
  \caption{Effect of view set choice on representation quality for IR and CMC on ImageNet. Empty or trivial augmentations (e.g. flipping) lead to poor performance.}
  \label{fig:viewset_lesion_dup}
\end{figure}

\section{Properties of a Good View Set}
\label{sec:properties}
Recall Fig.~\ref{fig:viewset_lesion}, which measured the quality of different representations learned under different view sets. We continue our analysis to summarize the results with a few properties that differentiate ``good'' and ``bad'' views (data augmentation techniques). Property 1 was mentioned in the main text.

\paragraph{Property \#1: (Lossy-ness)}
We can frame views as a device for information bottleneck that force the encoders to estimate mutual information under missing information: the stronger the bottleneck, the more robust the representation.
Further, if the view functions are lossy, then the mutual information $\mathcal{I}(X; X)$ is not obvious, making the objective nontrivial.
Some evidence of this as a desirable property can be seen in Fig.~\ref{fig:viewset_lesion}: lossless view functions (such as horizontal flipping or none at all) result in poor representations where as more lossy view functions (such as cropping and noise addition) form much stronger representations than lossless and less lossy view functions (e.g. grayscale).

\paragraph{Property \#2: (Completeness)}
Different views should differ in the information they contain (otherwise the objective is again trivial).
Consider the training paradigm for IR: every epoch, for a given $x$, we must choose a random view.
In the limit, this paradigm would \textit{amortize} the parameters of the encoders over all views in $v_x$.
This compounds the amount of information that must be compressed in the encoded representation, encouraging abstraction and invariance.
With this interpretation, completeness is important as we want to be invariant to the full spectrum of views --- otherwise, the learned representation would contain ``holes'' along certain transformations that may appear in unseen images. For instance, consider a view set of center crops versus a view set of all random crops.

\paragraph{Property \#3: (Uniqueness)}
Given an image, imagine we scramble all pixels to random values.
This would clearly make it difficult for the witness function to properly access similarity (and be maximally lossy) but so much so that no meaningful representation can be learned.
To avoid this, we demand the view set of a datum to be \textit{unique} to that datum alone.
In other words, the elements of view set $v_x$ for a data point $x$ should not appear in the view set $v_y$ for any $y$.
For instance, cropping an image or adding color jitter does not fully obscur the identity of the original image.

\paragraph{Property \#4: (Consistency)}
For a datum $x$, a view $v(x,a)$ for any $a \in \mathcal{A}$ must be in the same domain as $x$.
For instance, the view of an image cannot be a piece of text.
In this sense, we can consider $x$ to be a \textit{priviledged view} (of itself).
Because amortizing over views increases the generalization capability of our encoders to elements of $v_x$ (and since $x$ has membership in $v_x$), we expect amortization to result in a good representation for the priviledged view $x$.
The \textit{same-domain assumption} is worth highlighting: if we were to choose views of a different dimensionality or modality, we would not be able to encode (untransformed) images from a test set into our representation.

\paragraph{A Second Toy Example}
To showcase these properties, consider two interweaved spirals, one spinning clockwise and one counter-clockwise, composed of 2D coordinates: $(c_x, c_y)$. Define the elements of the view set $v_x$ for $x$ as $(c_x + \varepsilon_x, c_y + \varepsilon_y)$ where $\varepsilon_x, \varepsilon_y \sim U(0, \eta)$. We vary $\eta$ from 0 (only the priviledged view) to 5 (every point possible in the domain).

\begin{figure}[h!]
  \centering
  \begin{subfigure}[b]{0.24\textwidth}
    \centering
    \includegraphics[width=\textwidth]{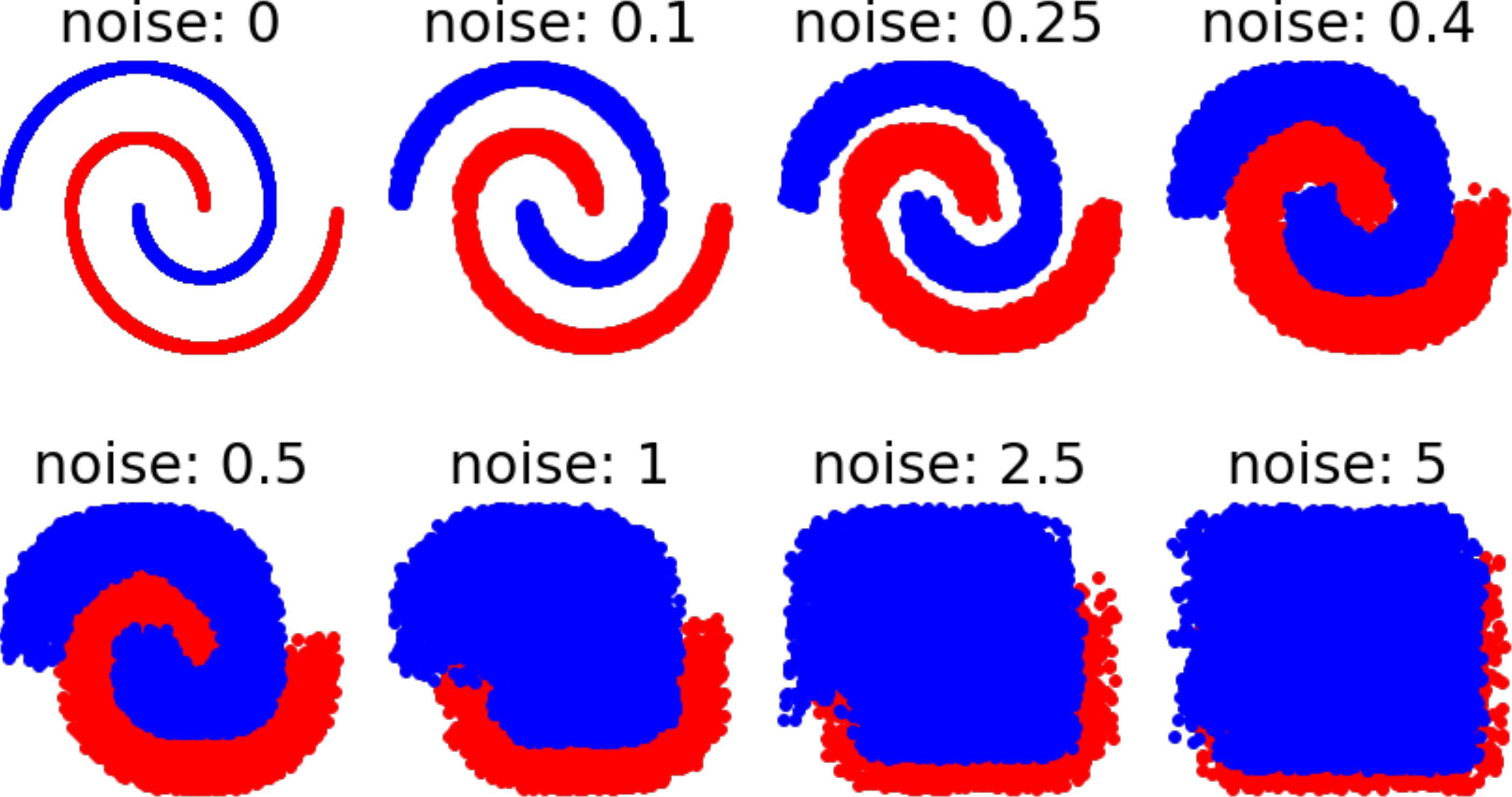}
    \caption{Sampled Views}
  \end{subfigure}
    \centering
  \begin{subfigure}[b]{0.24\textwidth}
    \includegraphics[width=\textwidth]{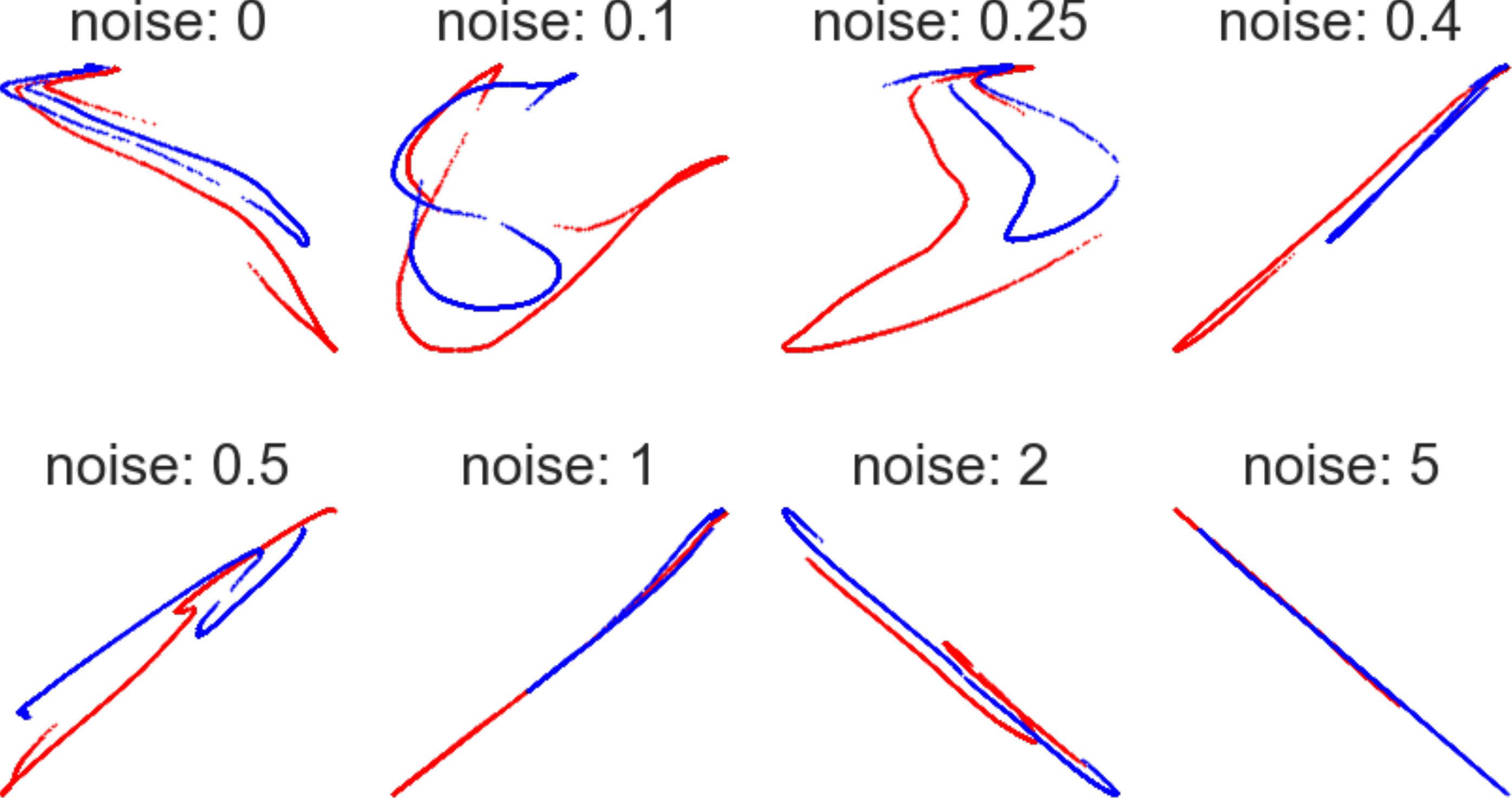}
    \caption{2D Repr.}
  \end{subfigure}
  \begin{subfigure}[b]{0.24\textwidth}
    \centering
    \includegraphics[width=\textwidth]{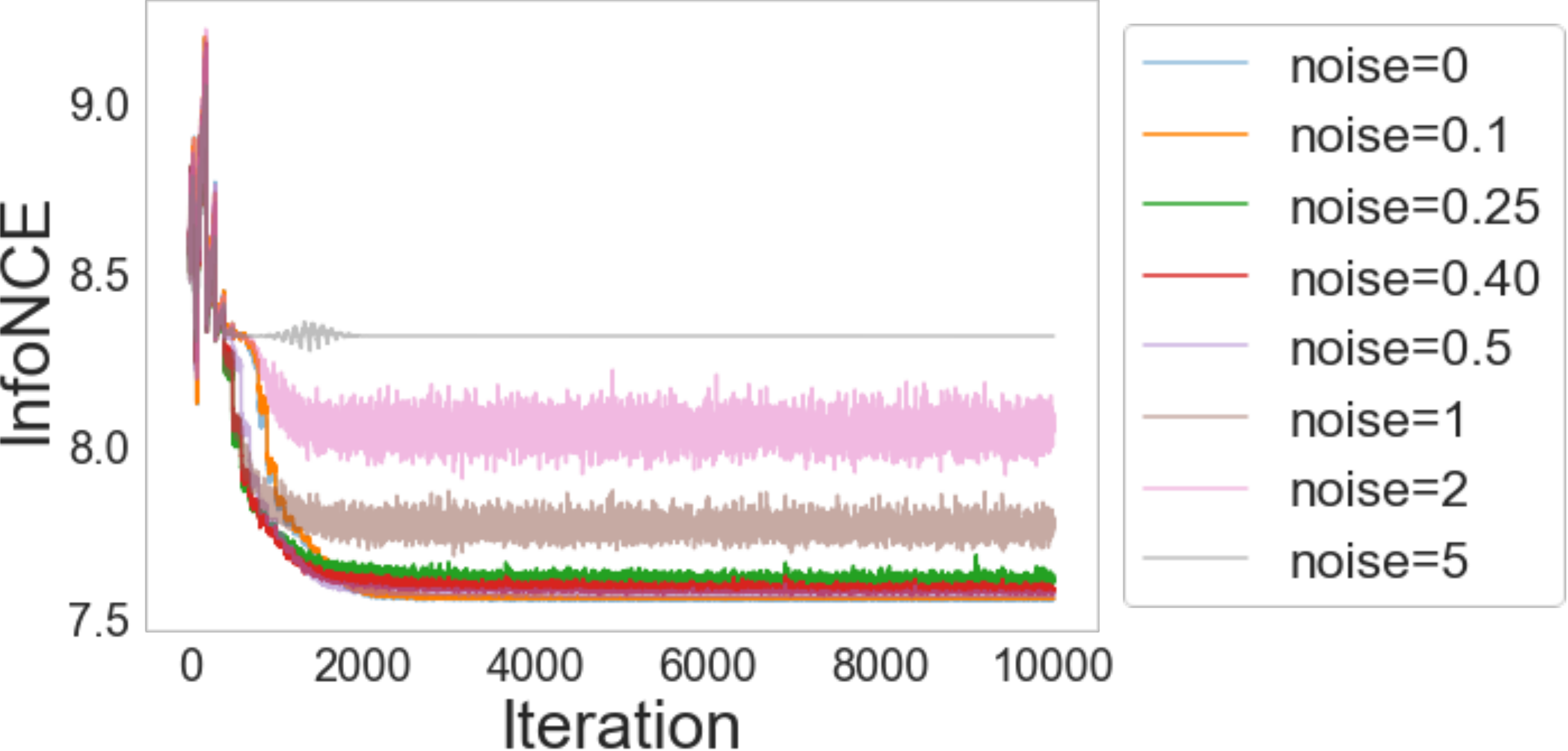}
    \caption{InfoNCE Loss}
  \end{subfigure}
  \begin{subfigure}[b]{0.24\textwidth}
    \centering
    \includegraphics[width=\textwidth]{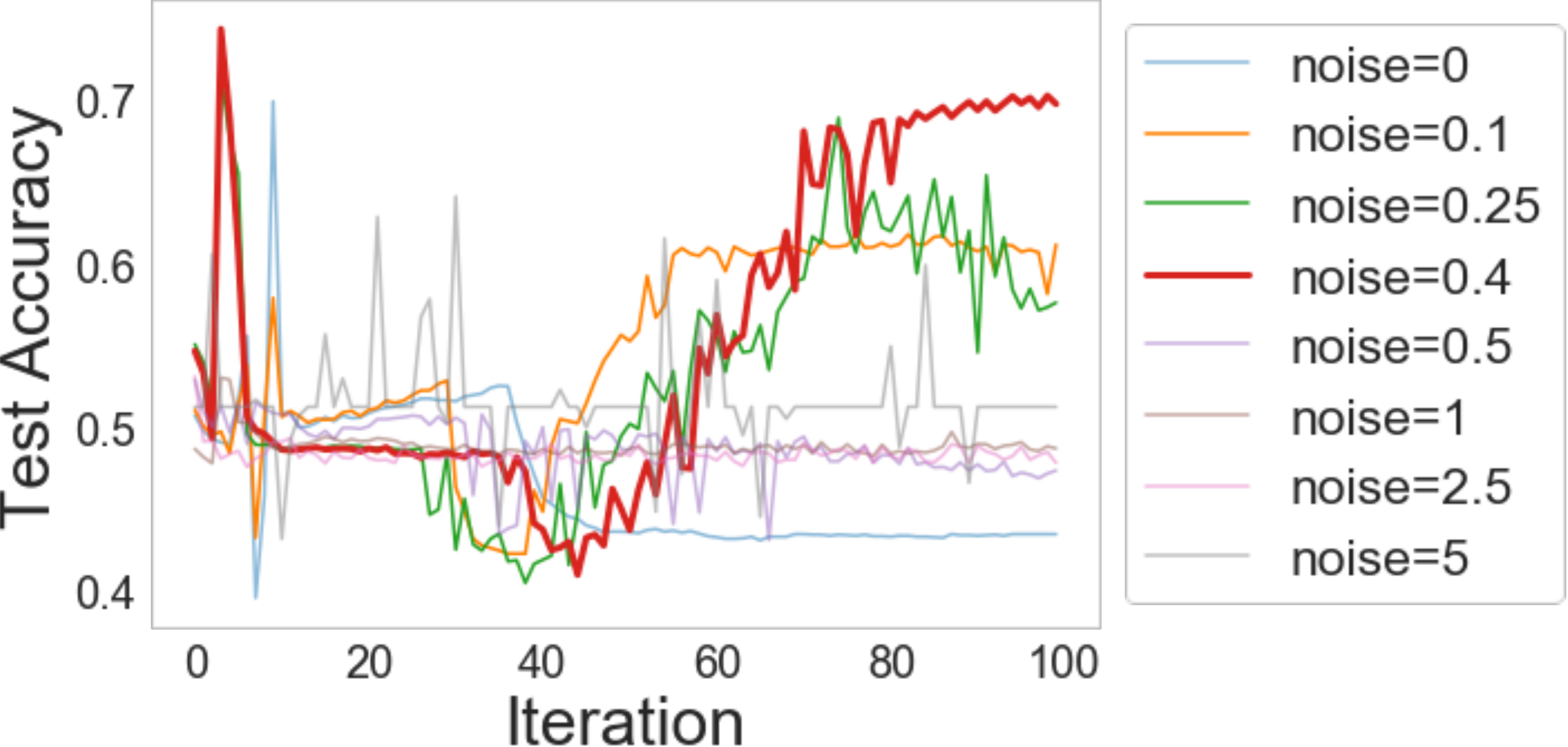}
    \caption{Test Acc.}
  \end{subfigure}
  \caption{Consider points along two intertwined spirals. (a) We vary the amount of noise added to the points. (b) We visualize the learned 2D representation. (c) Training losses over time. With more noise, MI is lower. (d) The transfer task is to predict which spiral a point belongs to.}
  \label{fig:toy}
\end{figure}

Consider the transfer task of predicting which spiral each point belongs to (using logistic regression on the representations).
Thus, the information required to disentangle the  spirals must be linearly separable in the representations.
From Fig.~\ref{fig:toy}, we see that with $\eta = 0$ (lossless views), we can maximize MI (low training loss) but the transfer accuracy is near chance. As noise approaches 0.4 (lossy and unique), we find steady improvements to accuracy. From Fig.~\ref{fig:toy}b, only $\eta = 0.4$ has a (mostly) linear separation. However, as noise surpasses 0.4 (and views lose uniqueness), transfer accuracy recedes to chance.
These results exemplify the importance of lossy-ness and uniqueness.

\paragraph{Training Details} The encoders are 5-layer MLPs with 128 hidden dimensions and map inputs $x$ to a 2D representation.
We use 4096 negative samples drawn from the training dataset to optimize InfoNCE with SGD using momentum 0.9, weight decay 1e-5, batch size 128, and learning rate 0.03 for 10k iterations. The memory bank update parameter $\alpha$ is set to $0.5$ with a temperature of $0.07$. The double spiral dataset contains 10k points contained within a 2D box of $[-2,2]^2$.

\begin{figure}[h!]
  \centering
  \begin{subfigure}[b]{0.24\textwidth}
    \centering
    \includegraphics[width=\textwidth]{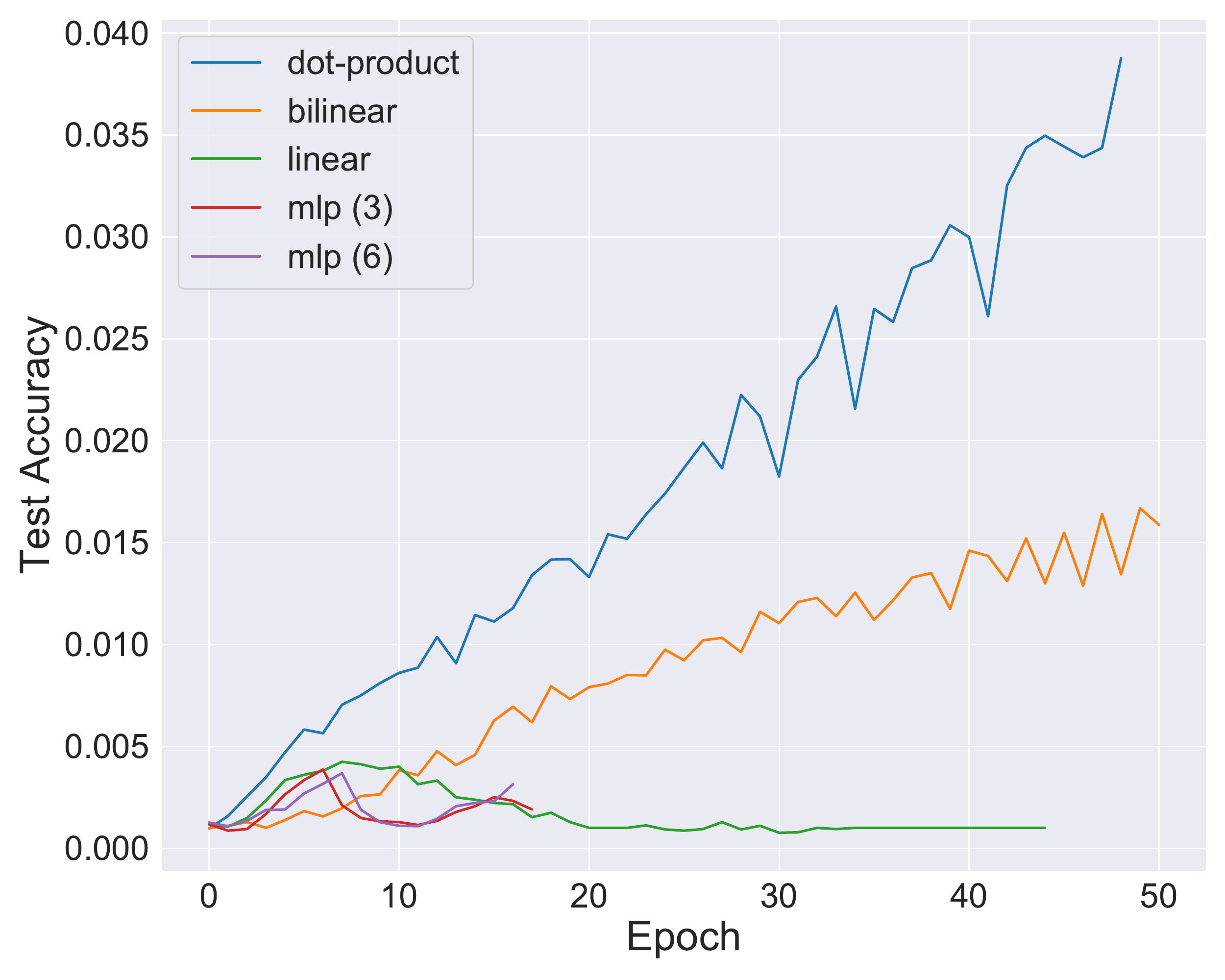}
    \caption{Witness (ImageNet)}
  \end{subfigure}
  \begin{subfigure}[b]{0.24\textwidth}
    \centering
    \includegraphics[width=\textwidth]{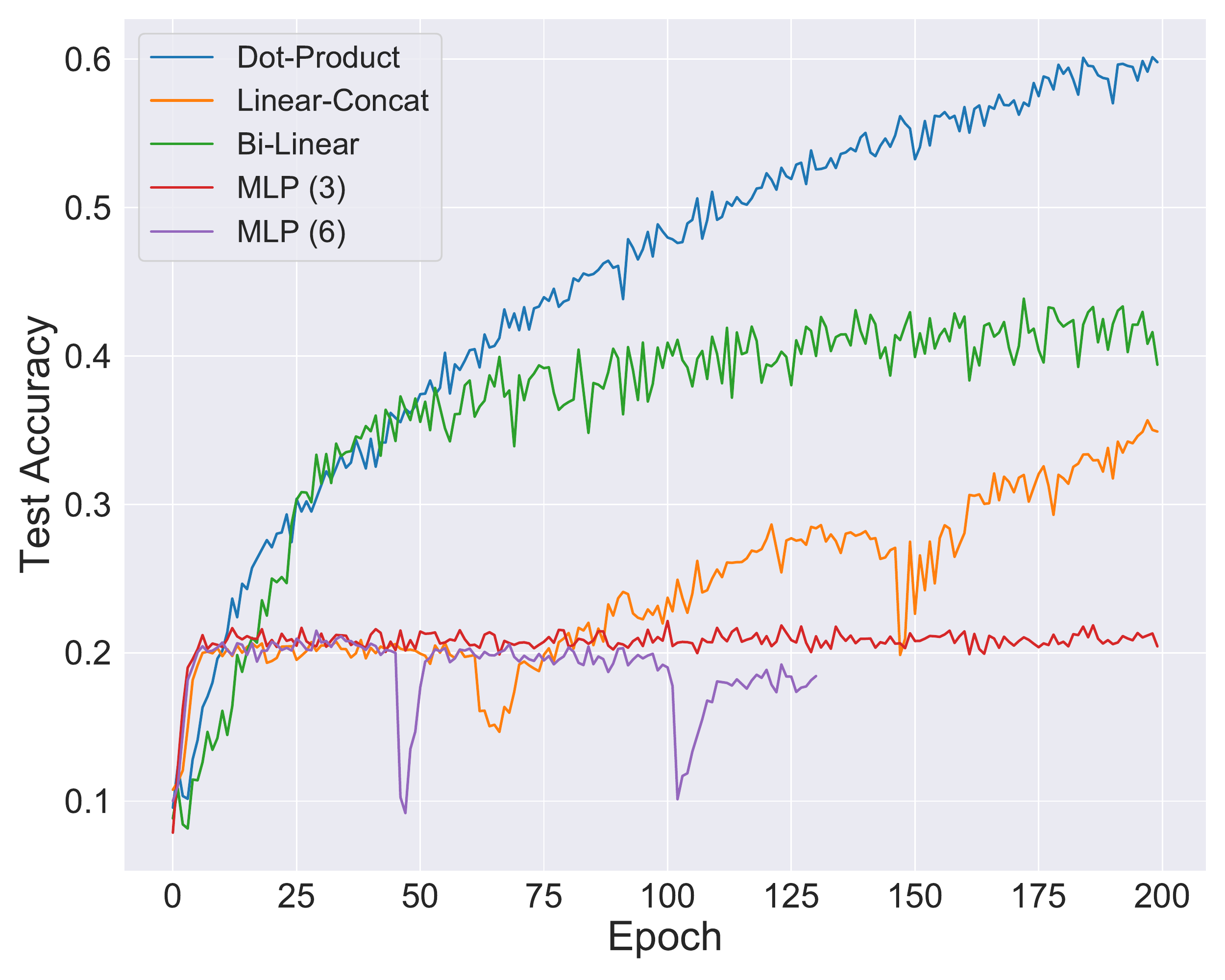}
    \caption{Witness (CIFAR)}
  \end{subfigure}
  \begin{subfigure}[b]{0.24\textwidth}
    \centering
    \includegraphics[width=\textwidth]{imagenet_stability.pdf}
    \caption{Stability (ImageNet)}
  \end{subfigure}
  \begin{subfigure}[b]{0.24\textwidth}
    \centering
    \includegraphics[width=\textwidth]{cifar_stability.pdf}
    \caption{Stability (CIFAR)}
  \end{subfigure}
  \begin{subfigure}[b]{0.24\textwidth}
    \centering
    \includegraphics[width=\textwidth]{imagenet_ir_memory.pdf}
    \caption{$\alpha$+IR (ImageNet)}
  \end{subfigure}
  \begin{subfigure}[b]{0.24\textwidth}
    \centering
    \includegraphics[width=\textwidth]{imagenet_la_memory.pdf}
    \caption{$\alpha$+LA (ImageNet)}
  \end{subfigure}
  \begin{subfigure}[b]{0.24\textwidth}
    \centering
    \includegraphics[width=\textwidth]{cifar_ir_memory.pdf}
    \caption{$\alpha$+IR (CIFAR)}
  \end{subfigure}
  \begin{subfigure}[b]{0.24\textwidth}
    \centering
    \includegraphics[width=\textwidth]{cifar_cmc_memory.pdf}
    \caption{$\alpha$+CMC (CIFAR)}
  \end{subfigure}
  \caption{Nearest neighbor classification accuracy for various experiments: (a, b) compare different architectures for $f_\theta(x,y)$; (c, d) compare the original IR objective with InfoNCE in stability; (e-h) compare different update parameters $\alpha$ for the memory bank of IR, LA, and CMC.}
  \label{fig:discussion2}
\end{figure}

\section{Additional Results}

\paragraph{Learning curves} Fig.~\ref{fig:results_knn} show the learning curves comparing IR, BALL, ANN, and LA on ImageNet and CIFAR10. The y-axis plots the nearest neighbor classification accuracy.

\begin{figure}[h!]
  \centering
  \begin{subfigure}[b]{0.20\textwidth}
    \centering
    \includegraphics[width=\textwidth]{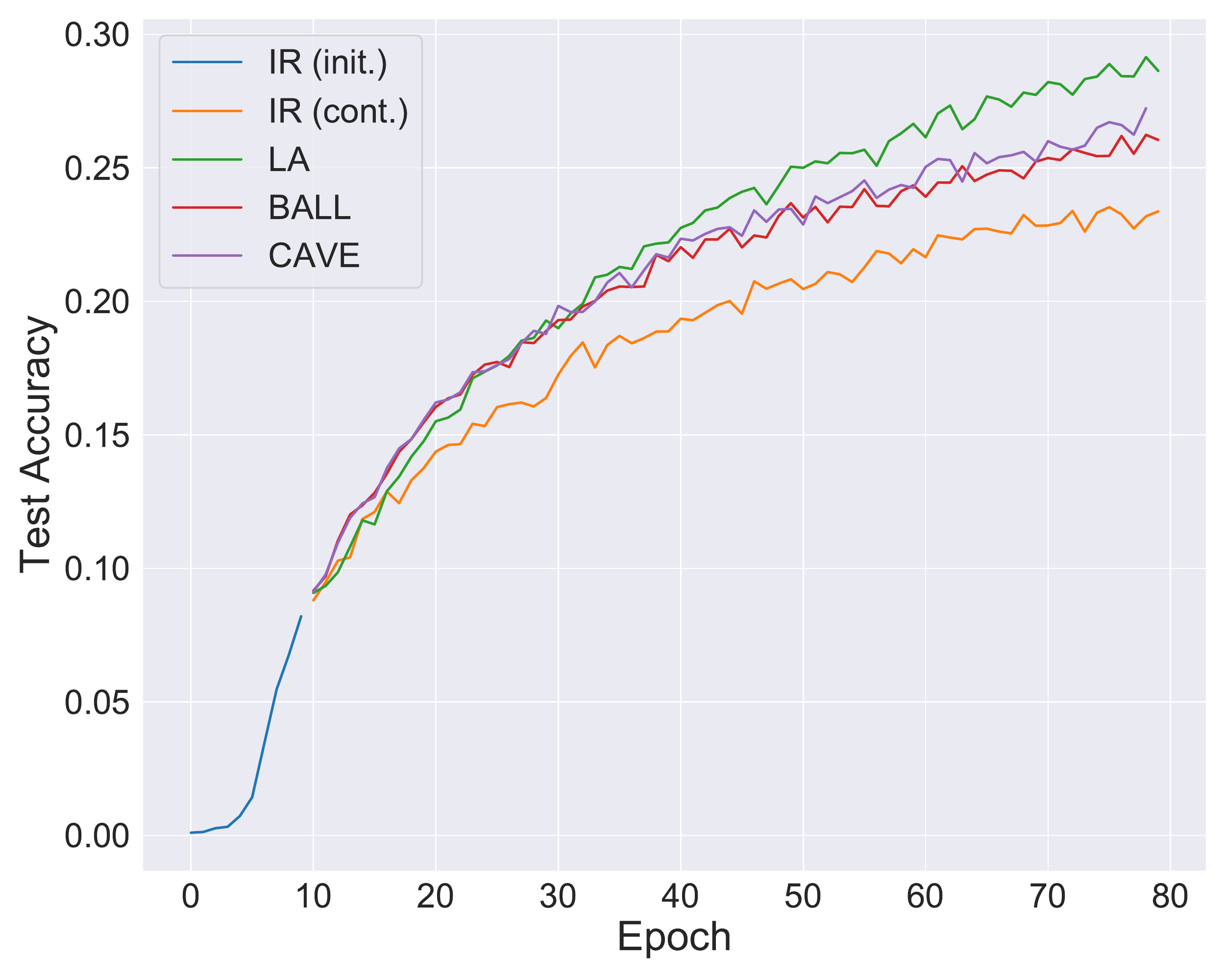}
    \caption{ImageNet (IR)}
  \end{subfigure}
  \begin{subfigure}[b]{0.20\textwidth}
    \centering
    \includegraphics[width=\textwidth]{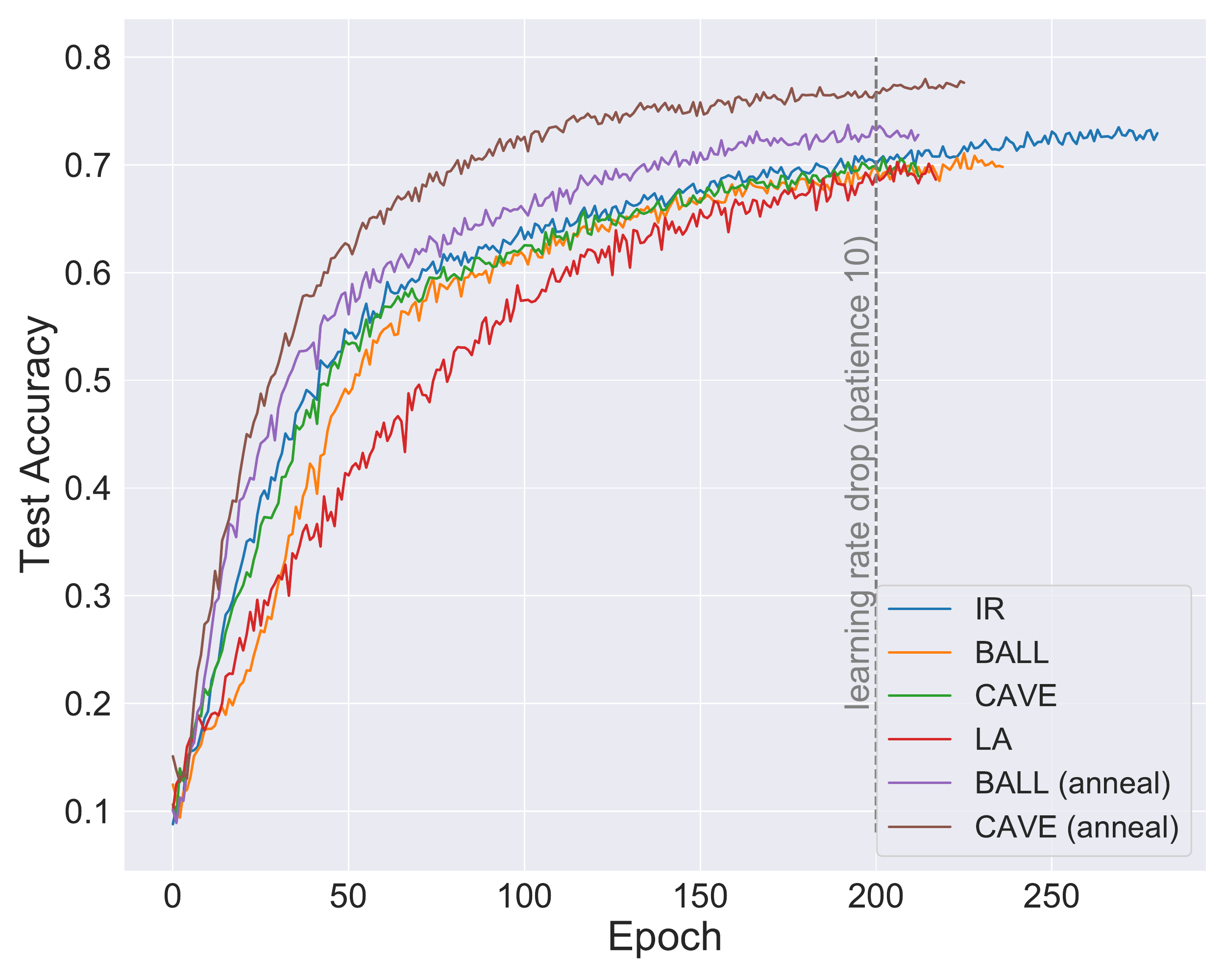}
    \caption{CIFAR10 (IR)}
  \end{subfigure}
  \begin{subfigure}[b]{0.20\textwidth}
    \centering
    \includegraphics[width=\textwidth]{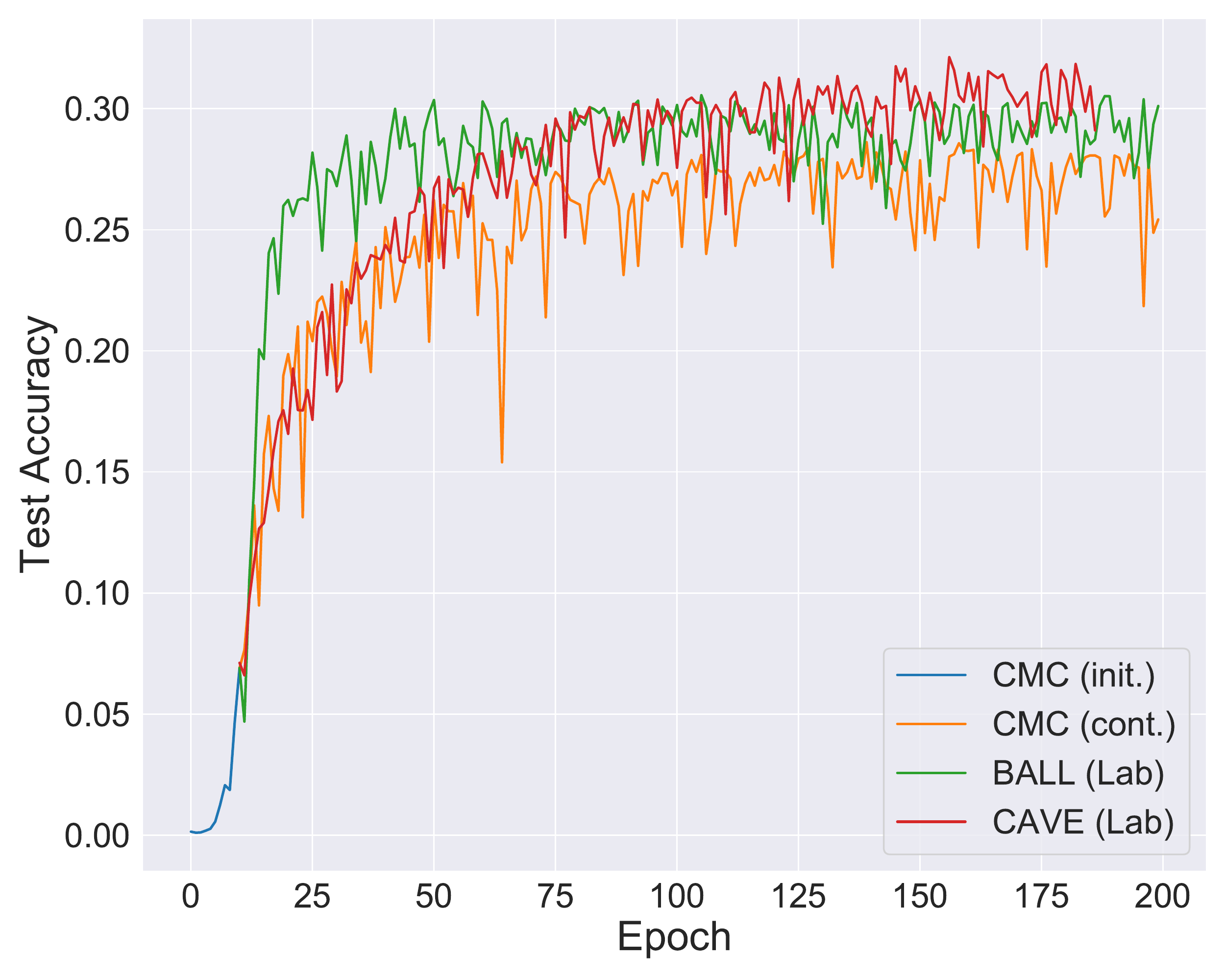}
    \caption{ImageNet (CMC)}
  \end{subfigure}
  \begin{subfigure}[b]{0.20\textwidth}
    \centering
    \includegraphics[width=\textwidth]{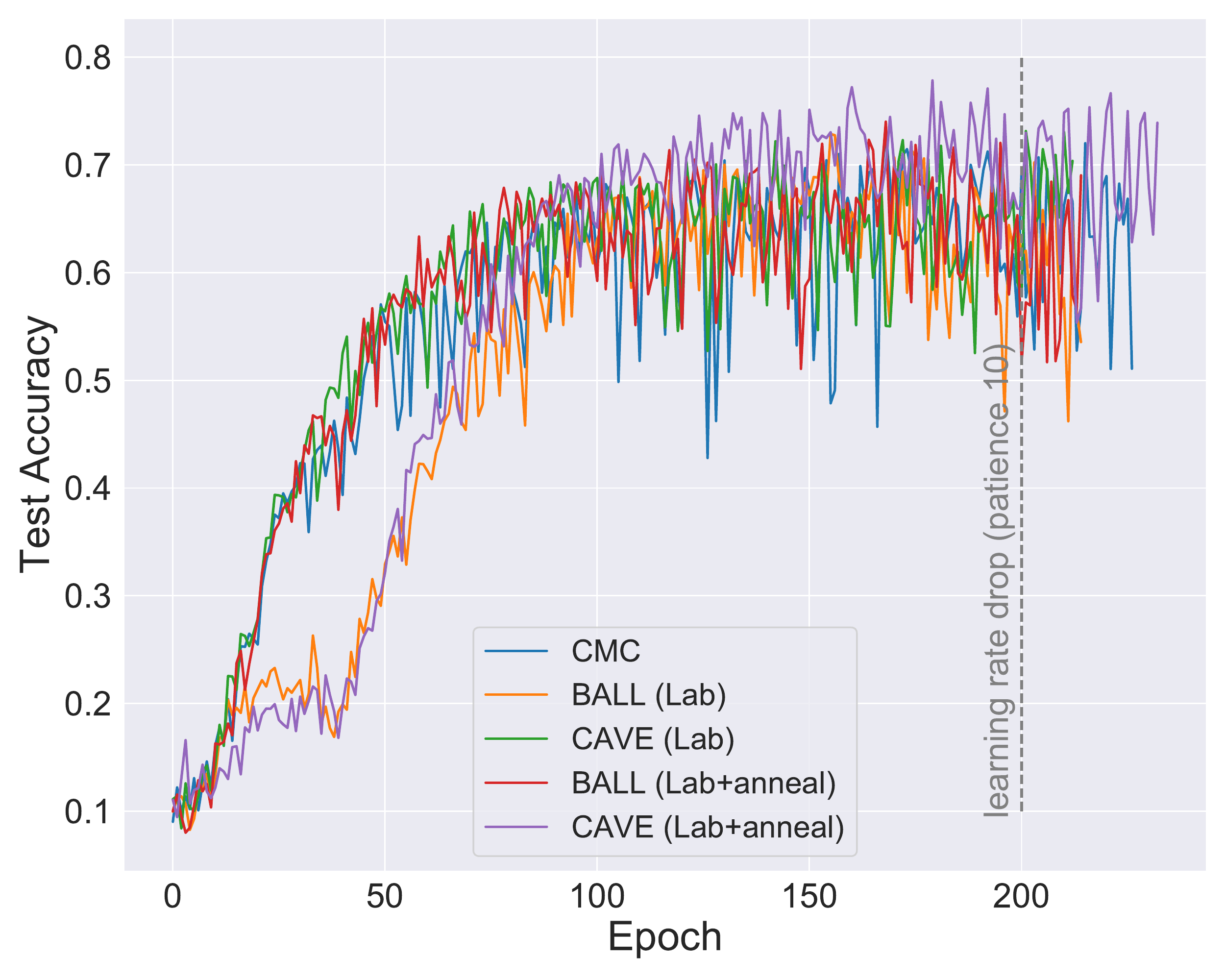}
    \caption{CIFAR10 (CMC)}
  \end{subfigure}
  \caption{Nearest neighbor (NN) classification accuracy throughout training: (a,b) show the IR family of models on ImageNet and CIFAR10; (c,d) show similar results for the CMC family.}
  \label{fig:results_knn}
\end{figure}

\paragraph{Additional COCO Mask R-CNN Results}
In the main text, Table~\ref{tab:othertransfertask} show object detection and segmentation results on COCO using the representations learned by various contrastive algorithms. We used a fixed backbone with a feature pyramid network. Table~\ref{table:nofeatpyra} show a similar set of results without a feature pyramid (as is also done in ~\cite{he2019momentum}
). We find the same patterns: that is, annealed Annulus Discrimination outperforms all other models.

\begin{table}[t!]
\caption{COCO: Mask R-CNN, R$_{18}$-C4, 1x schedule}
\centering
\begin{tabular}{l|c c c | c c c }
\toprule
Model & AP$^{\text{bb}}$ & AP$_{50}^{\text{bb}}$ & AP$_{75}^{\text{bb}}$ & AP$^{\text{mk}}$ & AP$_{50}^{\text{mk}}$ & AP$_{75}^{\text{mk}}$ \\
\midrule
IR & $6.1$ & $13.9$ & $4.4$ & $6.0$ & $12.7$ & $5.1$ \\
BALL & $6.3$ & $14.3$ & $4.9$ & $6.4$ & $13.2$ & $5.6$ \\
ANN$^{\text{kmeans}}$ & $6.7$ & $14.8$ & $5.3$ & $6.7$ & $13.9$ & $5.9$ \\
LA & $7.4$ & $16.1$ & $5.7$ & $7.4$ & $15.2$ & $6.4$ \\
BALL$^{\text{anneal}}$ & $7.2$ & $15.5$ & $5.6$ & $7.3$ & $14.4$ & $6.2$ \\
ANN$^{\text{kmeans}}_{\text{+anneal}}$ & $\mathbf{7.7}$ & $\mathbf{16.5}$ & $\mathbf{6.0}$ & $\mathbf{7.8}$ & $\mathbf{15.5}$ & $\mathbf{6.7}$ \\
\bottomrule
\end{tabular}
\label{table:nofeatpyra}
\end{table}

\paragraph{Additional CIFAR10 Results} Due to a lack of space, in the main text we only reported Top 1 accuracies for the CIFAR10 dataset. Here we report the full table with Top 5. See Table~\ref{table:cifar_supp}.

\begin{table}[h!]
\small
\centering
\begin{tabular}{l|c|c|l|c|c}
\toprule
Model & Top1 & Top5 & Model & Top1 & Top5\\
\midrule
IR \cite{wu2018unsupervised} & 64.3 & 95.9 & CMC \cite{tian2019contrastive} & 72.1 & 97.3 \\
IR$^{\text{nce}}$ & 81.2 & 98.9 & CMC$^{\text{nce}}$ & 85.6 & 99.4 \\
LA \cite{zhuang2019local} / BALL$^{\text{k-neigh}}$ & 81.8 & 99.1 & - & - & -  \\
LA$^{\text{nce}}$ & 82.3 & 99.1 & - & - & - \\
BALL & 81.4 & 99.0 & BALL$^{\textup{Lab}}$ & 85.7 & 99.5 \\
BALL$^{\textup{anneal}}$ & 82.1 & 99.1 & BALL$^{\textup{Lab+anneal}}$ & 86.8 & 99.5 \\
BALL$^{\textup{s-neigh}}$ & 75.1 & 95.9 & BALL$^{\textup{Lab+s-neigh}}$ & 63.3 & 94.5\\
BALL$^{\textup{s-neigh+anneal}}$ & 84.8 & 99.3 & BALL$^{\textup{Lab+s-neigh+anneal}}$ & 87.0 & 99.4\\
CAVE & 81.4 & 99.0 & CAVE$^{\textup{Lab}}$ & 86.1 & 99.5\\
CAVE$^{\textup{anneal}}$ & 84.8 & 99.3 & CAVE$^{\textup{Lab+anneal}}$ & \textbf{87.8} & \textbf{99.5}\\
RING & 84.7 & 99.3 & RING$^{\textup{Lab}}$ & 87.3 & 99.4 \\
RING$^{\textup{s-neigh}}$ & 76.6 & 98.6 & RING$^{\textup{Lab+s-neigh}}$ & 68.0 & 95.7\\
RING$^{\textup{anneal}}$ & 85.2 & 99.3 & RING$^{\textup{Lab+anneal}}$ & 87.6 & 99.5\\
RING$^{\textup{s-neigh+anneal}}$ & \textbf{85.5} & \textbf{99.4} & RING$^{\textup{Lab+s-neigh+anneal}}$ & \textbf{87.8} & \textbf{99.5}\\
\bottomrule
\end{tabular}
\caption{CIFAR10 Transfer Accuracy}
\label{table:cifar_supp}
\end{table}

\section{Additional Discussion}
\label{sec:adddis}
We discuss a few minor points regarding MI and contrastive learning not included in the main text. We also include some additional figures in Fig.~\ref{fig:discussion2} for memory bank experiments.

\begin{figure}[h!]
  \centering
  \begin{subfigure}[b]{0.32\textwidth}
    \centering
    \includegraphics[width=\textwidth]{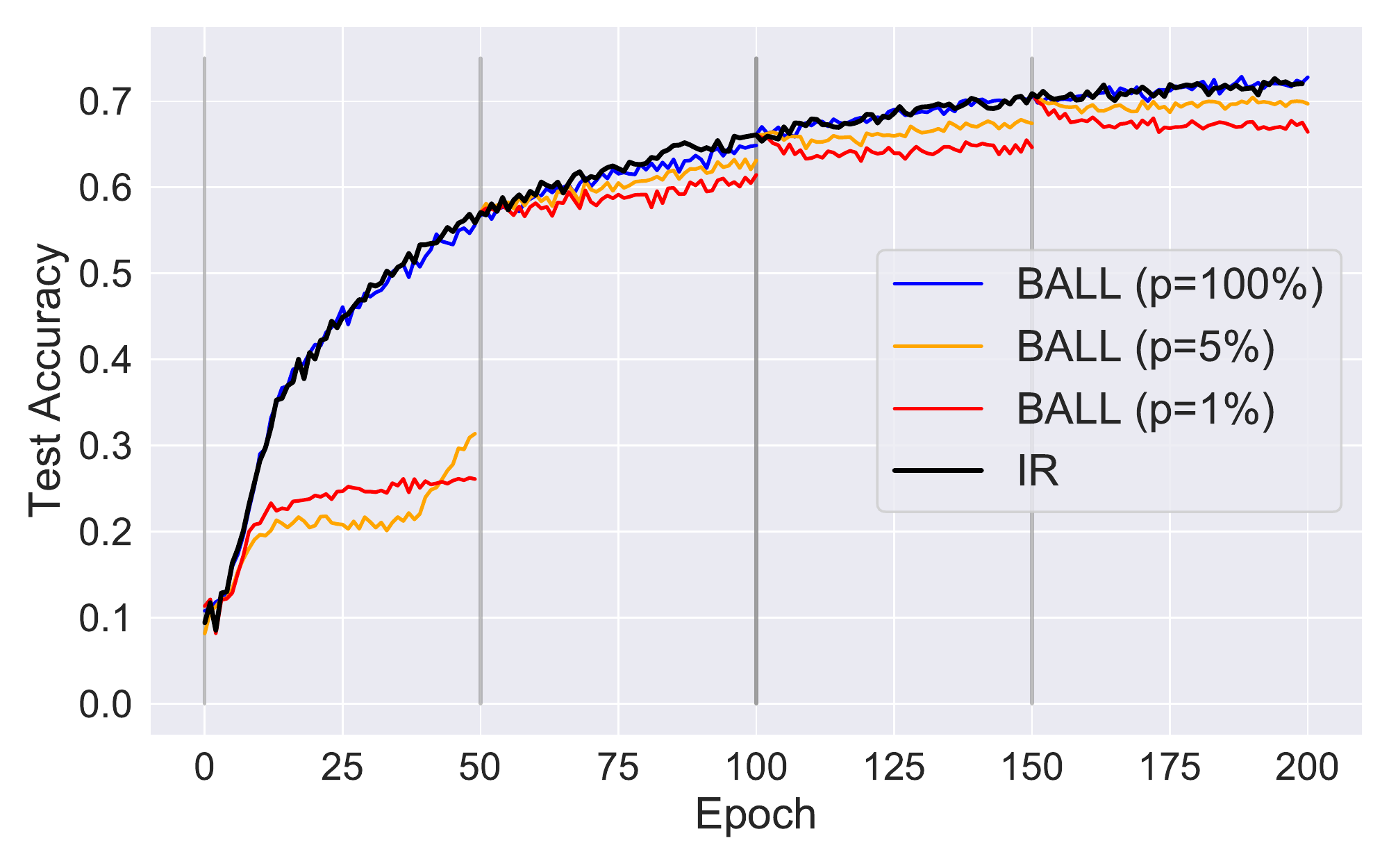}
    \caption{100 negatives}
  \end{subfigure}
  \begin{subfigure}[b]{0.32\textwidth}
    \centering
    \includegraphics[width=\textwidth]{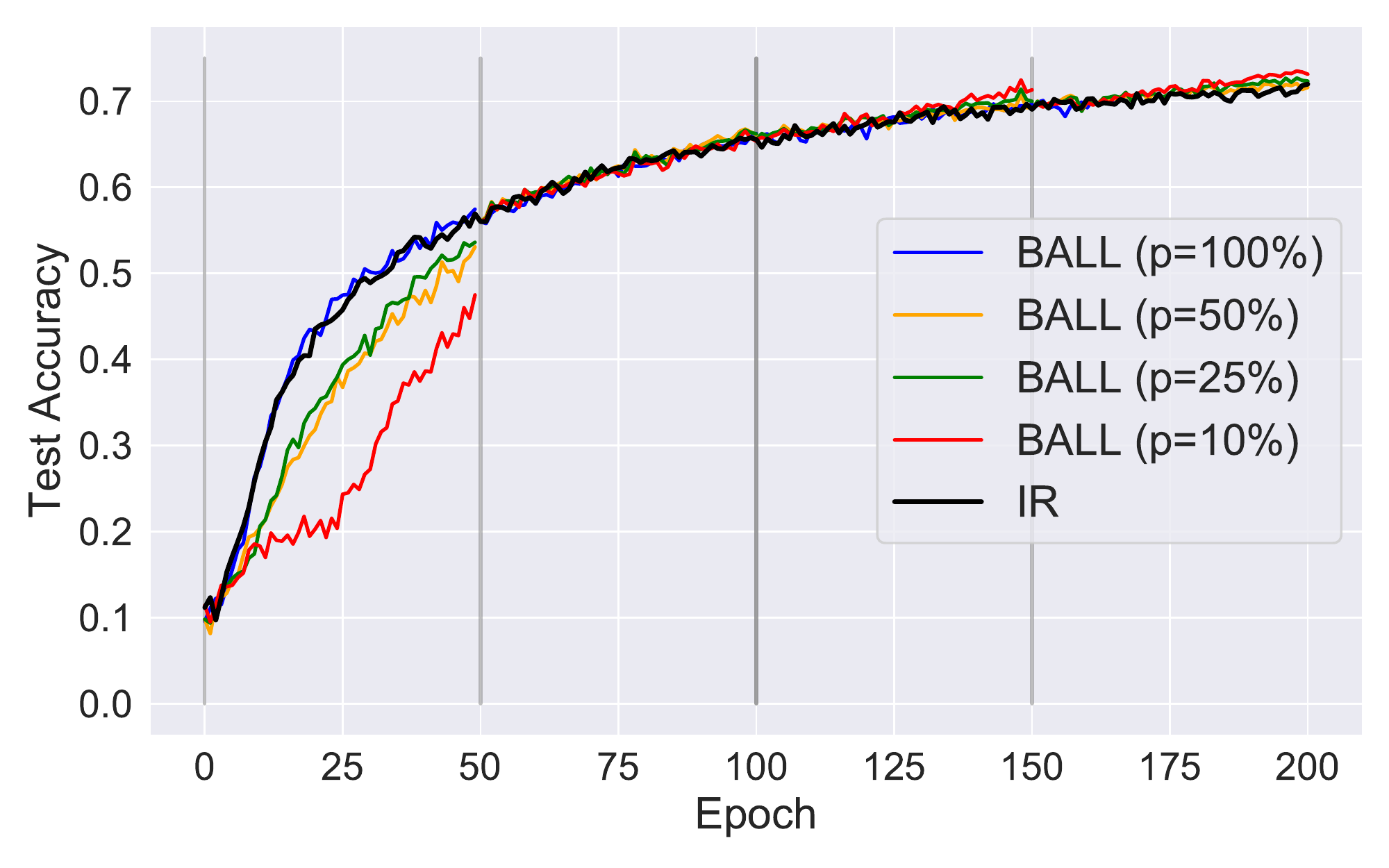}
    \caption{1024 negatives}
  \end{subfigure}
  \begin{subfigure}[b]{0.32\textwidth}
    \centering
    \includegraphics[width=\textwidth]{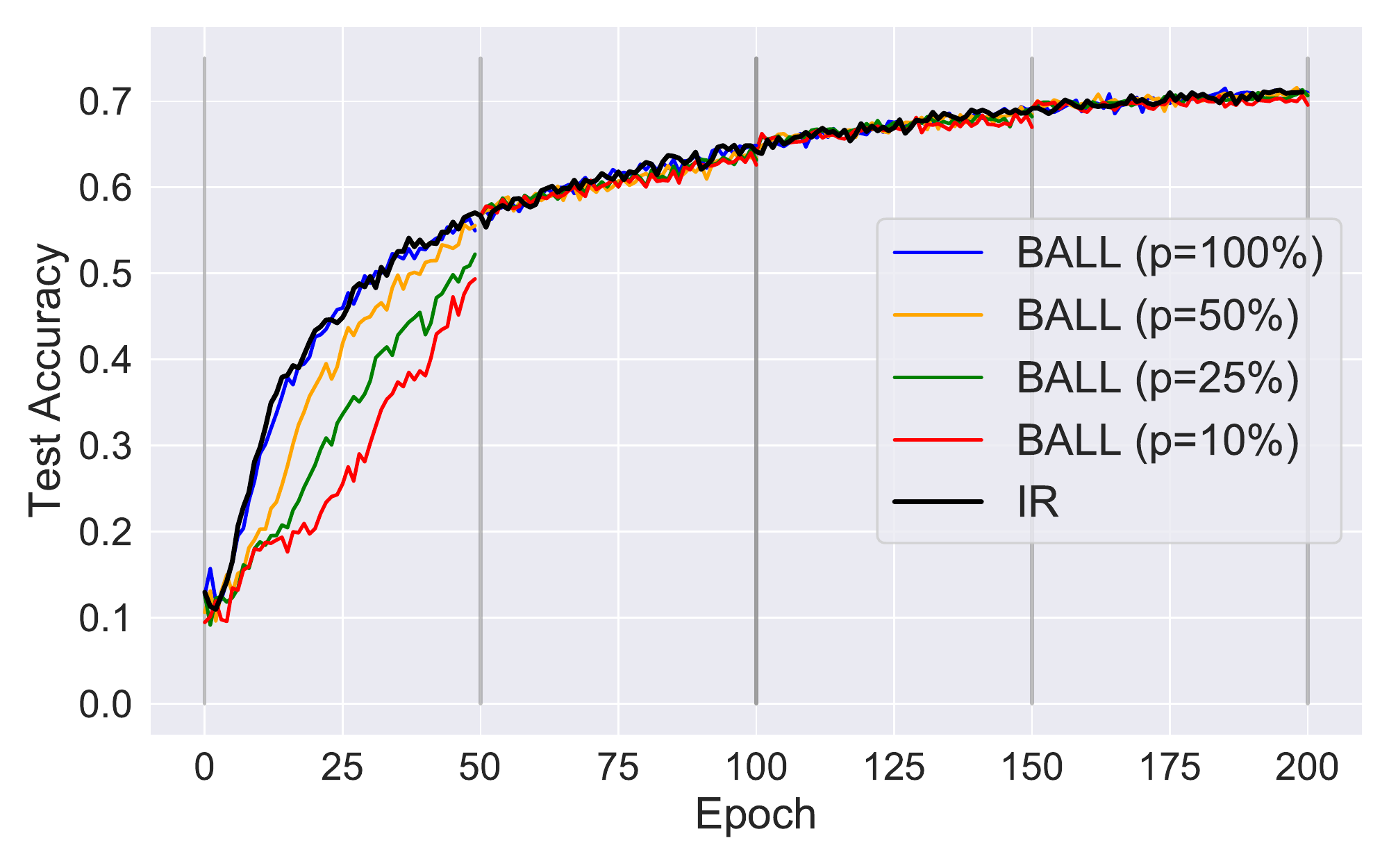}
    \caption{4096 negatives}
  \end{subfigure}
  \caption{Relationship between the total number of  negative samples drawn and its distribution: (a,b ,c) show the nearest neighbor classification accuracy when using various values for $\tau$ to define $q_\tau$.}
  \label{fig:limitation}
\end{figure}

\begin{figure}[h!]
  \centering
  \begin{subfigure}[b]{0.32\textwidth}
    \centering
    \includegraphics[width=\textwidth]{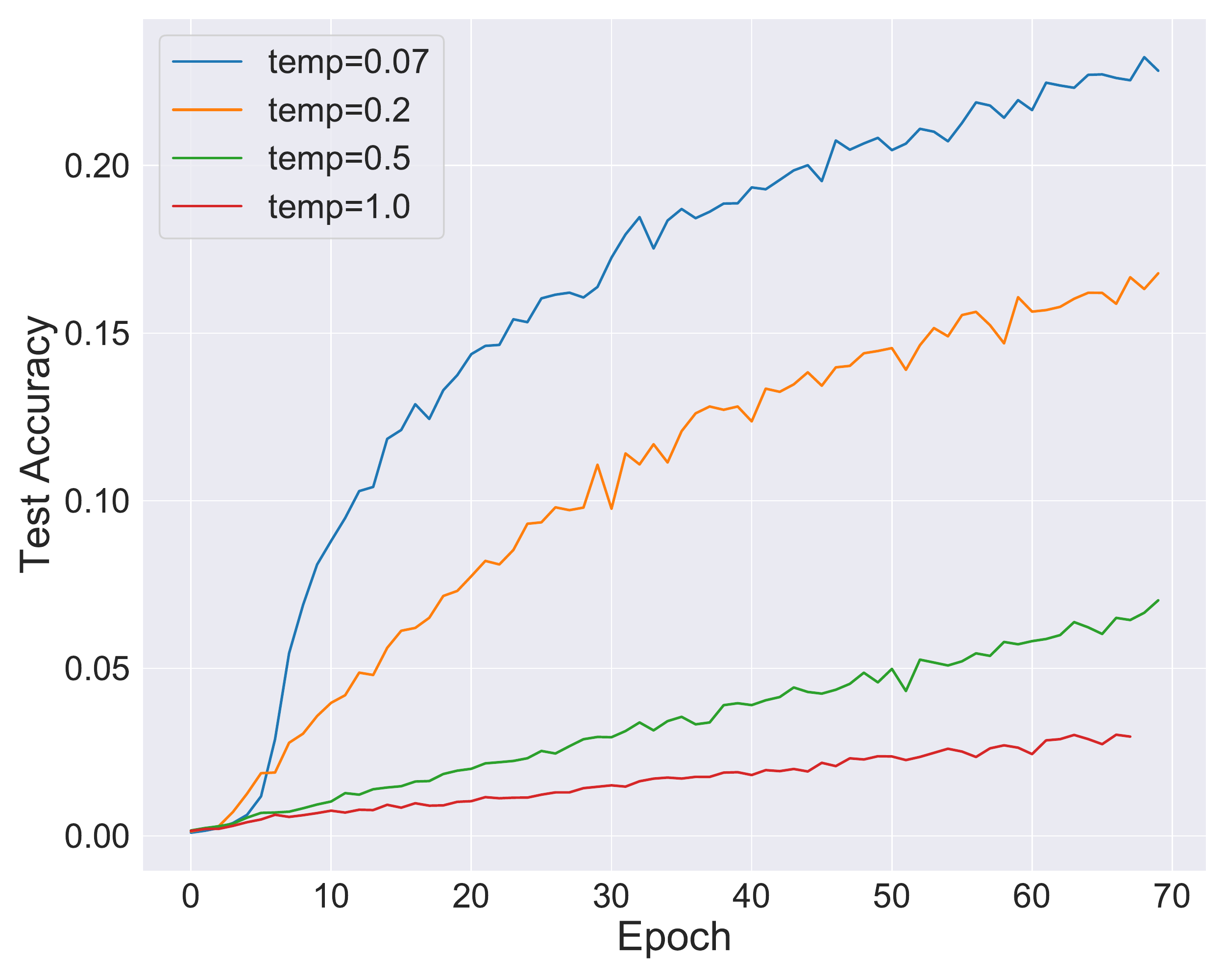}
    \caption{Imagenet (IR)}
  \end{subfigure}
  \begin{subfigure}[b]{0.32\textwidth}
    \centering
    \includegraphics[width=\textwidth]{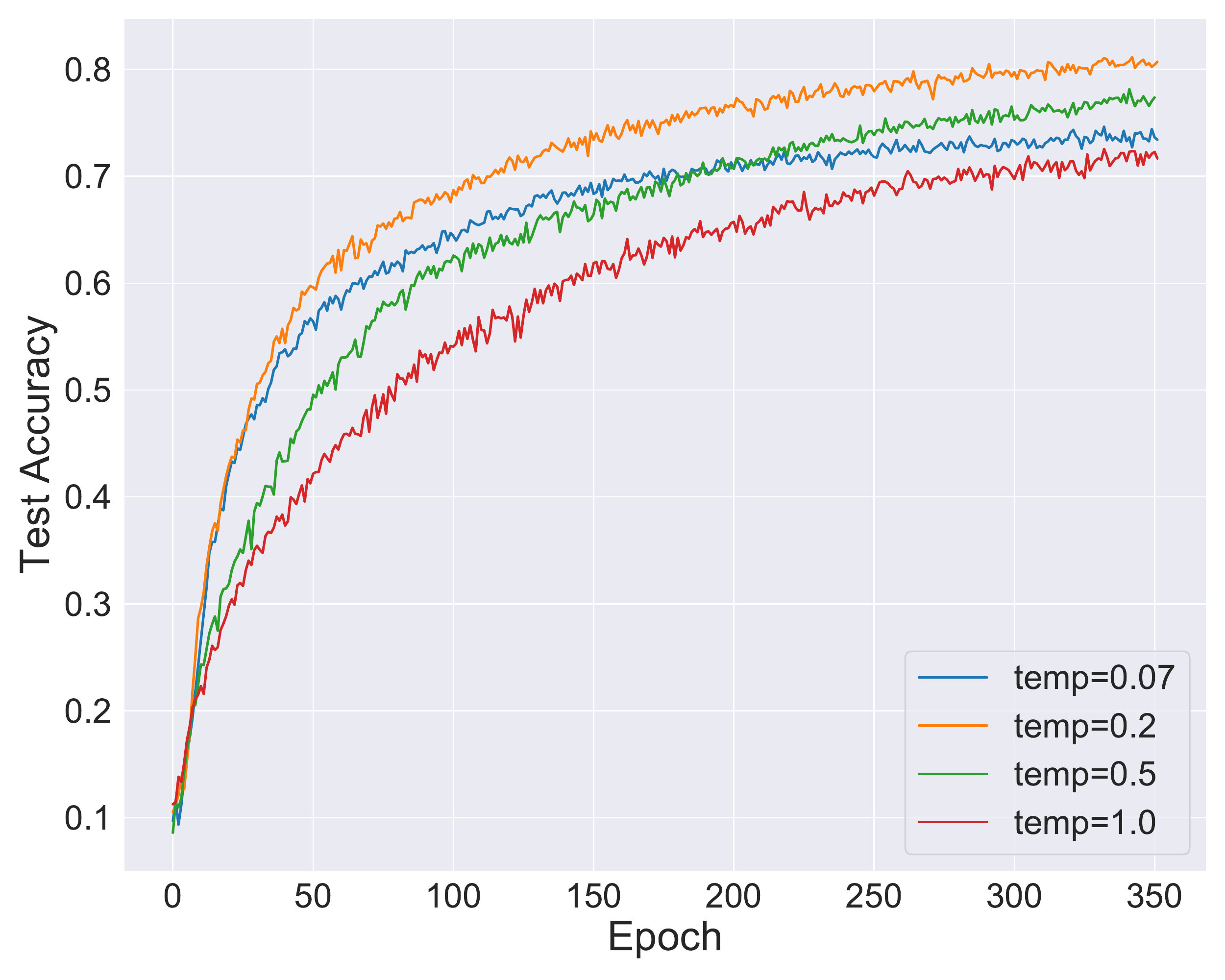}
    \caption{CIFAR10 (IR)}
  \end{subfigure}
  \caption{Relationship between nearest neighbor classification accuracy and temperature.}
  \label{fig:temperature}
\end{figure}

\paragraph{The witness function should be under-parameterized.}
Since $f_{\theta}(x,y)$ encodes $x$ and $y$ independently before the dot product, the encoders bear the near-full burden of estimating mutual information.
However, could we parameterize $f_{\theta}$ more heavily (e.g. perceptron)?
Consider the alternative design: $f_{\theta,\psi}(x,y) = h_\psi(g_\theta(x), g_\theta(y))$ where $h_\psi$ is a neural network.
At one extreme, $g_\theta$ and $g_\theta$ mimic identity functions whereas $h_\psi$ assumes the burden of learning.
Representations learned in this setting should prove far less useful.
To test this hypothesis, we compare the following designs: (1) a dot product $g_\theta(x)^T g_\theta(y)$; (2) a bilinear transform $g_\theta(x)^T W g_\theta(y)$; (3) concatenate $g_\theta^T(x)$ and $g_\theta(y)$ and pass it into a linear layer projecting to 1 dimension; (4,5) concatenate $g_\theta^T(x)$ and $g_\theta(y)$ and pass it through a series of nonlinear layers with 128 hidden nodes before projecting to 1 dimension.
In order from (1) to (5), the models increase in expressivity.
Fig.~\ref{fig:discussion2}a and b show classification performance on ImageNet and CIFAR10.
We find that a more expressive $h_\psi$ results in lower accuracy, symbolizing a weaker representation.
We highlight that these experiments differ from \cite{chen2020simple} where the authors found $f_{\theta,\psi}(x,y) = (h_\psi \circ g_\theta)(x)^T (h_\psi \circ g_\theta)(y)$ to improve performance.

\paragraph{Harder negative samples are not \textit{always} good.}
In Fig.~\ref{fig:results_knn} and Table~\ref{tab:othertransfertask}, we see that slowly increasing the threshold $\tau$ significantly improves performance (see BALL$^{\text{anneal}}$/CAVE$^{\text{anneal}}$).
This implies that using harder negative samples may not be beneficial at all points in training.
Early in training, it could that difficult negatives bias towards a local optima.
To test this, we can train BALL models with varied $\tau$ whose parameters are initialized with a IR model at different points of training (e.g. epoch 1 versus epoch 100). Fig.~\ref{fig:limitation} shows exactly this, additionally varying the number of total negatives drawn from 100 to 4096.
Recall from the first toy experiment that a smaller percentage $p$ represents a larger $\tau$ (with $p=100\%$ equaling IR).
We make two observations: First, at the beginning of training, difficult negatives hurt the representation. Second, with too few samples, difficult negatives are too close to the current image, stunting training.
But with too many samples, difficult negatives are ``drowned out'' by others. If we balance the difficulty of the negatives with the number of samples, we find improved performance with larger $\tau$ (see Fig.~\ref{fig:limitation}b).
In summary, always using harder negative samples is not a bulletproof strategy. In practice, we find slowly introducing them throughout training to work well.

\paragraph{The temperature parameter is important.} In simplifying contrastive learning, we might wonder if we could remove the temperature parameter $\omega$.
Fig.~\ref{fig:temperature} suggests that tuning $\omega$ is important for at least the speed of convergence. We find the best $\omega$ to be dataset dependent.

\end{document}